\documentclass{article} 
\usepackage{iclr2025_conference,times}

\usepackage{amsmath,amsfonts,bm}

\def\eqref#1{equation~\ref{#1}}

\def\1{\bm{1}}

\DeclareMathAlphabet{\mathsfit}{\encodingdefault}{\sfdefault}{m}{sl}
\SetMathAlphabet{\mathsfit}{bold}{\encodingdefault}{\sfdefault}{bx}{n}

\usepackage{hyperref}
\usepackage[capitalise]{cleveref}
\usepackage{url}
\usepackage{multirow}
\usepackage{graphicx}
\usepackage{subfigure}
\usepackage{adjustbox}
\usepackage{amsmath,amsfonts,bm}
\usepackage{amsthm}
\usepackage{pifont}
\usepackage{xcolor}
\usepackage{bbm}
\usepackage{booktabs}
\newtheorem{theorem}{Theorem}
\newtheorem{lemma}{Lemma}
\newtheorem{definition}{Definition}
\usepackage{algorithm2e}
\usepackage{xpatch}
\usepackage{colortbl}
\xpretocmd{\algorithm}{\hrule height 0.8pt\vspace{-2pt}}{}{}
\xapptocmd{\endalgorithm}{\vspace{-2pt}\hrule height 0.8pt}{}{}

\iclrfinalcopy
\title{Quasi Random Physics-informed Neural Networks}

\makeatletter
\let\@fnsymbol\@arabic
\makeatother
\author{Tianchi Yu \thanks{Skolkovo Institute of Science and Technology, corresponding author: \textit{tianchi.yu@skoltech.ru}.} \\
\And
Ivan Oseledets \thanks{Skolkovo Institute of Science and Technology; AIRI.} \\
}

\begin{document}

\maketitle
\begin{abstract}
Physics-informed neural networks have shown promise in solving partial differential equations (PDEs) by integrating physical constraints into neural network training, but their performance is sensitive to the sampling of points. Based on the impressive performance of quasi Monte-Carlo methods in high dimensional problems, this paper proposes Quasi-Random Physics-Informed Neural Networks (QRPINNs), which use low-discrepancy sequences for sampling instead of random points directly from the domain. Theoretically, QRPINNs have been proven to have a better convergence rate than PINNs. Empirically, experiments demonstrate that QRPINNs significantly outperform PINNs and some representative adaptive sampling methods, especially in high-dimensional PDEs. Furthermore, combining QRPINNs with adaptive sampling can further improve the performance.
\end{abstract}
\section{Introduction}\label{section: introduction}
Within scientific computing, solving partial differential equations (PDEs) is central to a wide range of applications, spanning across diverse fields such as fluid dynamics \cite{jin2021nsfnets}, heat transfer \cite{ozicsik2017finite, reddy2022finite}, and climate prediction \cite{saha2014ncep}. Traditional numerical methods \cite{thomas2013numerical,zienkiewicz2005finite,patankar2018numerical} have been the cornerstone for numerical PDEs in scientific computing since the inception of scientific computing. However, these methods face challenges when dealing with high-dimensional problems. In high-dimensional PDEs, the "curse of dimensionality" limits these traditional methods, leading to a drastic and catastrophic increase in computational cost and memory requirements \cite{constantine2015active,trefethen2017cubature}.

In recent years, Physics-Informed Neural Networks (PINNs) \cite{raissi2019physics} have emerged as a revolutionary approach for solving numerical PDEs which can overcome the curse of dimensionality \cite{wojtowytsch2020can,hu2024tackling}. PINNs integrate the power of deep learning with physical laws by incorporating PDE constraints, including governing equations and their conditions, into the loss function of a neural network with their integration form. To estimate the integration by discrete points the Monte-Carlo method (MC) \cite{kalos2009monte} is used in PINNs. However, the performance of MC is highly sensitive to the selection and distribution of sampled points. 

Many adaptive sampling strategies have been proposed to improve the performance of PINNs. \cite{wu2023comprehensive} developed the residual-based adaptive distribution (RAD) method, selecting high residual points from a candidate pool to update training sets. \cite{peng2022rang} introduced RANG, a residual-based adaptive node generation method for refining collocation points in high-residual regions. \cite{nabian2021efficient} proposed importance sampling for efficient PINN training, prioritizing collocation points with large residuals. \cite{gao2023failure} proposed FI-PINNs, defining failure probability based on residuals to enrich sampling in failure regions. \cite{wu2024ropinn} proposed RoPINN that extends PINN from sampling isolated points to their simply connected neighborhood regions. \cite{tang2023pinns} introduced Adversarial Adaptive Sampling (AAS), an adversarial framework unifying PINNs and optimal transport to push residual distributions toward uniformity via Wasserstein distance. \cite{lau2024pinnacle} presented ACLE, to optimize training points using neural tangent kernel (NTK) to computing the convergence degree for every point.

In addition to changing the random distribution to enhance the accuracy of MC, Quasi-Monte Carlo (QMC) methods \cite{morokoff1995quasi} are proposed to improve efficiency and convergence properties by introducing a quasi-random distribution.Based on aforementioned merits, QMC has received increasing attention in many fields, including optimization \cite{drew2006quas}, finance \cite{goncu2009monte} and particle physics \cite{kleiss2006error}. Furthermore, since QMC lacks randomness and statistical rigor, Randomized Quasi-Monte Carlo methods \cite{l2016randomized} are proposed to add random shifts or perturbations to QMC.

In this paper, we propose the Quasi-Random Physics-Informed Neural Networks (QRPINNs), which randomly sample points from low-discrepancy sequences rather than the input domain. Because of the more evenly distributed points, the QRPINNs have impressive performance in high-dimensional PDEs. Moreover, randomly sampling from deterministic points helps implement QMC successfully in machine learning tasks by improving convergence and efficiency.

Our specific contributions can be summarized as follows:
\begin{enumerate}
    \item We reveal that in high-dimensional PDEs, using low-discrepancy sequences is better than directly sampling from the high-dimensional domain.
    \item We propose the Quasi-Random Physics-Informed Neural Networks, and theoretically prove that its convergence rate is better than PINNs.
    \item We have conducted comprehensive experiments to provide convinced results that demonstrate the performance of Quasi-Random Physics-Informed Neural Networks.
\end{enumerate}

The paper is structured as follows: 
In \cref{Section: Methods}, we briefly introduce Physics-Informed Neural Networks, Monte-Carlo methods and Quasi-Monte Carlo methods; propose Quasi-Random Physics-Informed Neural Networks; and provide the corresponding convergence theorem and its proof. In \cref{Section: experiments}, we compare our QRPINNs with several representative and competitive sampling methods in several equations. In \cref{Section: conclusion}, we conclude the principal findings, acknowledge the study's limitations, and discuss subsequent exploration.

\section{Methods}\label{Section: Methods}
\subsection{Physics-informed neural networks (PINNs)}
We briefly review the physics-informed neural networks (PINNs) \cite{raissi2019physics} in the context of
inferring the solutions of PDEs. Generally, we consider time-dependent PDEs for $\boldsymbol{u}$ taking the form
\begin{equation}
\begin{aligned}
    & \partial_{t}\boldsymbol{u}+\mathcal{N}[\boldsymbol{u}]=0, \quad t \in[0, T],\ \boldsymbol{x} \in \Omega, \\
    & \boldsymbol{u}(0, \boldsymbol{x})=\boldsymbol{g}(\boldsymbol{x}), \quad \boldsymbol{x} \in \Omega, \\
    & \mathcal{B}[\boldsymbol{u}]=0, \quad t \in[0, T],\ \boldsymbol{x} \in \partial \Omega,
\end{aligned}\label{PDE}
\end{equation}
where $\mathcal{N}$ is the differential operator, $\Omega$ is the domain of grid points, and $\mathcal{B}$ is the boundary operator. 

The ambition of PINNs is to approximate the unknown solution $\boldsymbol{u}$ to the PDE system \cref{PDE}, by optimizing a neural network $\boldsymbol{u}^{\theta}$, where $\theta$ denotes the trainable parameters of the neural network. The constructed loss function is:
\begin{equation}\label{PINN}
\mathcal{L}(\theta)=\mathcal{L}_{i c}(\theta)+\mathcal{L}_{b c}(\theta)+\mathcal{L}_r(\theta) ,
\end{equation}
where
\begin{equation}\label{eq: split pinn loss}
\begin{aligned}
& \mathcal{L}_r(\theta)=\frac{1}{N_r} \sum_{i=1}^{N_r}\varepsilon_r(\theta,t_r^i, \boldsymbol{x}_r^i), \quad \varepsilon_r(\theta,t,\boldsymbol{x})=\left|\partial_{t}\boldsymbol{u}^\theta\left(t,\boldsymbol{x}\right)+\mathcal{N}\left[\boldsymbol{u}^\theta\right]\left(t,\boldsymbol{x}\right)\right|^2,\\
& \mathcal{L}_{i c}(\theta)=\frac{1}{N_{i c}} \sum_{i=1}^{N_{i c}}\varepsilon_{ic}\left(\boldsymbol{x}_{i c}^i\right), \quad \varepsilon_{ic}\left(\boldsymbol{x}\right)=\left|\boldsymbol{u}^\theta\left(0, \boldsymbol{x}\right)-\boldsymbol{g}\left(\boldsymbol{x}\right)\right|^2, \\
& \mathcal{L}_{b c}(\theta)=\frac{1}{N_{b c}} \sum_{i=1}^{N_{b c}}\varepsilon_{bc}\left(t_{b c}^i, \boldsymbol{x}_{b c}^i\right), \quad \varepsilon_{bc}\left(t, \boldsymbol{x}\right)=\left|\mathcal{B}\left[\boldsymbol{u}^\theta\right]\left(t, \boldsymbol{x}\right)\right|^2, \\
\end{aligned}
\end{equation}
corresponds to the three equations in \cref{PDE} individually; $\boldsymbol{x}_{i c},\boldsymbol{x}_{b c},\boldsymbol{x}_{r}$ are the sampled points from the initial constraint, boundary constraint, and residual constraint, respectively; $N_{i c},N_{b c},N_{r}$ are the total number of sampled points for each constraint, correspondingly.

\subsection{Monte-Carlo methods}
The Monte Carlo method is a numerical method that uses random sampling to obtain numerical results. It's often used to estimate quantities that are difficult or impossible to compute exactly, especially when dealing with complex systems or high-dimensional problems. Given function $f: 
\mathbb{R}^{d}\rightarrow \mathbb{R}$ , where $d$ is the dimensionality, and the randomly sampling set $\mathcal{X}=\{\boldsymbol{x}_i\}_{i=1}^N$, the integration of $f$ by Monte Carlo methods is:
\begin{equation}
    I\left(f\right):=\int_\Omega f(\boldsymbol{x})\mathrm{d}\boldsymbol{x}\approx I_{\text{MC}}(f) :=\frac{1}{N}\sum_{i=1}^N f(\boldsymbol{x}_i),
\end{equation}
where $\Omega \subset \mathbb{R}^d$ is a definite set and $\mathcal{X}\subset \Omega$. 

Monte Carlo methods have wildly applications in several fields. For example, solving high-dimensional distribution functions in gas dynamics \cite{moss2005direct}, modeling uncertainty in financial markets \cite{glasserman2004monte}, and simulating systems with many coupled degrees of freedom in cellular structures \cite{graner1992simulation}. 

In machine learning, especially in the context of PINNs, Monte Carlo methods are used to approximation the integration of the target equations, because its convergence rate is $\mathcal{O}\left(N^{-1/2}\right)$ (the proof is provided in \cref{Appendix: MC convergence}) which is independent of the dimensionality $d$. 

\subsubsection{Monte Carlo Methods in PINNs}
For a PDE system \cref{PDE}, suppose the analytic solution is $u^*$ (Without loss of generality, here we consider $u\in \mathcal{H}$ that $u:\mathbb{R}^d\rightarrow \mathbb{R}$, where $\mathcal{H}$ is the Hilbert space.) and the output solution from the network is $u^{\theta^*}$ defined as follows:
\begin{equation}
    \theta^* := \arg\min_\theta \mathcal{L}(\theta).
\end{equation}
Note that, every single equation in \cref{eq: split pinn loss} can be regarded as a Monte Carlo approximation, \textit{i.e.} 
\begin{equation}\label{eq: mc trend pinn loss} 
\begin{aligned}
& \frac{1}{N_r} \sum_{i=1}^{N_r}\left|\partial_{t}u\left(t_r^i, \boldsymbol{x}_r^i\right)+\mathcal{N}\left[u\right]\left(t_r^i, \boldsymbol{x}_r^i\right)\right|^2 \rightarrow \int_t\int_{\Omega} \left|\partial_{t}u\left(t, \boldsymbol{x}\right)+\mathcal{N}\left[u\right]\left(t, \boldsymbol{x}\right)\right|^2\mathrm{d}\boldsymbol{x}\mathrm{d}t,\\
& \frac{1}{N_{i c}} \sum_{i=1}^{N_{i c}}\left|u\left(0, \boldsymbol{x}_{i c}^i\right)-g\left(\boldsymbol{x}_{i c}^i\right)\right|^2 \rightarrow \int_\Omega \left|u\left(0, \boldsymbol{x}\right)-g\left(\boldsymbol{x}\right)\right|^2\mathrm{d}\boldsymbol{x}, \\
& \frac{1}{N_{b c}} \sum_{i=1}^{N_{b c}}\left|\mathcal{B}\left[u\right]\left(t_{b c}^i, \boldsymbol{x}_{b c}^i\right)\right|^2\rightarrow \int_t\int_{\partial\Omega}\left|\mathcal{B}\left[u\right]\left(t, \boldsymbol{x}\right)\right|^2\mathrm{d}\boldsymbol{x}\mathrm{d}t, \\
\end{aligned}
\end{equation}
with $N_r, N_{ic}, N_{bc} \rightarrow \infty$, if the integrand is Riemann-integrable functions in $[0,T]\times\Omega$ (see \cite{kuipers2012uniform}).  Suppose that
\begin{equation}\label{eq: loss_int}
\begin{aligned}
        \mathcal{L}_{int}(\theta):= & \int_t\int_{\Omega} \left|\partial_{t}u^\theta\left(t, \boldsymbol{x}\right)+\mathcal{N}\left[u^\theta\right]\left(t, \boldsymbol{x}\right)\right|^2\mathrm{d}\boldsymbol{x}\mathrm{d}t \\
        & +\int_\Omega \left|u^\theta\left(0, \boldsymbol{x}\right)-g\left(\boldsymbol{x}\right)\right|^2\mathrm{d}\boldsymbol{x}\\
        & + \int_t\int_{\partial\Omega}\left|\mathcal{B}\left[u^\theta\right]\left(t, \boldsymbol{x}\right)\right|^2\mathrm{d}\boldsymbol{x}\mathrm{d}t. \\
        \theta^*_\infty := & \arg \min_\theta \mathcal{L}_{int}(\theta).
\end{aligned}
\end{equation}
Then,
\begin{equation}
    \|u^*-u^{\theta^*}\|_{L^2(\boldsymbol{\Omega})}\leq \|u^*-u^{\theta^*_\infty}\|_{L^2(\boldsymbol{\Omega})} +\|u^{\theta^*_\infty}-u^{\theta^*}\|_{L^2(\boldsymbol{\Omega})},
\end{equation}
where $\boldsymbol{\Omega}=[0,T]\times\Omega$. However, as $u^{\theta^*}$ is dependent on the randomly sampled points. Like Monte Carlo methods, the error analysis of the above equation requires an expectation:
\begin{equation}
    \mathbb{E}\left[\|u^*-u^{\theta^*}\|_{L^2(\boldsymbol{\Omega})}\right]\leq \|u^*-u^{\theta^*_\infty}\|_{L^2(\boldsymbol{\Omega})} +\mathbb{E}\left[\|u^{\theta^*_\infty}-u^{\theta^*}\|_{L^2(\boldsymbol{\Omega})}\right].
\end{equation}

The first term $\|u^*-u^{\theta^*_\infty}\|_{L(\boldsymbol{\Omega})}$ depends on the size of the neural network, and the smoothness of $u$; the second term $\mathbb{E}\left[\|u^{\theta^*_\infty}-u^{\theta^*}\|_{L(\boldsymbol{\Omega})}\right]$ depends on how well the summation represents the integration, \textit{i.e.} the error of the Monte Carlo methods. If $u$ is smooth and the neural network can hold that $\|u^*-u^{\theta^*_\infty}\|_{L^2(\boldsymbol{\Omega})} \ll \mathbb{E}\left[\|u^{\theta^*_\infty}-u^{\theta^*}\|_{L^2(\boldsymbol{\Omega})}\right]$, then $    \mathbb{E}\left[\|u^*-u^{\theta^*}\|_{L^2(\boldsymbol{\Omega})}\right]\leq \mathbb{E}\left[\|u^{\theta^*_\infty}-u^{\theta^*}\|_{L^2(\boldsymbol{\Omega})}\right]$. Combining with \cref{theorem: convergence}, we can claim that, under specific requirements of the neural networks, the convergence rate of PINN is the same as the order of the quadrature. Thus, the PINNs can suffer from the curse of dimensionality theoretically.

\begin{theorem}\label{theorem: convergence}
    If the PDE system is well defined and the solution $u \in \mathcal{H}$ is a smooth function, then there exists a constant r, such that for $u^\theta \in B_r(u^{\theta^*})$ where $B_r(x)$ is a compact ball with radius $r$,  $\mathbb{E}\left[\|u^{\theta^*_\infty}-u^{\theta^*}\|_{L^2(\boldsymbol{\Omega})}\right] = \mathcal{O}\left(\mathbb{E}\left[\left|\mathcal{L}_{int}(\theta)-\mathcal{L}(\theta)\right|\right]\right)$. 
\end{theorem}
\cref{theorem: convergence} reveals that with enough accurate approximation, the error convergence in PINNs is dominated by the error convergence of the quadrature. We provide the proof of \cref{theorem: convergence} in \cref{appendix: convergence rate}. 

Therefore, some strategies of adaptive sampling \cite{wu2023comprehensive,tang2023pinns,gao2023failure} are proposed to enhance the convergence rate of the MC as well as the performance of PINNs by changing either the distribution of the sampling points or decreasing the variance of the target function which is another principle element of Monte Carlo methods (see \cref{eq:mc error}).

However, a fundamental difficulty of MC stems from the requirement that the points should be independently random samples. But in practice, there is no method to generate independent random samples concretely. On the other hand, although $P_{\mathcal{D}}(x=x_i| x\in[0,1])=0 \ \forall x_i \in [0,1] $ under any possibility distribution functions $\mathcal{D}(x)$, $P_{\mathcal{D}}(\|x-x_i|<\epsilon|x\in[0,1])>0$ damages the confidence and accuracy of the Monte Carlo quadrature.

\subsection{Quasi Monte-Carlo methods}
Quasi Monte Carlo (QMC) method is proposed to accelerate the convergence rate of Monte-Carlo methods in high-dimensional integrations with low-discrepancy sequences which are chosen deterministically and
methodically. QMC methods have a convergence rate $\mathcal{O}(N^{-(1-\epsilon)})$ where $\epsilon$ depends on the specific sequence and in general $\epsilon\in (0,1)$ (see \cite{asmussen2007stochastic, hung2024review}). Furthermore, \cite{sloan1998quasi} showed that when dealing with functions where the behavior varies across different dimensions with distinct weights, the convergence rate of quasi-Monte Carlo sampling can be reduced to $\mathcal{O}(N^{-1/p})$ for some $p\in [1,2]$, which corresponds to $\epsilon=(p-1)/p\in [0,0.5]$.

Notably, except the theoretical convergence rate, the performance of MC and QMC is significant different in practical: the experimental error of MC is always close to the claimed convergence rate \cite{quarteroni2006numerical} while the practical convergence rate of QMC always larger than that of MC \cite{Paskov1995,gurov2016energy,berblinger1991monte}.
\begin{figure}[ht]
  \centering
  \subfigure[\label{qmc-sin-d2} $\sin d=2$]{\includegraphics[width=0.23\linewidth]{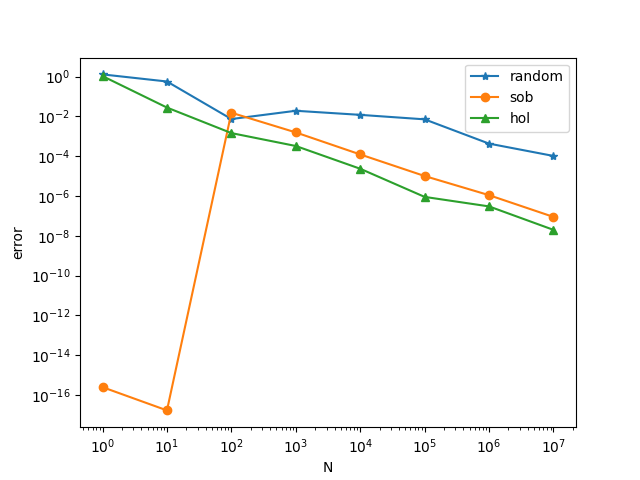}} 
  \subfigure[\label{qmc-sin-d10} $\sin d=10$]{\includegraphics[width=0.23\linewidth]{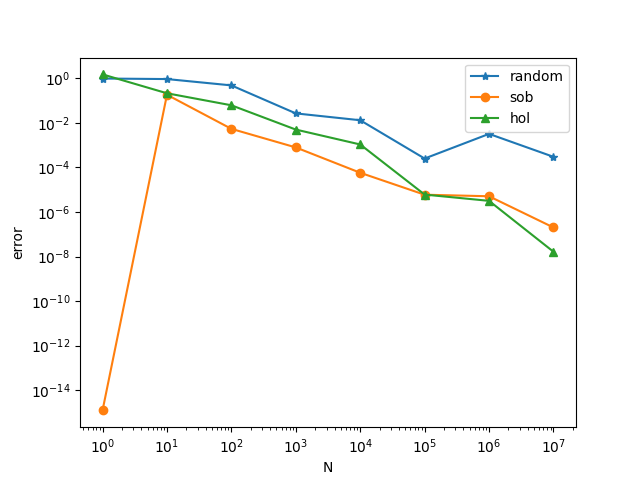}}
  \subfigure[\label{qmc-sin-d20} $\sin d=20$]{\includegraphics[width=0.23\linewidth]{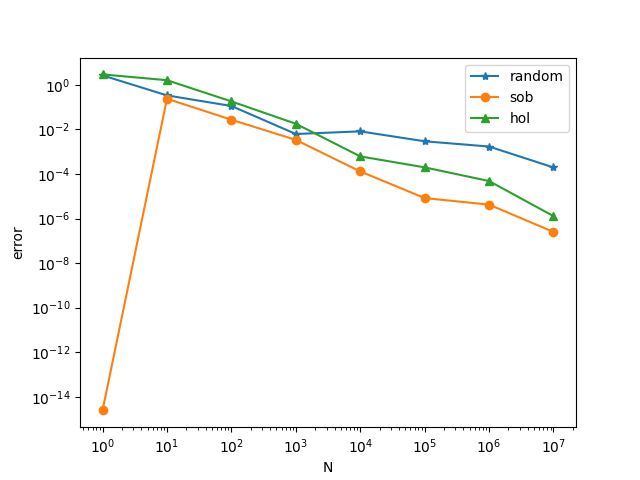}}
  \subfigure[\label{qmc-sin-d100} $\sin d=100$]{\includegraphics[width=0.23\linewidth]{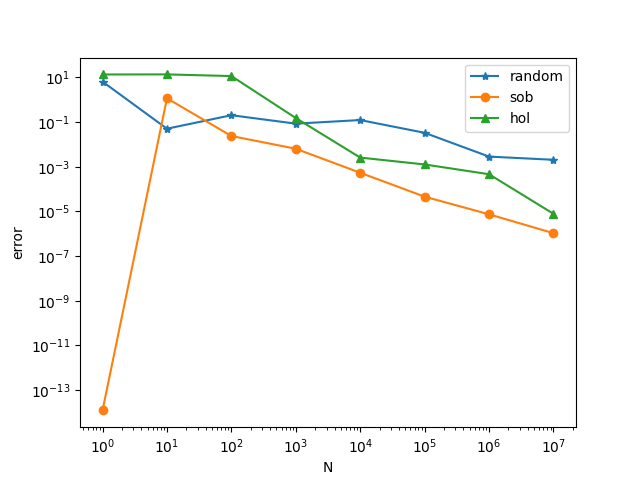}}\\
  \subfigure[\label{qmc-exp-d2} $\exp d=2$]{\includegraphics[width=0.23\linewidth]{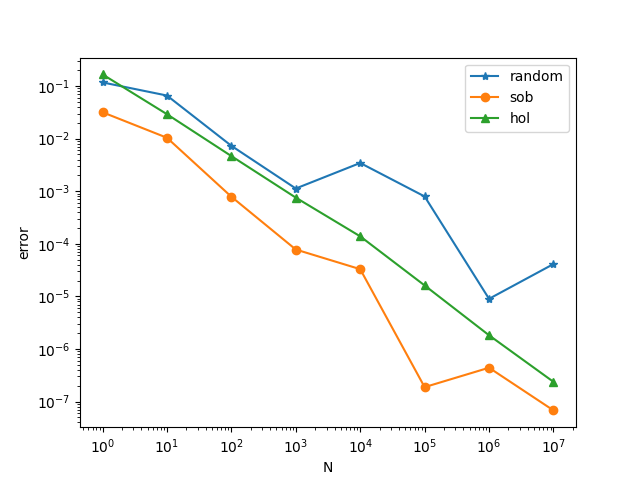}} 
  \subfigure[\label{qmc-exp-d10} $\exp d=10$]{\includegraphics[width=0.23\linewidth]{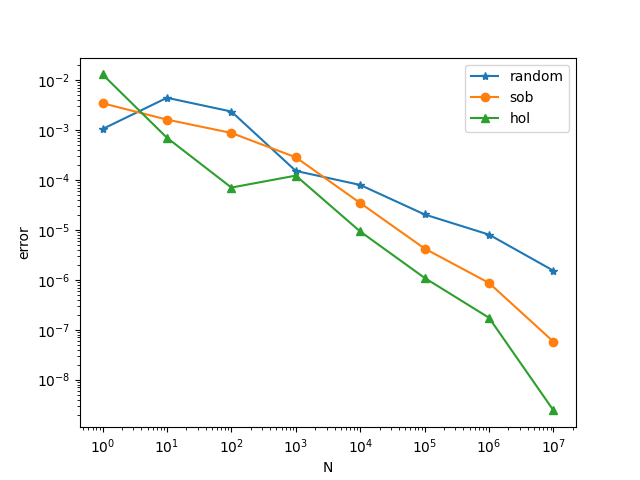}}
  \subfigure[\label{qmc-exp-d20} $\exp d=20$]{\includegraphics[width=0.23\linewidth]{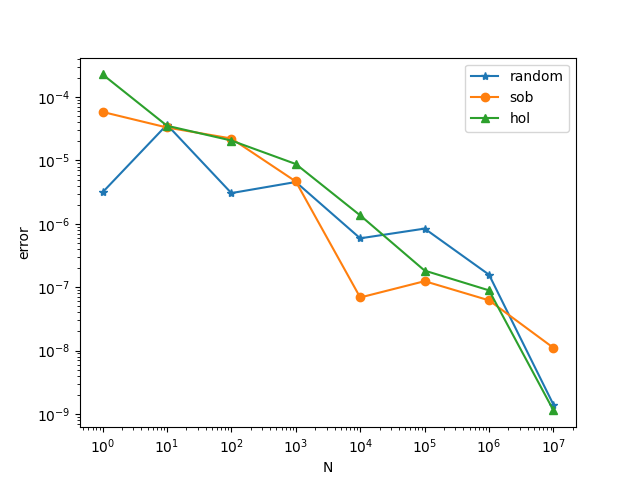}}
  \subfigure[\label{qmc-exp-d100} $\exp d=100$]{\includegraphics[width=0.23\linewidth]{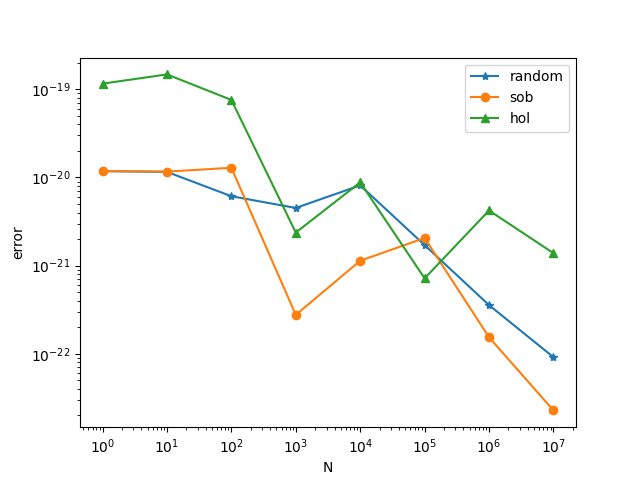}}
  \caption{QMC integration for different dimensionality with different $N$.}
  \label{fig:QMC_quadrature}
\end{figure}
\subsubsection{Quasi Monte Carlo in PINNs}\label{section: qmc_in_pinns}
In the context of PINNs, QMC has already been investigated to replace the MC. However, although comprehensive experiments in \cite{wu2023comprehensive} show that QMC methods perform better than MC methods in PINNs, those experiments also show that the QMC is worse than those adaptive sampling methods. We question the performance of QMC methods for the following reasons:

\paragraph{1. high-dimensional cases} QMC is proposed to accelerate high-dimensional integrations, and has been widely applied in high-dimensional problems rather than low-dimensional problems \cite{fang2002some,l2009quasi,case2025comparative}. Thus, solving low-dimensional PDEs cannot exploit the capabilities of QMC and is unfair for QMC.

We conducted a series of experiments with $f(\boldsymbol{x})=\sum_{i=1}^d\sin(2\pi x_i), \ \boldsymbol{x}\in[0,1]^d$ and $f(\boldsymbol{x})=\prod_{i=1}^d \exp(-x_i), \ \boldsymbol{x} \in [0,1]^d$ for different $d$ in \cref{fig:QMC_quadrature}.  Although some advanced papers \cite{doerr2013constructing,steinerberger2019nonlocal,clement2024heuristic} provide novel approaches to generate advanced low-discrepancy sequences, the Helton sequences \cite{halton1960efficiency} and Sobol' sequences \cite{sobol1967distribution} are implemented in the QMC in our experiments due to the generalization and simplicity\footnote{The details about Helton sequences and Sobol' sequences are provided in \cref{appendix: Sequences}}. The theoretically convergence rate of both Helton and Sobol' sequences is $\mathcal{O}((\log N)^d /N)$ \cite{niederreiter1992random}. But the results of \cref{fig:QMC_quadrature} align the statement of aforementioned QMC researchers: although for some $N$ and $d$,  $\frac{\left(\log N\right) ^d}{N} > \frac{1}{\sqrt{N}}$, the error of QMC is still better than MC. Furthermore, \cref{fig:mc points}, which demonstrates the distribution of uniform grids, random sampling points, Halton sequences and Sobol' sequences, reveals that the points generated by random sampling are relatively scattered in a arbitrary manner, while Halton and Sobol' samplings produce point distributions that are more evenly spread out, enabling them to explore a larger portion of the domain. Obviously, this phenomenon becomes more pronounced in higher dimensions.

\paragraph{2. Random sampling}
Random sampling a batch of the total dataset is a principle strategy in machine learning to improve the performance \cite{bottou2010large} and efficiency \cite{el2019unrestricted}. The deterministic sequence pollutes these features and thus influences the performance. We argue that one should randomly sample points from the deterministic sequence \textit{i.e.} regard the deterministic sequence as the total dataset and randomly sample a fixed number of points every epoch. We call it the Random Quasi Monte Carlo method (RQMC)\footnote{Pay attention to the Randomized Quasi-Monte Carlo which put a random perturbation on the deterministic sequence.}.

However, different to MC, random sampling destroys the deterministic integration error of QMC. Herein, the error convergence of RQMC should be verified, especially when compared with MC. Notably, since every epoch will resample a batch of points, the number of untrained points is close to zero after a finite number of epochs with high possibility.

\begin{theorem}\label{theorem: RQMC}
    For a given deterministic sequence $P_{N_{total}}=\{\boldsymbol{x}_i\}_{i=1}^{N_{total}}$, if $P_{N}=\{\boldsymbol{\tilde{x}}_j\}_{j=1}^{N}$ is a subset sampled $N$ points from  $P_{N_{total}}$. Suppose  $E_{\text{QMC}} =\left|\int f(\boldsymbol{x})\mathrm{d}\boldsymbol{x}-\frac{1}{N_{total}}\sum_{i=1}^{N_{total}} f(\boldsymbol{x}_i)\right|= \mathcal{O}(N_{total}^{-(1-\epsilon)})$ for some $\epsilon \in (0,1)$ and $E_{\text{RQMC}} =\left|\int f(\boldsymbol{x})\mathrm{d}\boldsymbol{x}-\frac{1}{N}\sum_{i=j}^{N} f(\boldsymbol{\tilde{x}}_j)\right|$, then
    \begin{equation}\label{eq: E_rqmc}
        E_{\text{RQMC}}\leq \frac{dD}{2k N_{total}}+dD \left(1-k\right)+C N_{total}^{-(1-\epsilon)},
    \end{equation}
    for constant $C,D>0$,$k=N/N_{total}$ and $d$ is the dimensionality.
\end{theorem}

\begin{proof}

Given a point set $P=\{\boldsymbol{x}_i\}$, the star discrepancy $D^*_N(P)$ of $P$ is defined as follows:
\begin{definition}
    \begin{equation}
        D^*_N(P)=\sup_{u_1,\cdots,u_d\in[0,1]}\left|\frac{A(J;P)}{N}-\lambda_d(J)\right|,
    \end{equation}
    where $J=\prod_{i=1}^d[0,u_i)$ and $\lambda_d(J)=\prod_{i=1}^d u_i$, and
    \begin{equation}
                A(J,P)=\sum_{n=1}^N \mathbbm{1}_{J}(\boldsymbol{x}_n), \quad \mathbbm{1}(\boldsymbol{x}) \text{ is the characteristic function}.
    \end{equation}
\end{definition}
According to the error analysis of QMC\footnote{The proof is in \cref{Appendix: QMC convergence}}, we have.
\begin{equation}\label{eq: inequality}
    E_{\text{QMC}}\leq V(f)D_{N_\text{total}}^*\left(P_{N_{total}}\right).
\end{equation}

In the following proof, without loss of generality, we consider the 1D case. The following lemma is useful and its proof can be found in \cite{kuipers2012uniform}.

\begin{lemma}\label{lamma:discrepancy} If $0 \leq x_1 \leq x_2 \leq \cdots \leq x_N \leq 1$, then 
    \begin{equation}
        D_N^*\left(x_1, \ldots, x_N\right)=\frac{1}{2 N}+\max _{1 \leq n \leq N}\left|x_n-\frac{2 n-1}{2 N}\right|.
    \end{equation}
\end{lemma}

\cref{lamma:discrepancy} shows that the star discrepancy $D_N^*(x_1, \dots, x_N)$ is determined by the maximum absolute deviation between the points $x_n$ and the ideal uniform grid positions $\frac{2n - 1}{2N}$.

Suppose $P_{N_{total}}=\{x_i\}_{i=1}^{N_{total}}$ is the set of a given deterministic sequence in QMC that satisfies $E_{\text{QMC}}=\mathcal{O}\left(N_{total}^{-(1-\epsilon)}\right)$ for a constant $\epsilon\in (0,1)$ dependent on the given deterministic sequence,  then $D_{N_{total}}^*(P_{N_{total}})= \mathcal{O}(N_{total}^{-(1-\epsilon)})$. 

Since $\frac{1}{2N} = \omega\left(N^{-(1-\epsilon)}\right), \quad \forall \epsilon \in (0,1)$, $\max _{1 \leq n \leq N_{total}}\left|x_n-\frac{2 n-1}{2 N_{total}}\right|\leq CN_{total}^{-(1-\epsilon)}$ \textit{i.e.} 
\begin{equation}\label{eq: ineq N_total}
\left|x_n-\frac{2 n-1}{2 N_{total}}\right|\leq CN_{total}^{-(1-\epsilon)}, \ \forall n \in [1,N_{total}] .    
\end{equation}

Now we consider the subset $P_{N}=\{\boldsymbol{\tilde{x}}_j\}_{j=1}^{N}$. At first, we analyze the dominant term $\max _{1 \leq j \leq N}\left|\tilde{x}_j-\frac{2 j-1}{2 N}\right|$. Since $P_N$ is randomly sampled from $P_{N_{total}}$, the term we actually consider is $\max_{P_N\subset P_{N_{total}}}\max _{1 \leq j \leq N}\left|\tilde{x}_j-\frac{2 j-1}{2 N}\right|$. 

Then, we want to construct the subset $P_N^*=\arg\max_{P_N\subset P_{N_\text{total}}}\max _{1 \leq j \leq N}\left|\tilde{x}_j-\frac{2 j-1}{2 N}\right|$. Let $k_j$ be the index of $\tilde{x}_j$ in the original sequence $\left(1 \leq k_1<k_2<\cdots<k_N \leq N_{total}\right)$. By \cref{eq: ineq N_total}, we have $x_{k_j} \in\left[\frac{2 k_j-1}{2 N_{total}}-C N_{total}^{-(1-\epsilon)}, \frac{2 k_j-1}{2 N_{total}}+C N_{total}^{-(1-\epsilon)}\right]$. So
\begin{equation}
\begin{aligned}
        \left|\tilde{x}_j-\frac{2 j-1}{2 N}\right| & = \left|x_{k_j}-\frac{2 j-1}{2 N}\right|\\
        &\leq \left|\frac{2 k_j-1}{2 N_{total}}-\frac{2 j-1}{2 N}\right|+C N_{total}^{-(1-\epsilon)}\\
        & = \frac{\left|(2 k_j-1)N-(2 j-1)N_{total}\right|}{2NN_{total}}+C N_{total}^{-(1-\epsilon)}
\end{aligned}
\end{equation}
Thus, constructing $P_N^*$ is equivalent to maximize the term $\left|(2 k_j-1)N-(2 j-1)N_{total}\right|$. Since $j \leq k_j\leq N_{total}-N+j$, the $P_N^*$ is either the $k_j=j$ or $k_j=N_{total}-N+j$ \textit{i.e.} either the first $N$ points or the last $N$ points. Herein, $\forall P_N\subset P_{N_{total}}$,
\begin{equation}
\begin{aligned}
        \max _{1 \leq j \leq N}\left|\tilde{x}_j-\frac{2 j-1}{2 N}\right| &\leq \frac{(2 N-1)\left(N_{total}-N\right)}{2 N N_{total}}+C N_{total}^{-(1-\epsilon)},\\
        & \leq 1-\frac{N}{N_{total}}+C N_{total}^{-(1-\epsilon)} .
\end{aligned}
\end{equation}

Then 
\begin{equation}
    D_N^*(P_N)=\frac{1}{2 N}+\max _{1 \leq j \leq N}\left|\tilde{x}_j-\frac{2 j-1}{2 N}\right|\leq \frac{1}{2 N}+ 1-\frac{N}{N_{total}}+C N_{total}^{-(1-\epsilon)}.
\end{equation}
Let $N=kN_{total}$, $k \in (0,1]$,
\begin{equation}
    D_N^*(P_N)\leq \frac{1}{2k N_{total}}+ 1-k+C N_{total}^{-(1-\epsilon)}.
\end{equation}
Then, for $d$ dimensional sequences,
\begin{equation}
    D_N^*(P_N)\leq \frac{d}{2k N_{total}}+d \left(1-k\right)+C N_{total}^{-(1-\epsilon)}.
\end{equation}
Recall \cref{eq: inequality}, we get that 
\begin{equation}
    E_\text{RQMC}\leq V(f)D_N^*(P_N)=\frac{dD}{2k N_{total}}+\frac{dD}{2} \left(1-k\right)+C N_{total}^{-(1-\epsilon)},
\end{equation}
for $D=V(f)$. If $k=1$, then $E_\text{RQMC}=\mathcal{O}\left(N_{total}^{-(1-\epsilon)}\right)$.

\end{proof}
\textbf{Remark}. Suppose $N$ (\textit{i.e.} $N_{total}$ in \cref{theorem: RQMC}) is the total number of the dataset, and $N_b$ is the batch size for every epoch, and $\epsilon=N-kN$, then after $s$ epochs, the possibility of $\epsilon$ points that have never been sampled is 
\begin{equation}
P(\mathcal{X}= \epsilon)=\frac{\binom{N}{\epsilon} \sum_{i=0}^{N-\epsilon}(-1)^i\binom{N-\epsilon}{i}\left(\binom{\left(N-\epsilon\right)-i}{N_b}\right)^s}{\left(\binom{N}{N_b}\right)^s}.
\end{equation}
So
\begin{equation}\label{eq: qmc_possibility}
P(\mathcal{X}= 0)=\frac{\sum_{i=0}^N(-1)^i\binom{N}{i}\left(\binom{N-i}{N_b}\right)^s}{\left(\binom{N}{N_b}\right)^s},
\end{equation}
which is almost 1, when $s=20$, $N_b=0.1N$. Thus, $k=1$ is practical in RQMC, so the error convergence of RQMC is almost the same as the the error convergence of QMC
\paragraph{3. Combining} Under the randomly sampling idea, QMC can also combine with aforementioned adaptive sampling methods. Therefore, it is unfair to only compare the performance between QMC and adaptive sampling methods. Here we provide the algorithm of combining RAD and RQMC in Algorithm \ref{algorithm:rad_quasi} as an example. Compared with the original RAD algorithm (Algorithm \ref{algorithm:rad}), The Algorithm \ref{algorithm:rad_quasi} calculates the probability distribution on the whole low-discrepancy sequences. 
\vspace{0.2cm}
\begin{algorithm}[H]
    \SetAlgoLined
    \KwIn{Dimensionality $d$, Number of epochs $s$.}
    \KwOut{the output of PINN $u$}
     \vspace{0.1cm}
    Generate the point pool $\mathcal{P}_{pool}$ by $d$ and $N_{total}$.
     
     \vspace{0.1cm}
    Generate $\mathcal{P}$ by uniformly sampling from $\mathcal{P}_{pool}$ 
        
    Train the PINN for a certain number of iterations;

    \vspace{0.1cm}
    \For{$i = 2,\cdots,s$}{
        Generate $\mathcal{P}$ by randomly sampling from $\mathcal{P}_{pool}$ based on the probability distribution function:
        \begin{equation}
            p(\mathbf{x}) \propto \frac{\varepsilon_r(\theta,\mathbf{x})}{\mathbb{E}\left[\varepsilon_r(\theta,\mathbf{x})\right]}+1,
        \end{equation}
        
        Train the PINN for a certain number of iterations;
    }
     \vspace{0.1cm}
    \Return the output of PINN
    \vspace{0.1cm}
    \caption{RAD for low-discrepancy points}\label{algorithm:rad_quasi}
\end{algorithm}

\begin{figure}[ht]
  \centering 
  \subfigure[\label{uniform} uniform]{\includegraphics[width=0.23\linewidth]{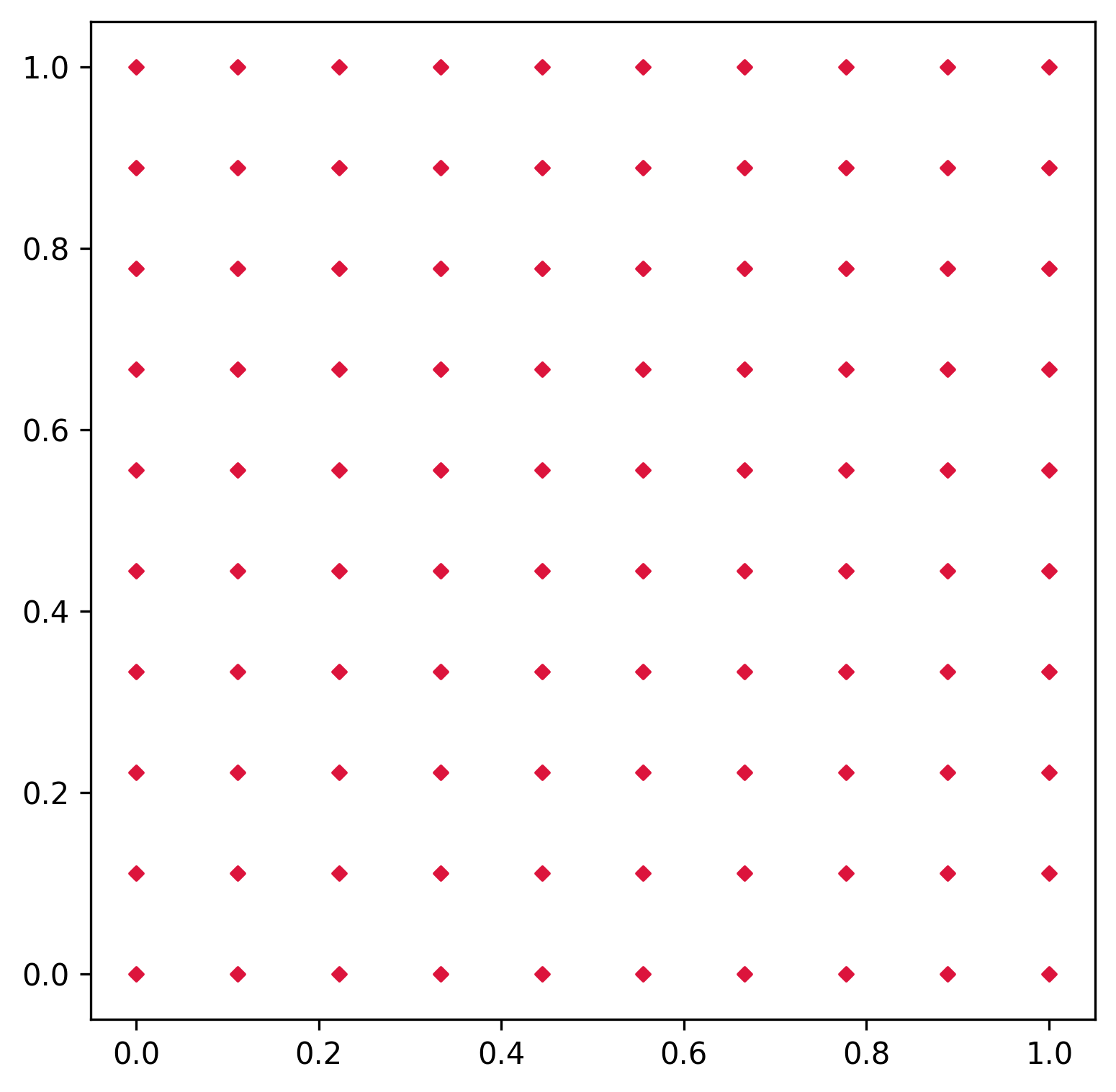}} 
  \subfigure[\label{random} random]{\includegraphics[width=0.23\linewidth]{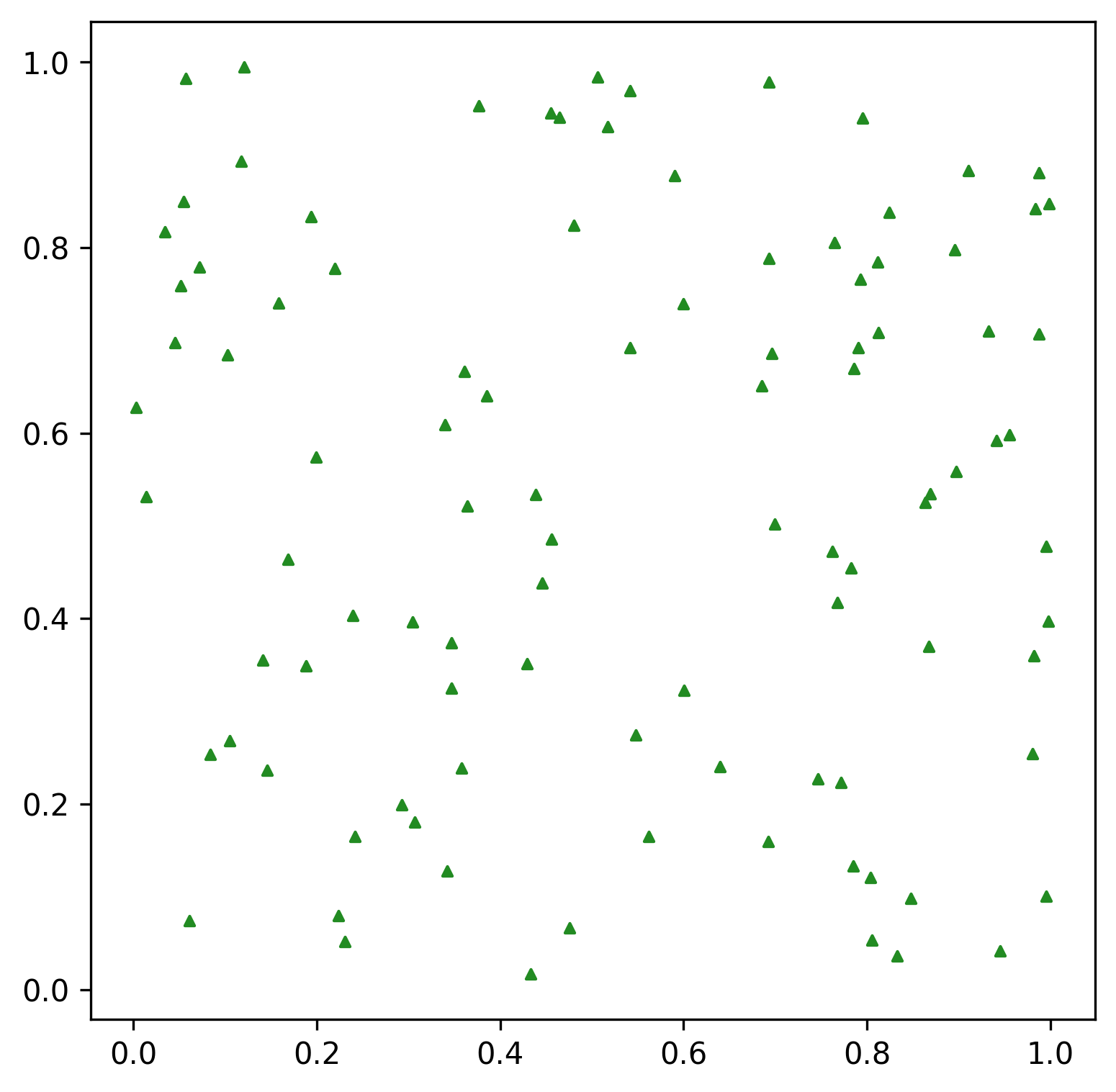}}
  \subfigure[\label{halton} Halton]{\includegraphics[width=0.23\linewidth]{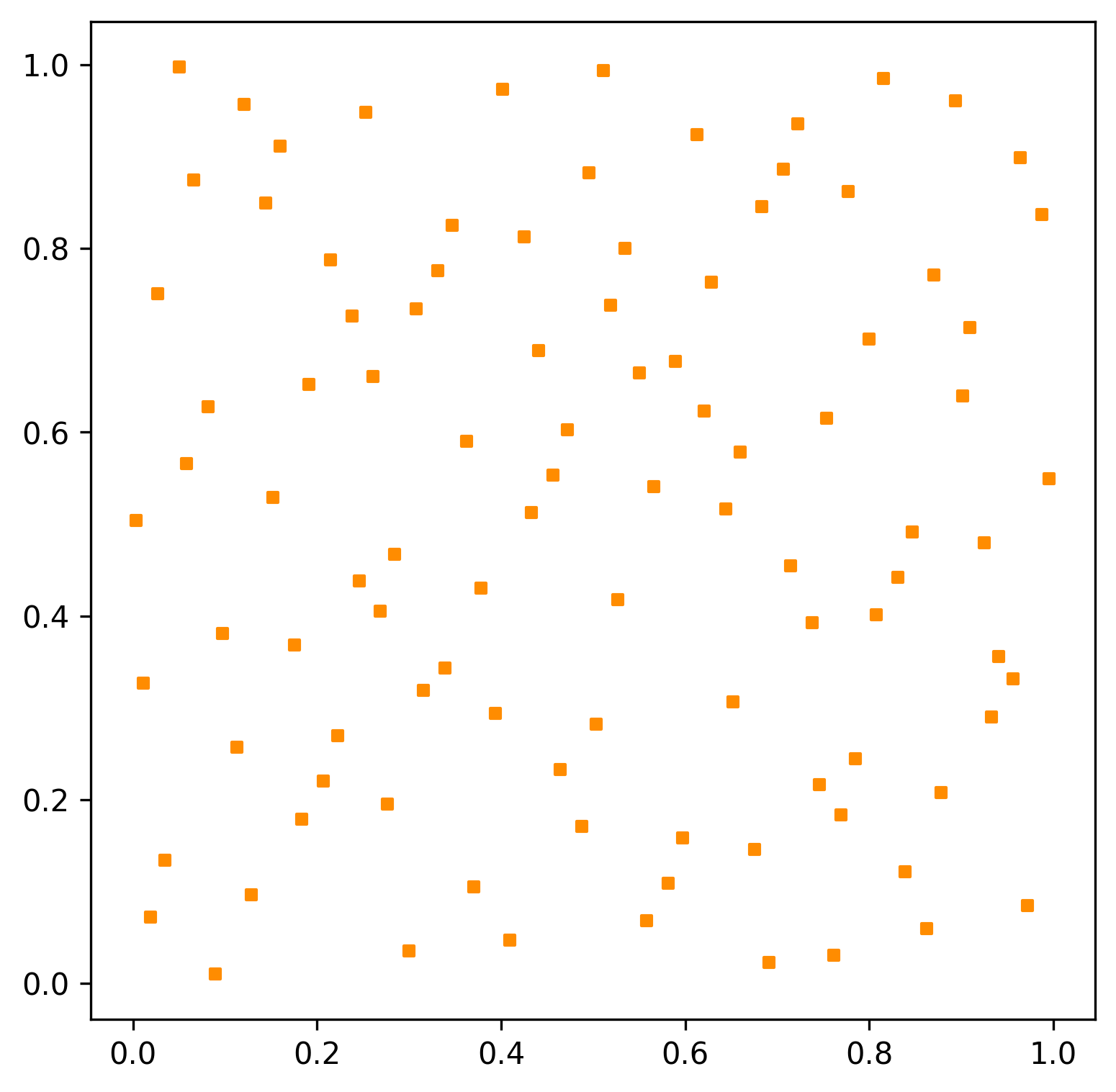}}
  \subfigure[\label{sobol} Sobol']{\includegraphics[width=0.23\linewidth]{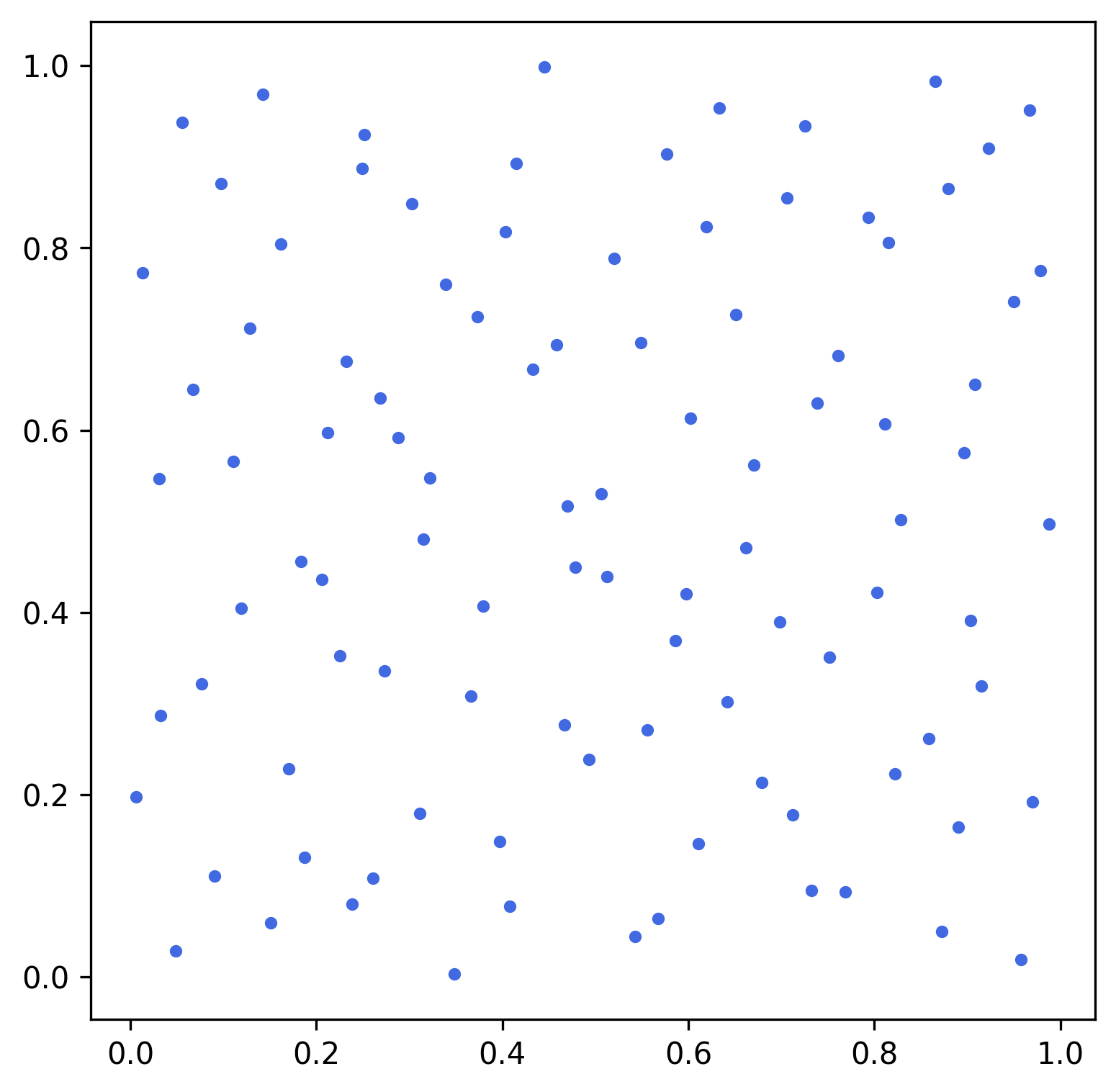}}\\
    \subfigure[\label{uniform10000} uniform]{\includegraphics[width=0.23\linewidth]{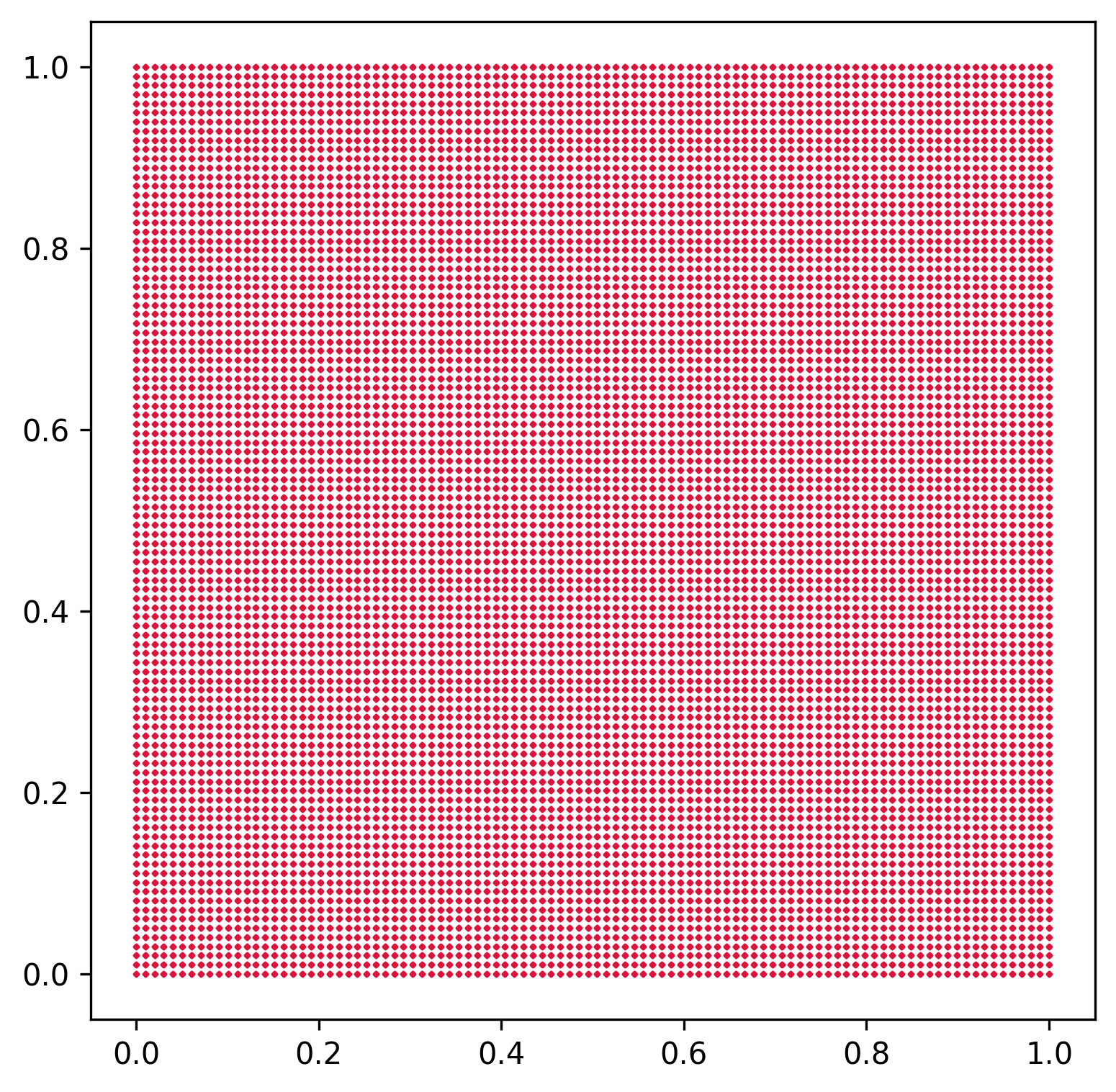}} 
  \subfigure[\label{random10000} random]{\includegraphics[width=0.23\linewidth]{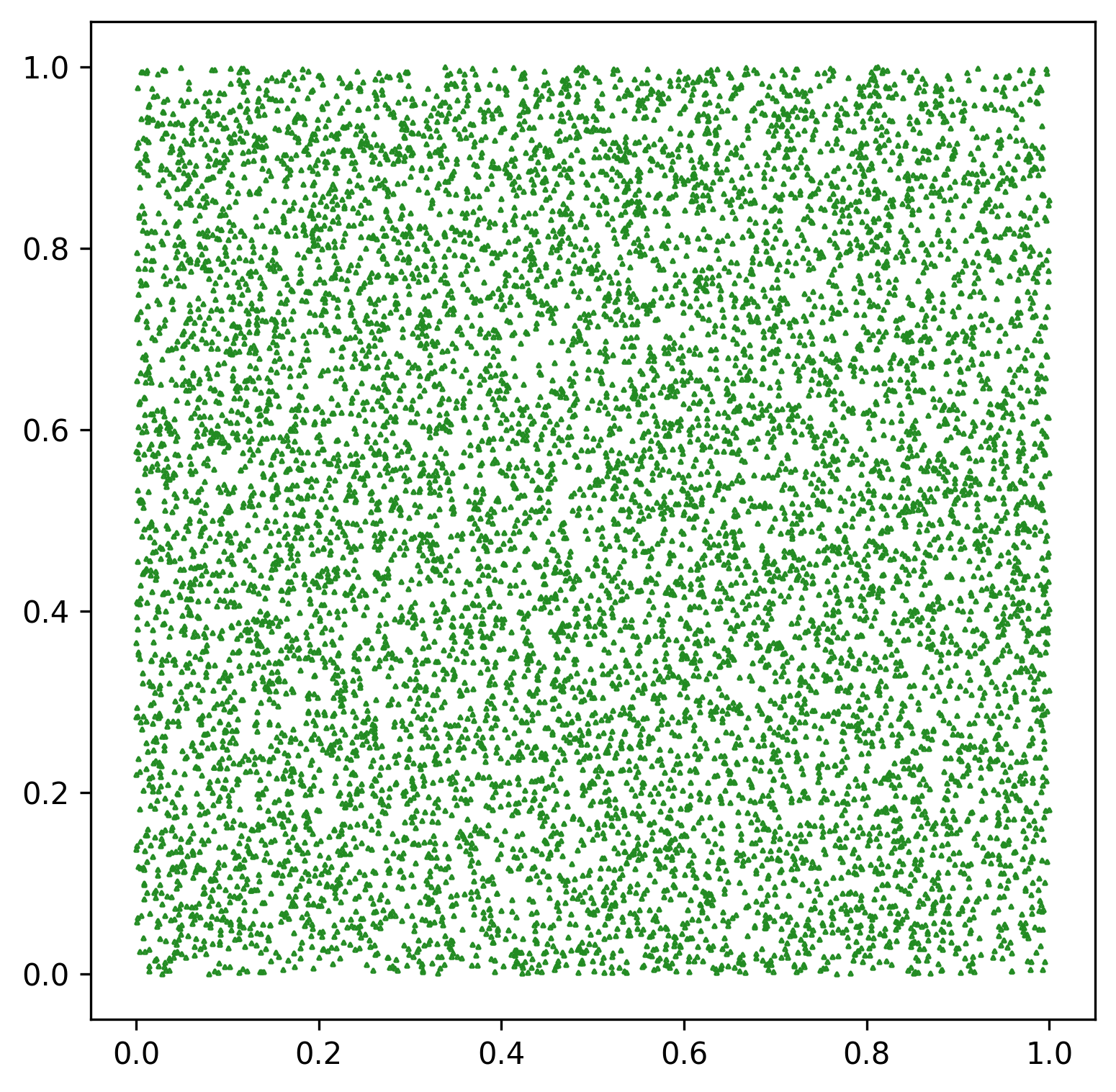}}
  \subfigure[\label{halton10000} Halton]{\includegraphics[width=0.23\linewidth]{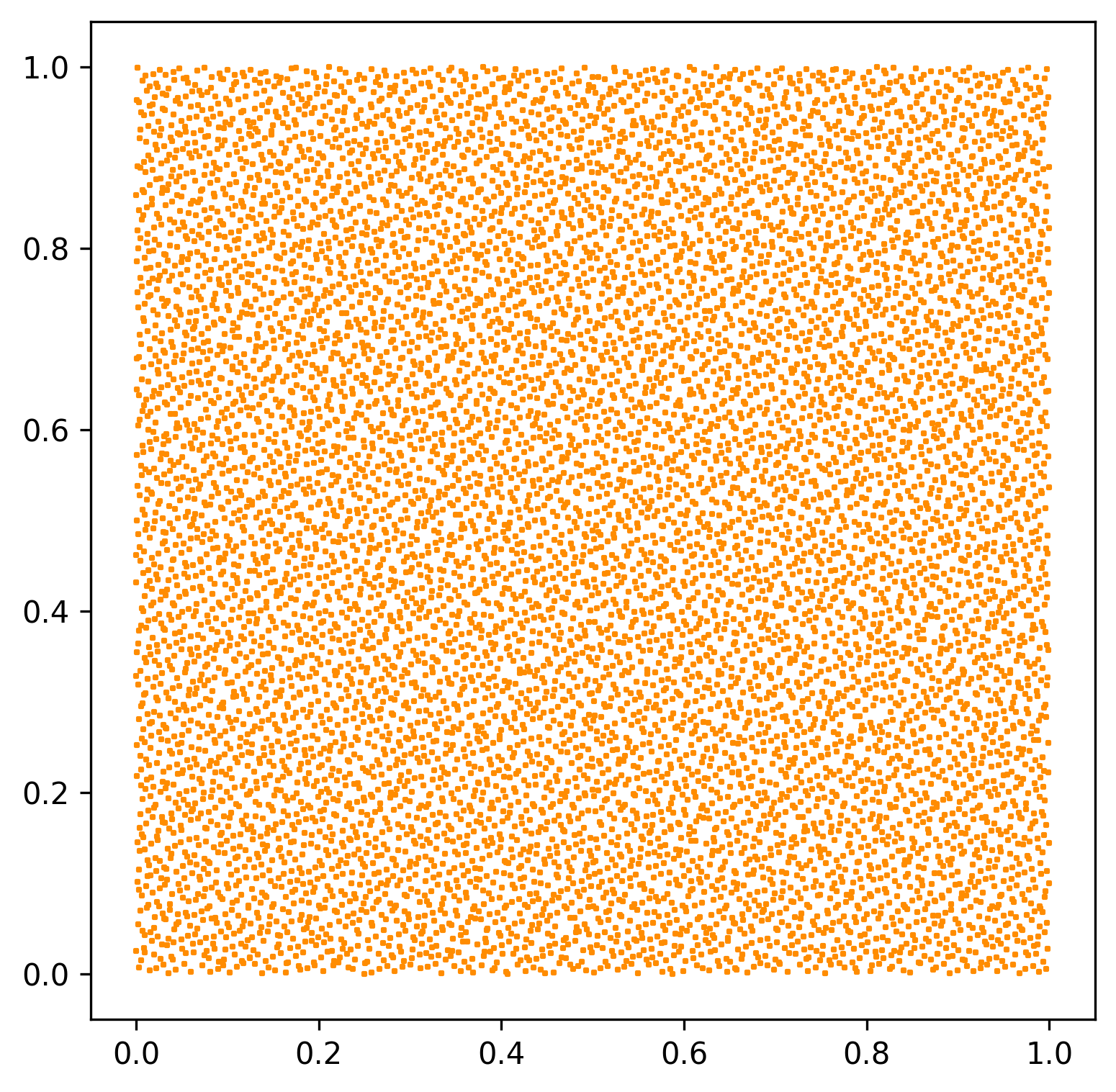}}
  \subfigure[\label{sobol10000} Sobol']{\includegraphics[width=0.23\linewidth]{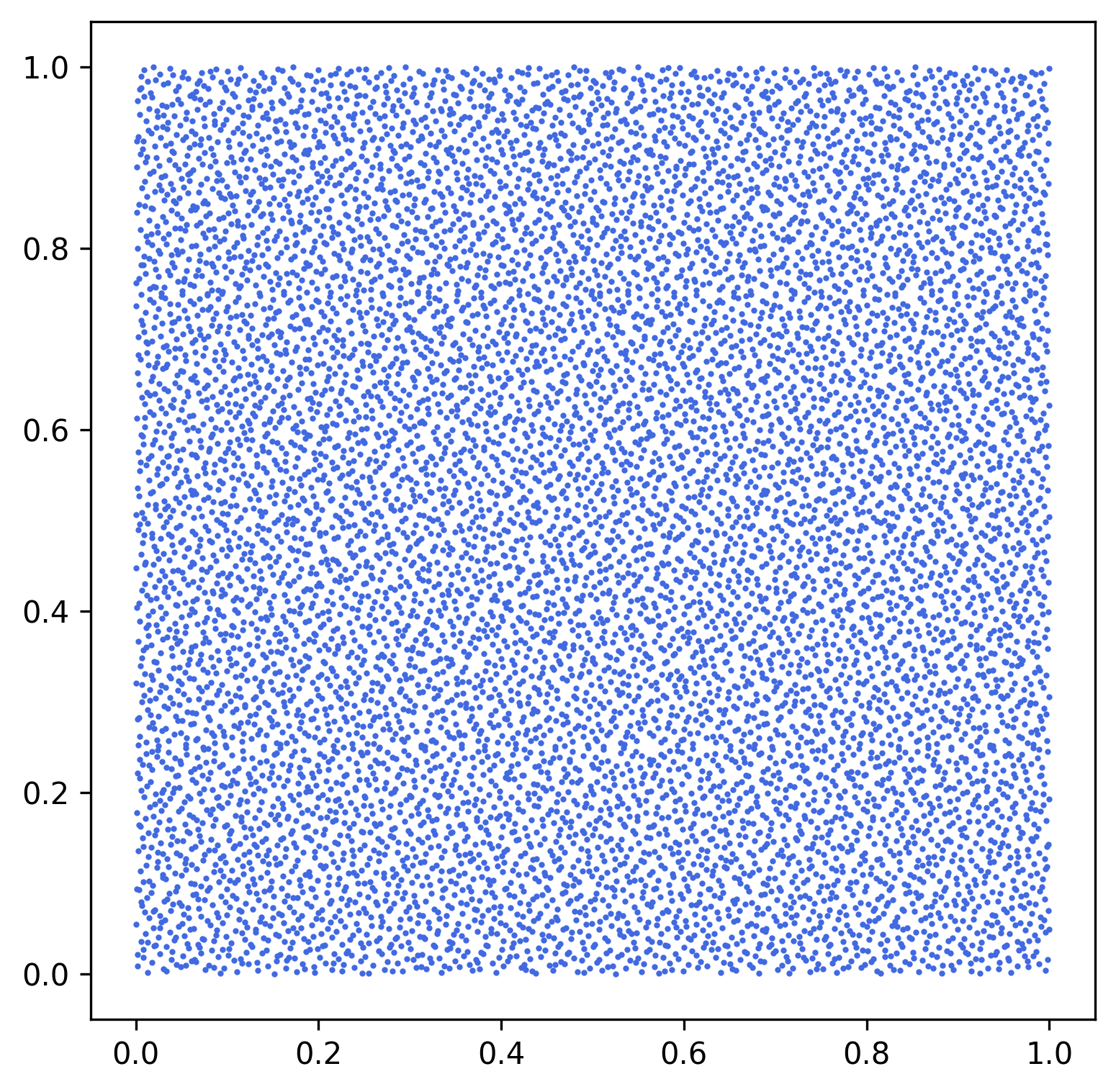}}
  \caption{Examples of 100 (first row) and 10000 (second row) points generated in $[0,1]^2$.}
  \label{fig:mc points}
\end{figure}
\section{Experiments}\label{Section: experiments}
In the experiments, except the vanilla way (\textit{i.e.} randomly sampling from the domain) to generate the input points, we introduce two adaptive sampling methods as baselines: 1) RAD \cite{wu2023comprehensive} is the classical adaptive sampling methods, and 2) ACLE \cite{lau2024pinnacle} is the advanced adaptive sampling methods. The details of the aforementioned methods are in \cref{Appendix: adaptive sampling methods}. Notably, due to the complexity of ACLE, ACLE cannot be implemented in $d=100$ even under the estimation proposed by \cite{lau2024pinnacle}.

Firstly, we consider three high-dimensional problems and conducted experiments with different dimensionality and some other hyperparameters in \cref{section: exp_high_pde}. Then we explore the cost of generating those sequences and the relationship between the accuracy in \cref{section: cost} and the scale of sampling pool in \cref{section: sacle}. Finally, to verify the performance of combining QRPINNs and adaptive sampling methods, we conducted experiments on ablation study in \cref{section: ablation}; to explore the limits of the dimensionality that can be effectively solved, we introduce a high-dimensional benchmark from \cite{shi2024stochastic} in \cref{section: cost}.

In tables, we mark out the minimum relative error by blue and we regard experiments with errors greater than $10^{-1}$ as failed results. The available code and hyperparameters are in GitHub at \url{https://github.com/DUCH714/quasi_sampling_PINN}.

\subsection{Solving high-dimensional PDEs}\label{section: exp_high_pde}
In this section, we verify three high-dimensional problems: the steady Poisson's equations \cref{eq: poisson}, the steady Allen-Cahn equations\cref{eq: ac}, and the steady Sine-Gordon equations \cref{eq: sg}
\paragraph{Poisson's equations}
    
We consider the classical high-dimensional Poisson's equation:
\begin{equation}\label{eq: poisson}
    \Delta u =f, \quad \boldsymbol{x}\in [-1,1]^d,
\end{equation}
If the solution $u(\boldsymbol{x})=e^{-\alpha \|\boldsymbol{x}\|^2_2}$, then $f=2\alpha(2\alpha \|\boldsymbol{x}\|^2_2-d) e^{-\alpha \|\boldsymbol{x}\|_2^2}$. Because of the distribution of the solution has locality, Poisson equations are a classical example in adaptive sampling to verify the performance of the sampling methods. 

\begin{table}[ht]
\caption{Poisson's equation}
\label{table:poisson}
\begin{center}
\begin{adjustbox}{width=\columnwidth, center}
\begin{tabular}{lllllll}
\multicolumn{1}{c}{\bf $d$ } &\multicolumn{1}{c}{\bf $\alpha$ } & \multicolumn{1}{c}{\bf Vanilla} & \multicolumn{1}{c}{\bf RAD} &\multicolumn{1}{c}{\bf ACLE} &\multicolumn{1}{c}{\bf Halton} & \multicolumn{1}{c}{\bf Sobol}
\\ \toprule 
\multirow{3}{*}{3}    
     & 1   & $3.96e-04\pm 2.27e-04$ & $2.00e-04\pm 9.83e-05$ & \cellcolor{blue!25}$9.90e-05\pm 3.41e-05$ & $6.27e-04\pm 4.03e-04$ & $2.85e-04\pm 1.24e-04$ \\
     & 10  & $3.24e-03\pm 3.39e-03$ & $3.11e-03\pm 1.84e-03$ & $3.44e-03\pm 2.67e-03$ & $1.04e-03\pm 3.30e-04$ & \cellcolor{blue!25}$8.92e-04\pm 2.25e-04$ \\
     & 100 & $1.28e-01\pm 2.98e-02$ & \cellcolor{blue!25}$5.51e-02\pm 1.22e-02$ & $7.68e-02\pm 2.70e-02$ & $2.35e-01\pm 5.19e-03$ & $1.48e-01\pm 4.67e-02$ \\
     \midrule
\multirow{3}{*}{10}    
   & 1   & $8.09e-04\pm 2.34e-05$ & $1.04e-03\pm 3.27e-04$ & \cellcolor{blue!25}$7.98e-04\pm 4.85e-05$ & $8.38e-04\pm 9.97e-05$ & $8.23e-04\pm 4.49e-05$ \\
   & 10  & \cellcolor{blue!25}$1.74e-02\pm 4.10e-03$ & $3.31e-02\pm 1.80e-02$ & $1.89e-02\pm 2.93e-03$ & $2.03e-02\pm 6.18e-03$ & $1.95e-02\pm 4.73e-03$ \\
   & 100 & $1.19e+00\pm 2.30e-01$ & $1.86e+00\pm 7.17e-01$ & $1.04e+00\pm 3.86e-02$ & $1.04e+00\pm 3.30e-02$ & $5.78e+00\pm 4.28e+00$ \\
   \midrule
\multirow{2}{*}{100}  
   & 0.1   & $3.15e-03\pm 6.13e-05$ & $3.19e-03\pm 2.16e-05$ & \multicolumn{1}{c}{-} & \cellcolor{blue!25}$3.11e-03\pm 9.43e-06$ & $3.13e-03\pm 1.89e-05$ \\
   & 1  & \cellcolor{blue!25}$3.10e-02\pm 1.70e-04$ & $3.13e-02\pm 2.49e-04$ & \multicolumn{1}{c}{-} & $3.12e-02\pm 8.16e-05$ & $3.11e-02\pm 1.41e-04$ \\
\bottomrule
\end{tabular}
\end{adjustbox}
\end{center}
\end{table}

\cref{table:poisson} validates the efficiency of adaptive sampling methods in solving low-dimensional Poisson's equations. On the contrary, when dealing with high-dimensional problems, the chosen methods become invalid, because of the complex structure in high-dimensional spaces. However, since directly randomly sampling from high-dimensional region is difficult to capture the features, combining QRPINNs with adaptive sampling methods can mitigate this challenge, resulting in better accuracy.. Those experiments are discussed in \cref{section: ablation}.

Furthermore, the adaptive sampling methods can improve the performance because it can capture the local features and make the sampling points concentrate to the locality. However, if there are enough points to capture those local features, the efficiency of adaptive sampling methods is under scrutiny. Herein, we conducted the experiments of different number of training points $N_{points}$. \cref{fig:poisson points} demonstrates the results and indicates that the effect of adaptive sampling methods can be replaced by increasing $N_{points}$. Furthermore, in \cref{poisson_n_points}, the relative error of RAD and ACLE increases as $N_{points}$ increases when $N_{points}$ greater than a specific value and reveals that an excessive concentration of points in a specific region can lead to overfitting.  

\begin{figure}[ht]
  \centering
\subfigure[\label{poisson_n_points}]{\includegraphics[width=0.23\linewidth]{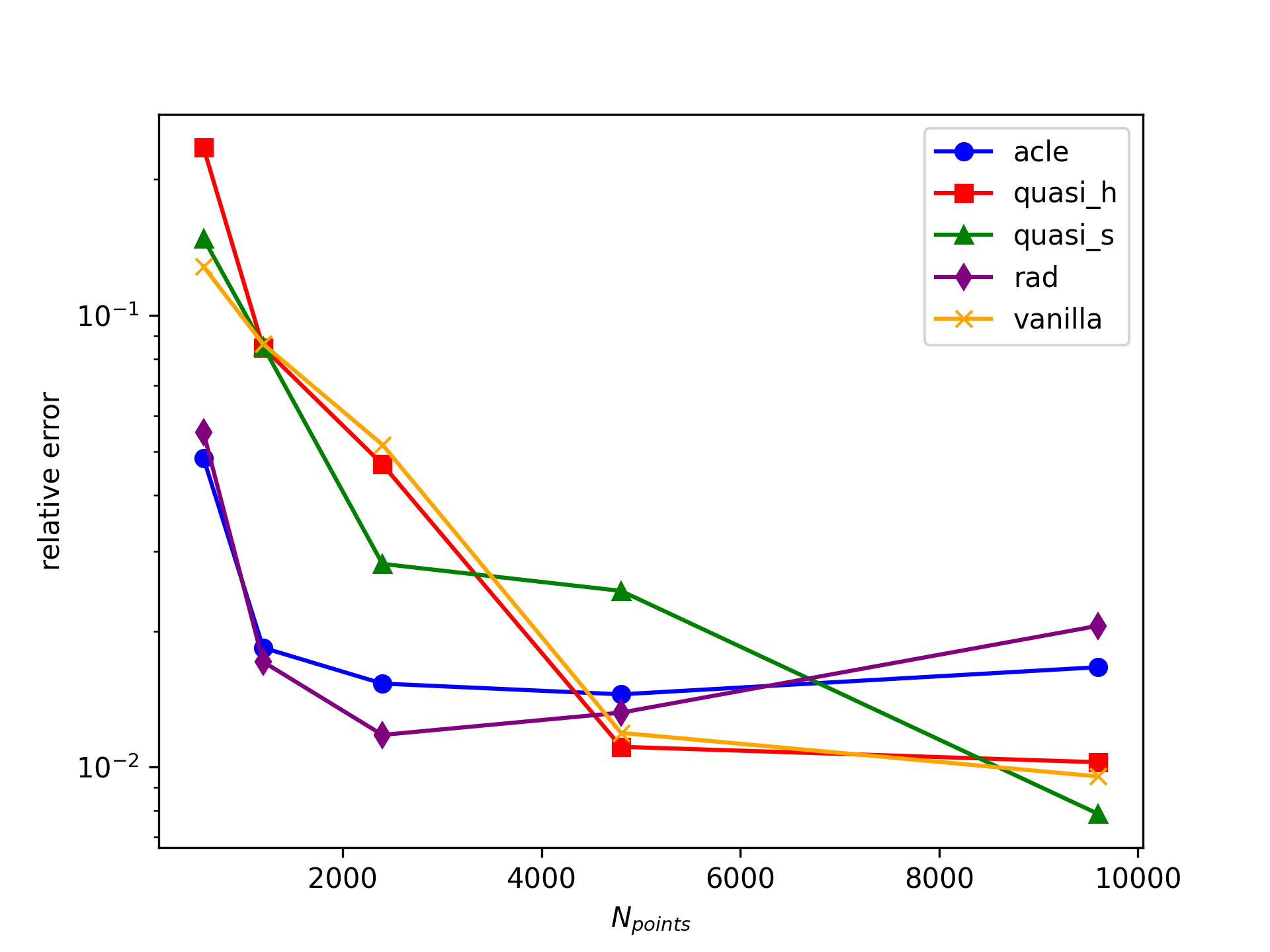}}
  \subfigure[\label{rad_alpha100} RAD]{\includegraphics[width=0.23\linewidth]{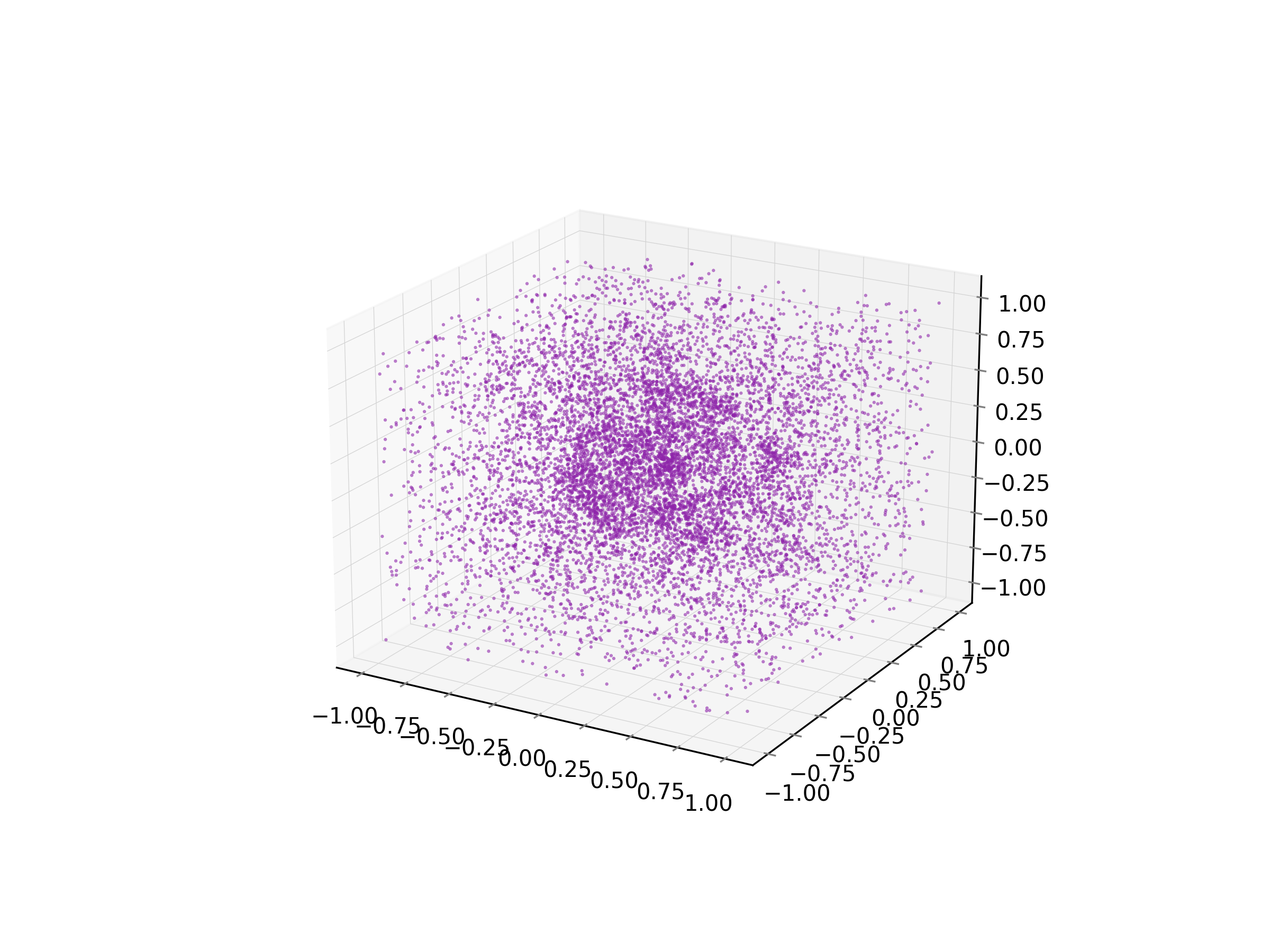}}
  \subfigure[\label{acle_alpha100} ACLE]{\includegraphics[width=0.23\linewidth]{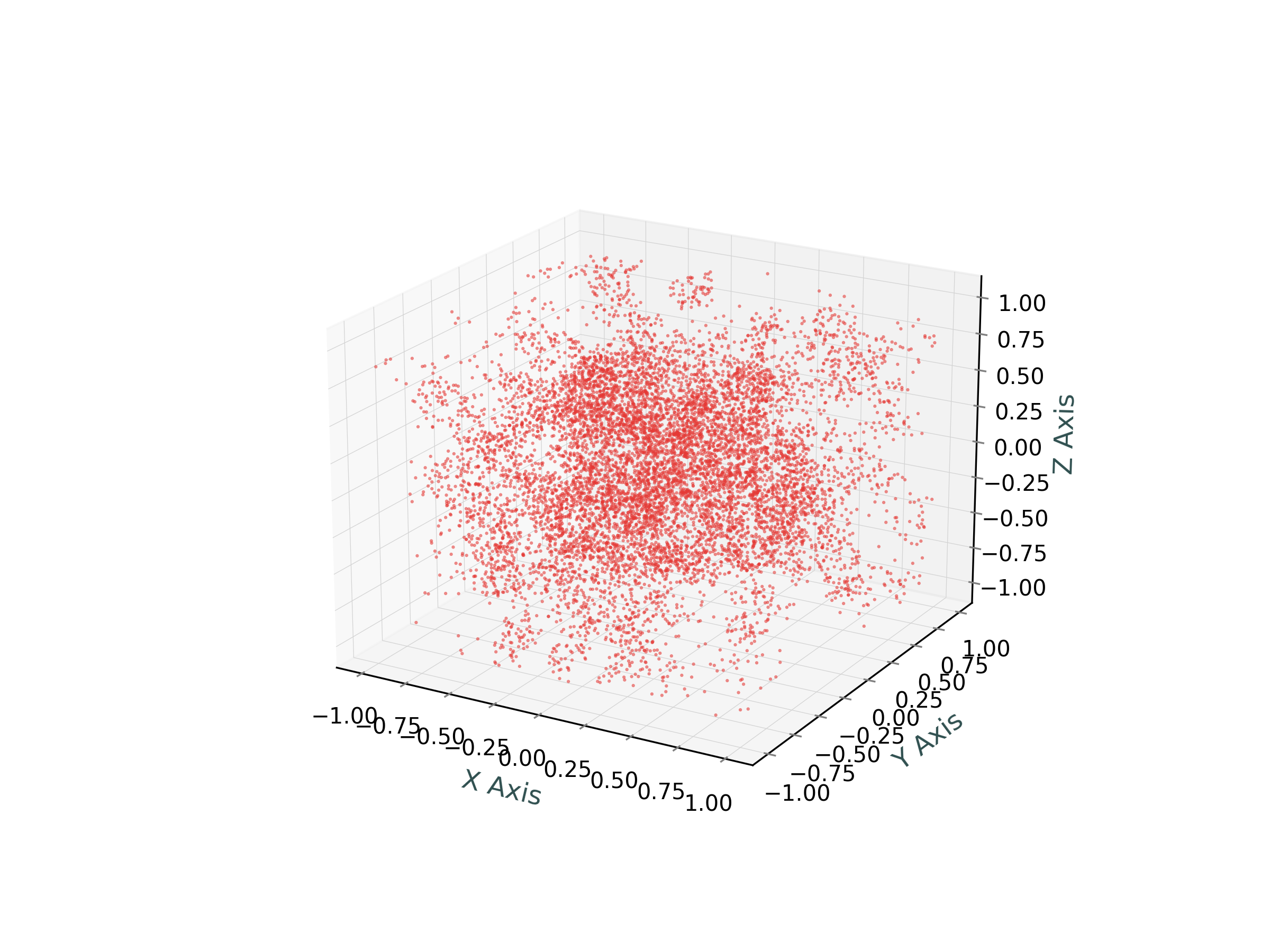}}
  \subfigure[\label{quasi_alpha100} Sobol]{\includegraphics[width=0.23\linewidth]{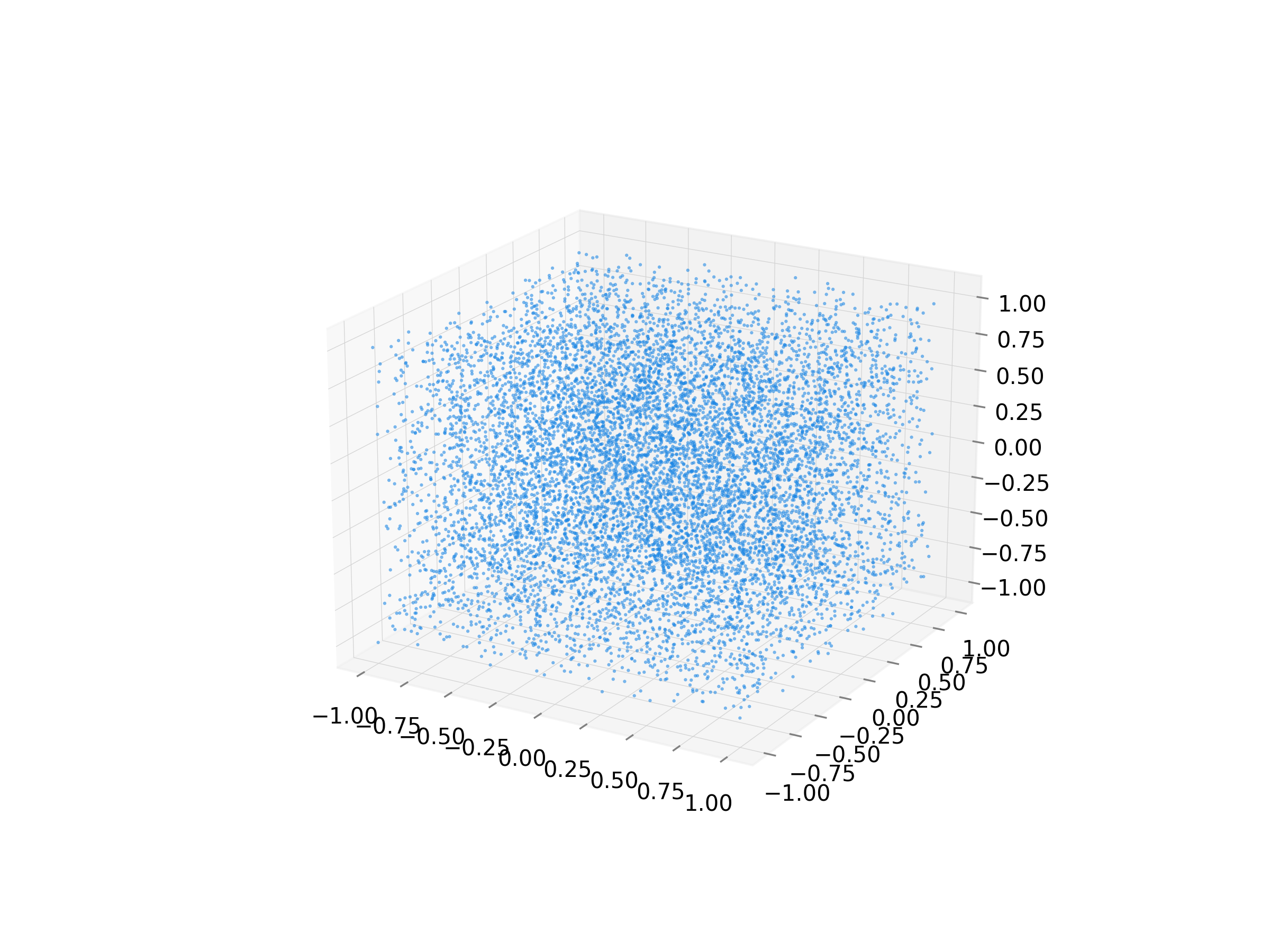}}
  \caption{Poisson's equations with $d=3$ and $\alpha=100$. \subref{poisson_n_points} presents the relationship between the relative error and $N_{points}$. \subref{rad_alpha100}, \subref{acle_alpha100}, and \subref{quasi_alpha100} shows the distribution of training points for RAD, ACLE, and Sobol respectively.}
  \label{fig:poisson points}
\end{figure}

\paragraph{Allen-Cahn equations}
We consider the steady Allen-Cahn equations:
\begin{equation}\label{eq: ac}
    \Delta u + u-u^3=f, \boldsymbol{x}\in [-1,1]^d,
\end{equation}
we use the two-body interaction exact solution with $c_i\sim \mathcal{N}(0, 1)$:
\begin{equation}
    u(\boldsymbol{x})=\left(1-\frac{1}{d}\|\boldsymbol{x}\|_2^2\right)\left(\sum_{k=1}^{d-1}c_k\sin\left(x_k+\cos\left(x_{k+1}\right)+x_{k+1}\sin\left(x_k\right))\right)\right)
\end{equation}
Let $A:=\sum_{k=1}^{d-1}c_k\sin\left(x_k+\cos\left(x_{k+1}\right)+x_{k+1}\sin\left(x_k\right))\right)$, $B:=1-\frac{1}{d} \|\boldsymbol{x}\|^2_2$,
then 
\begin{equation}
    f(\boldsymbol{x})=B\Delta A+A\Delta B+\nabla A^T\nabla B +AB-A^3B^3.
\end{equation}
\begin{figure}[ht]
  \centering
\subfigure[\label{Vanilla_ac3} Vanilla]{\includegraphics[width=0.19\linewidth]{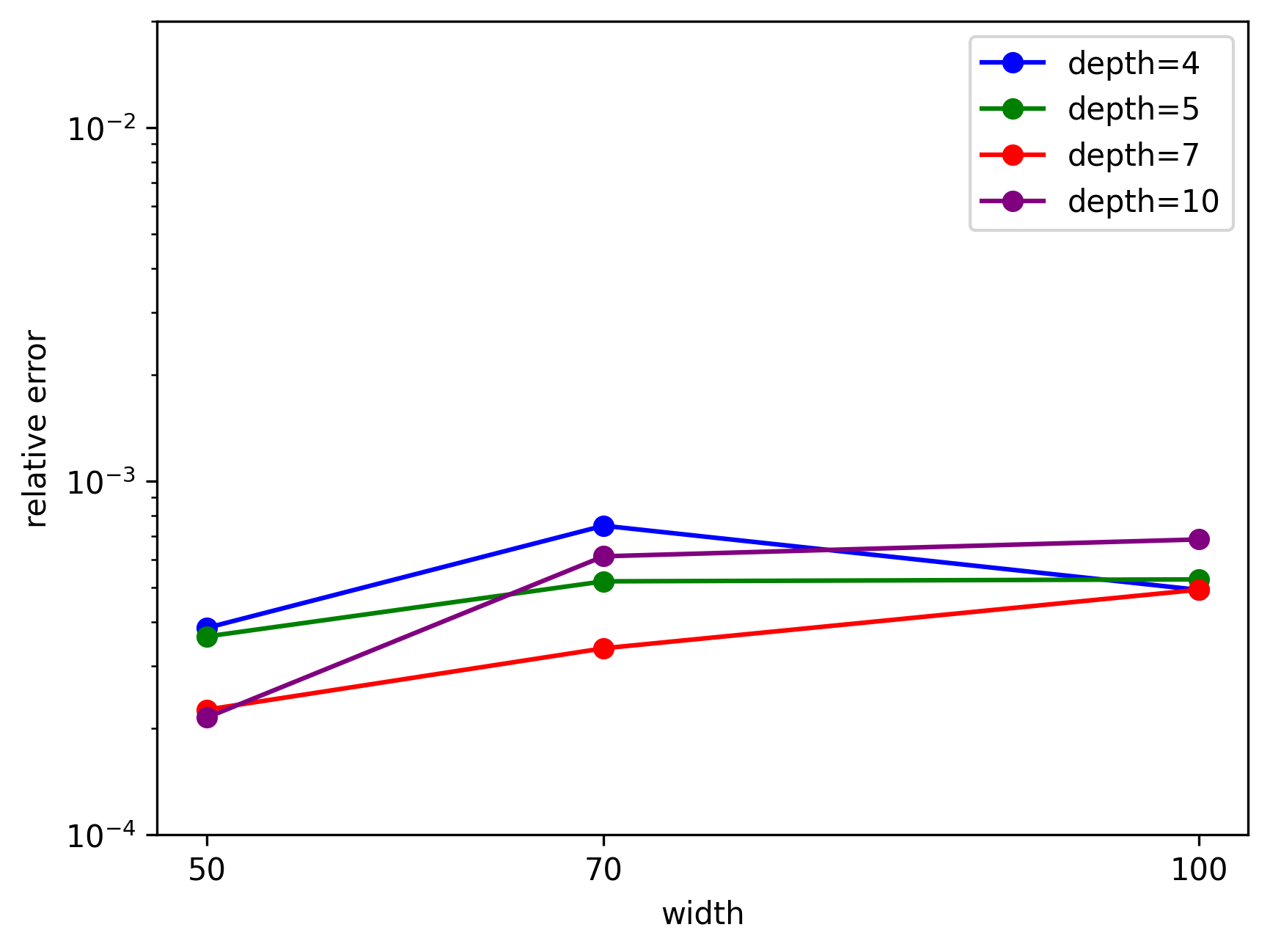}}
  \subfigure[\label{rad_ac3} RAD]{\includegraphics[width=0.19\linewidth]{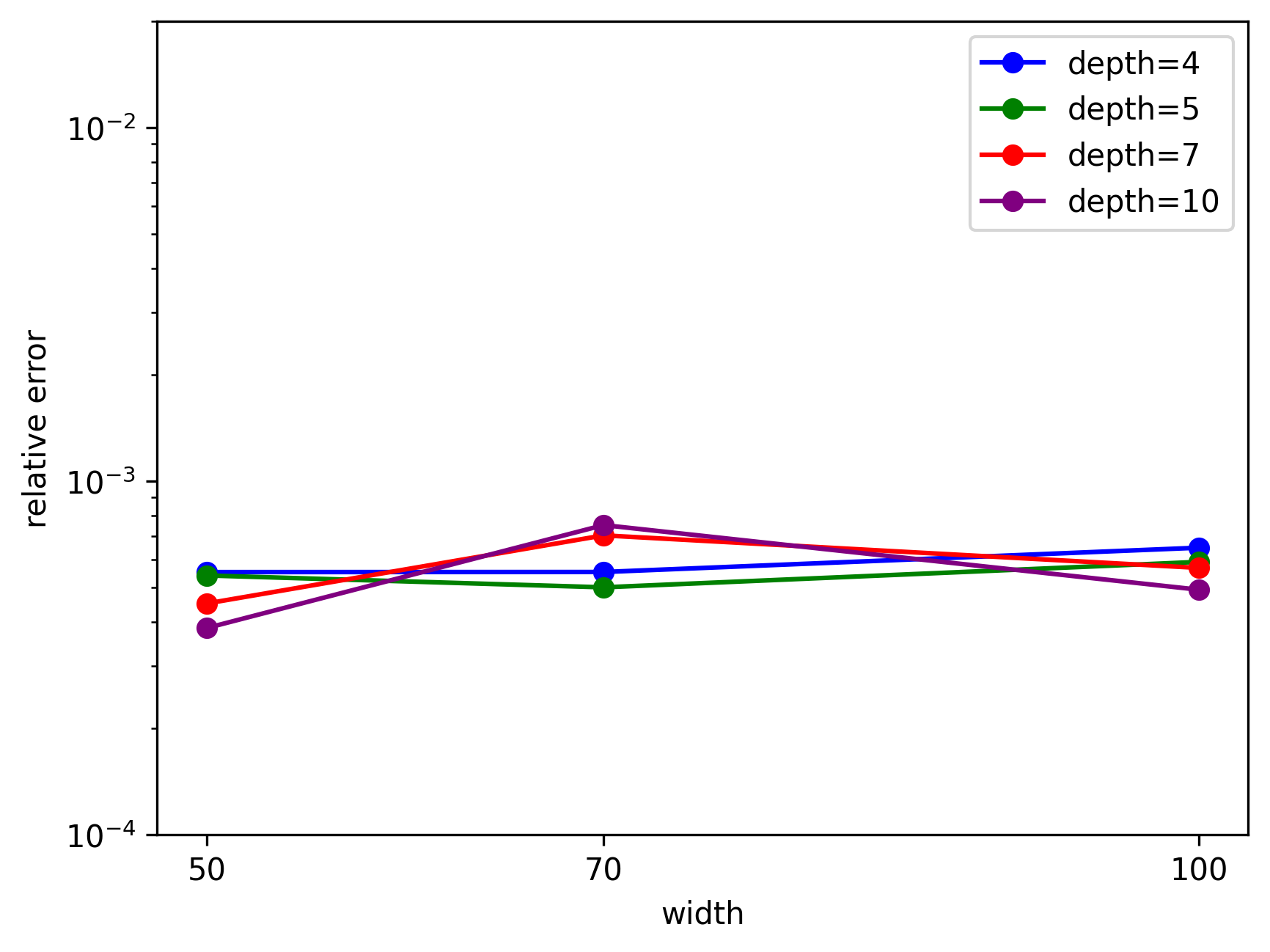}} 
  \subfigure[\label{acle_ac3} ACLE]{\includegraphics[width=0.19\linewidth]{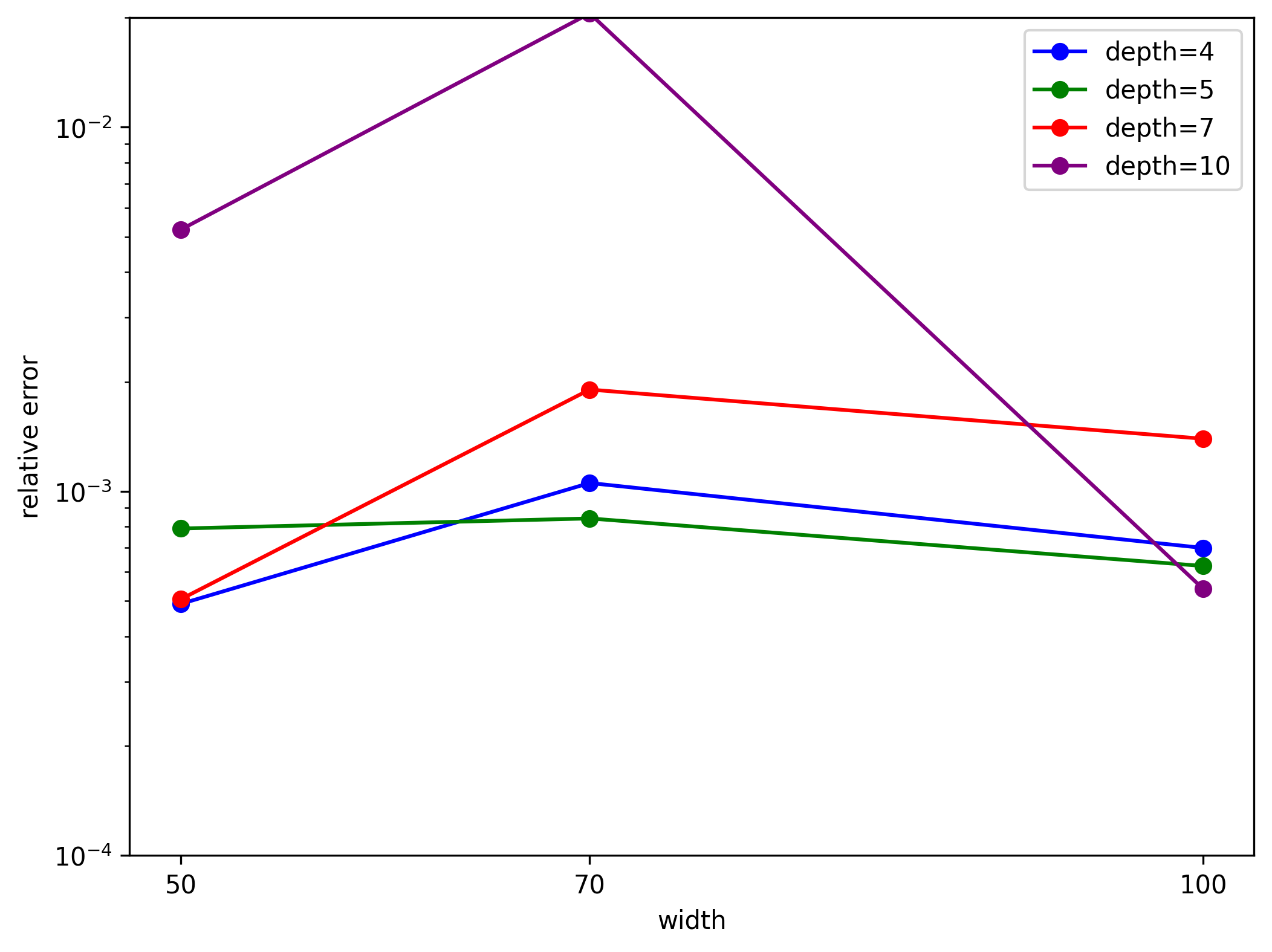}}
  \subfigure[\label{hac3} Helton]{\includegraphics[width=0.19\linewidth]{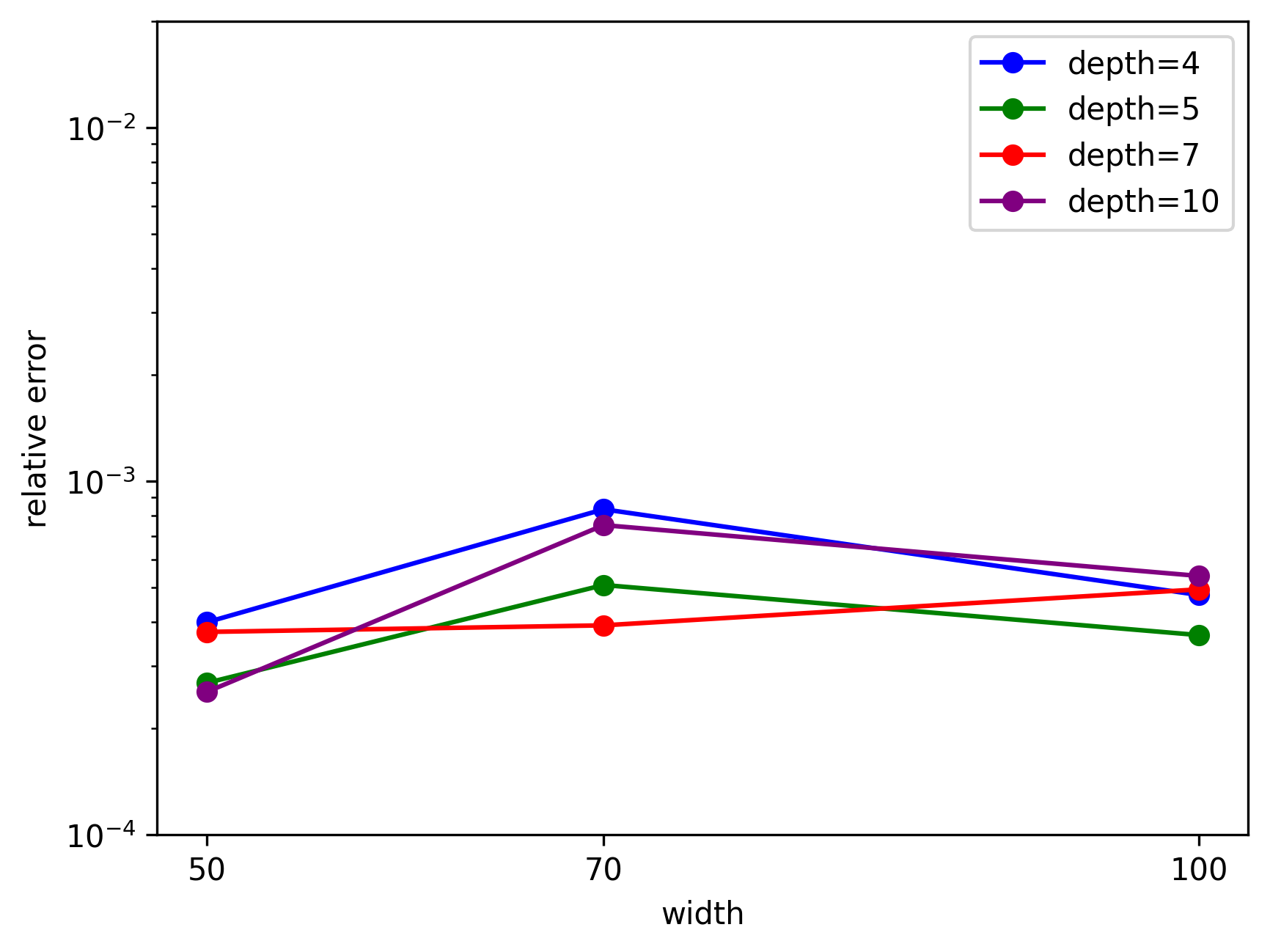}}
  \subfigure[\label{sac3} Sobol]{\includegraphics[width=0.19\linewidth]{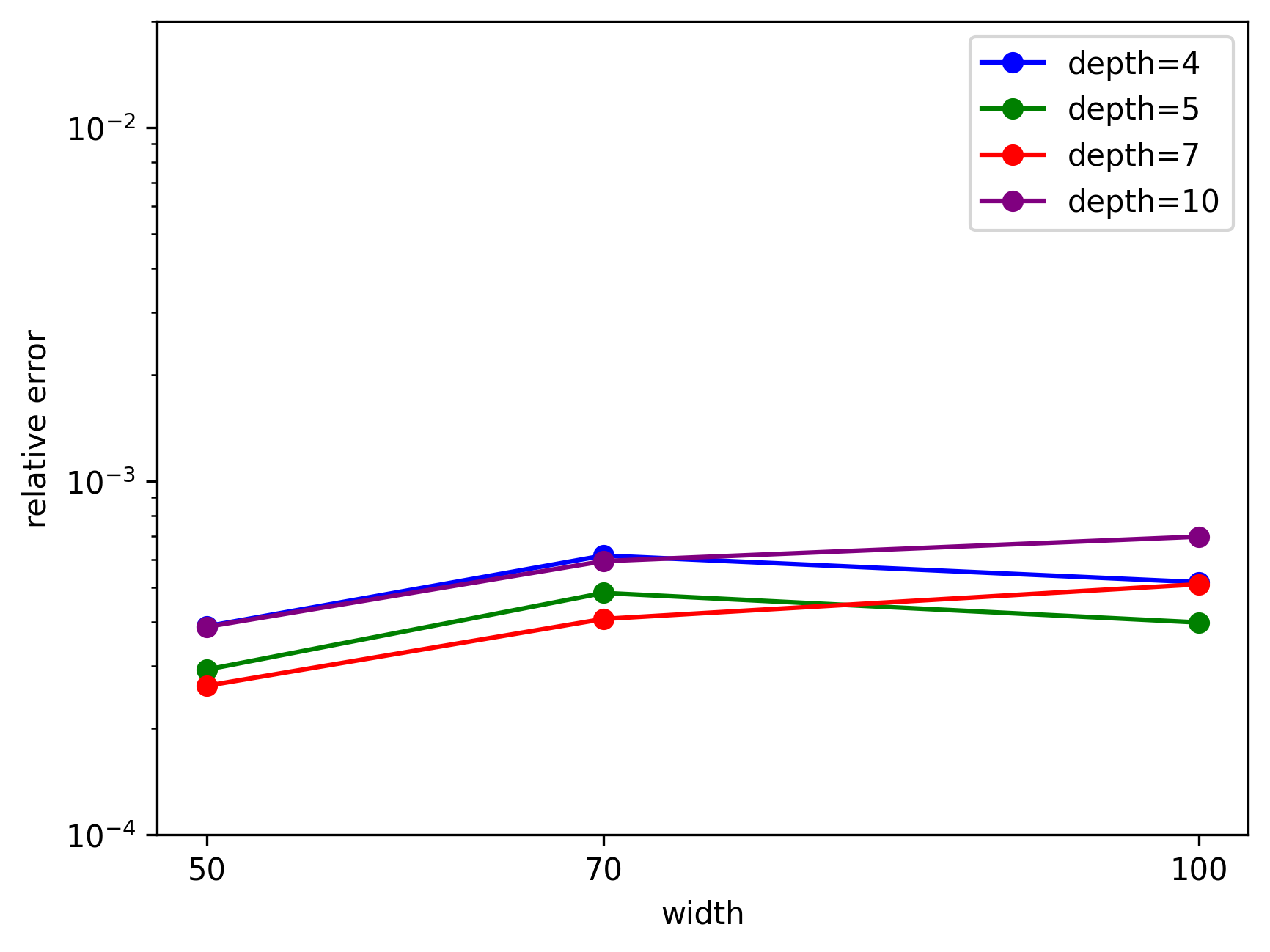}}
  \caption{Allen-Cahn with $d=3$}
  \label{fig:allencahn3}
\end{figure}

\begin{figure}[ht]
  \centering
\subfigure[\label{Vanilla_ac10} Vanilla]{\includegraphics[width=0.19\linewidth]{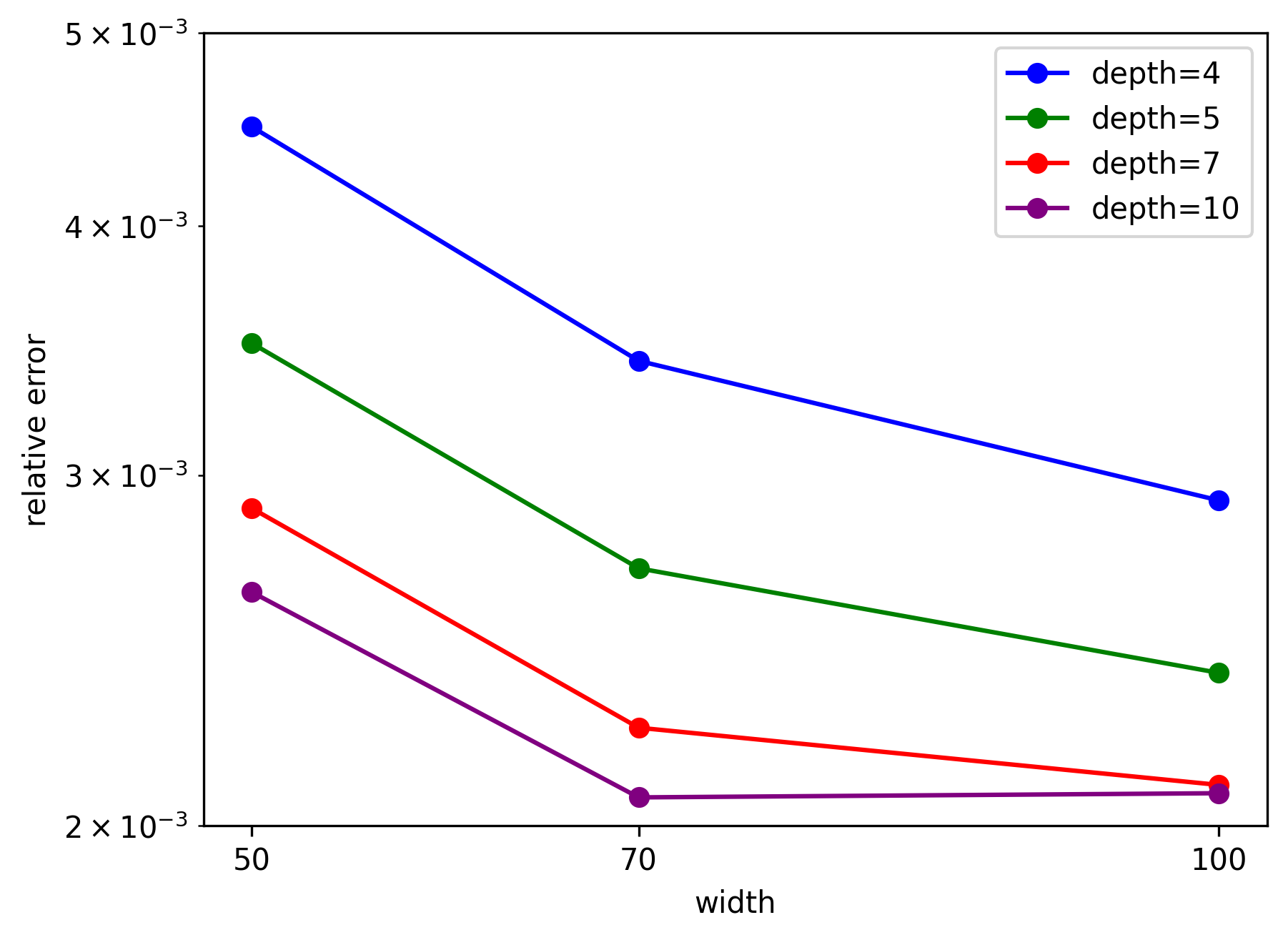}}
  \subfigure[\label{rad_ac10} RAD]{\includegraphics[width=0.19\linewidth]{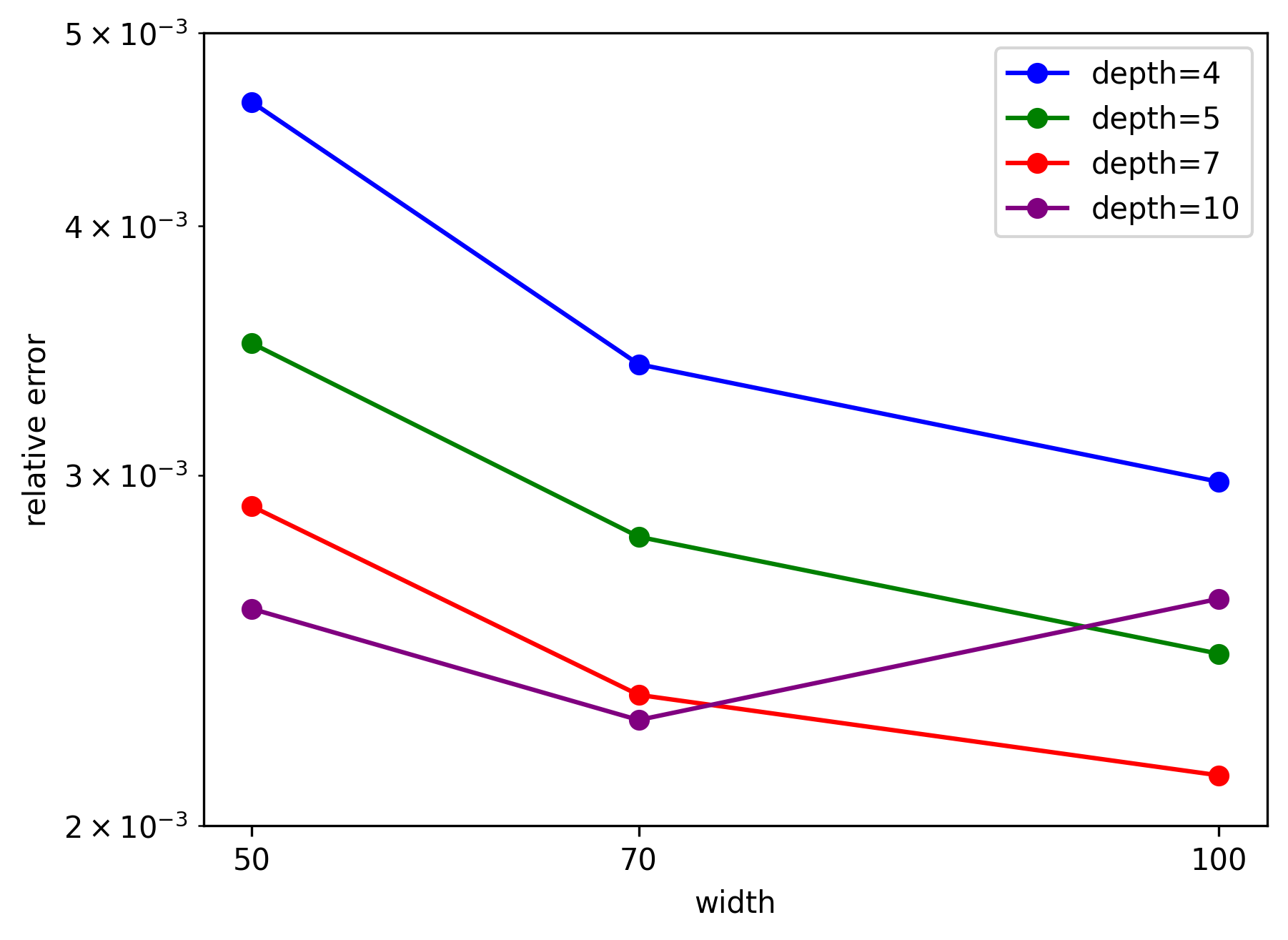}} 
  \subfigure[\label{acle_ac10} ACLE]{\includegraphics[width=0.19\linewidth]{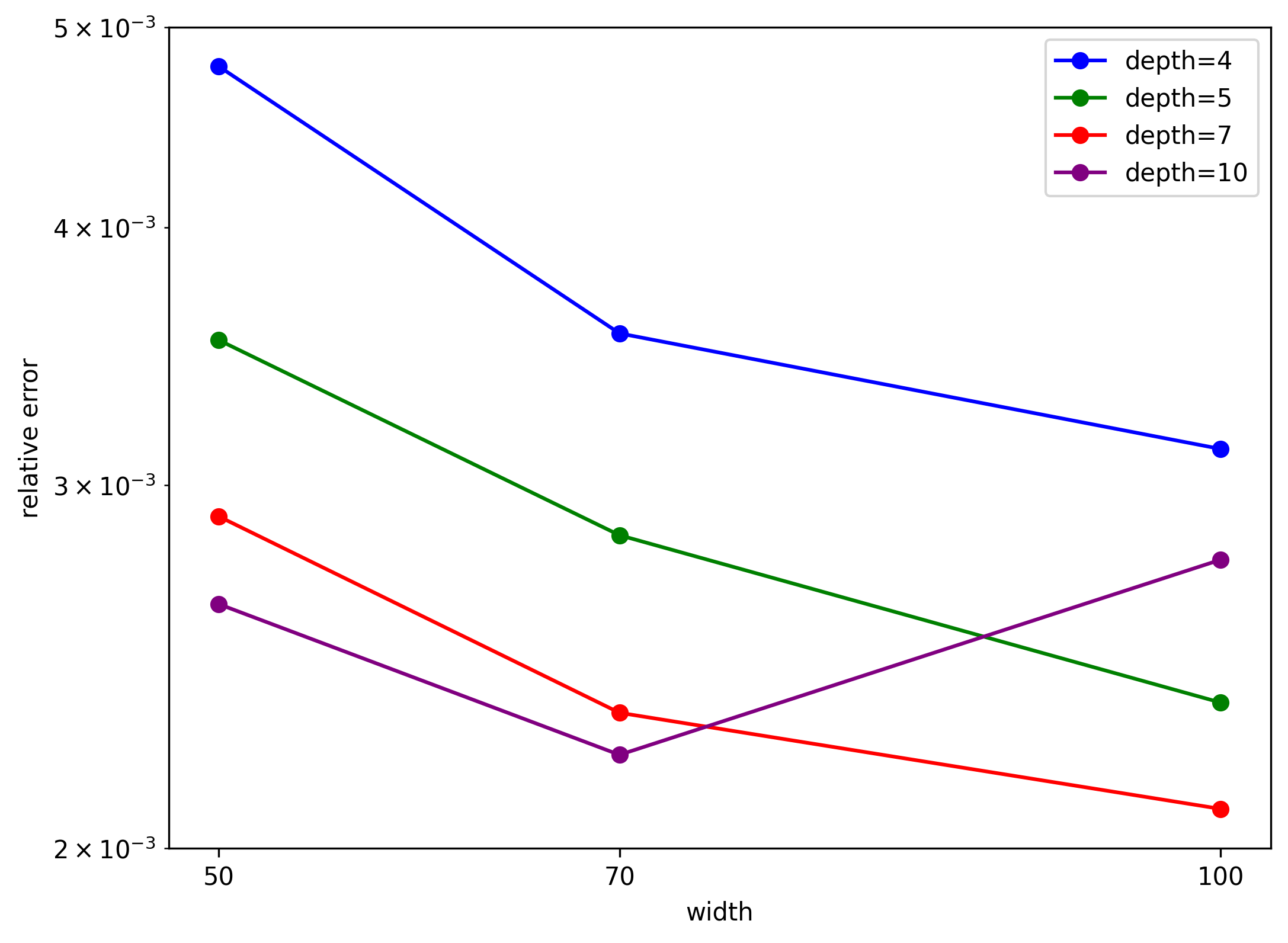}}
  \subfigure[\label{hac10} Helton]{\includegraphics[width=0.19\linewidth]{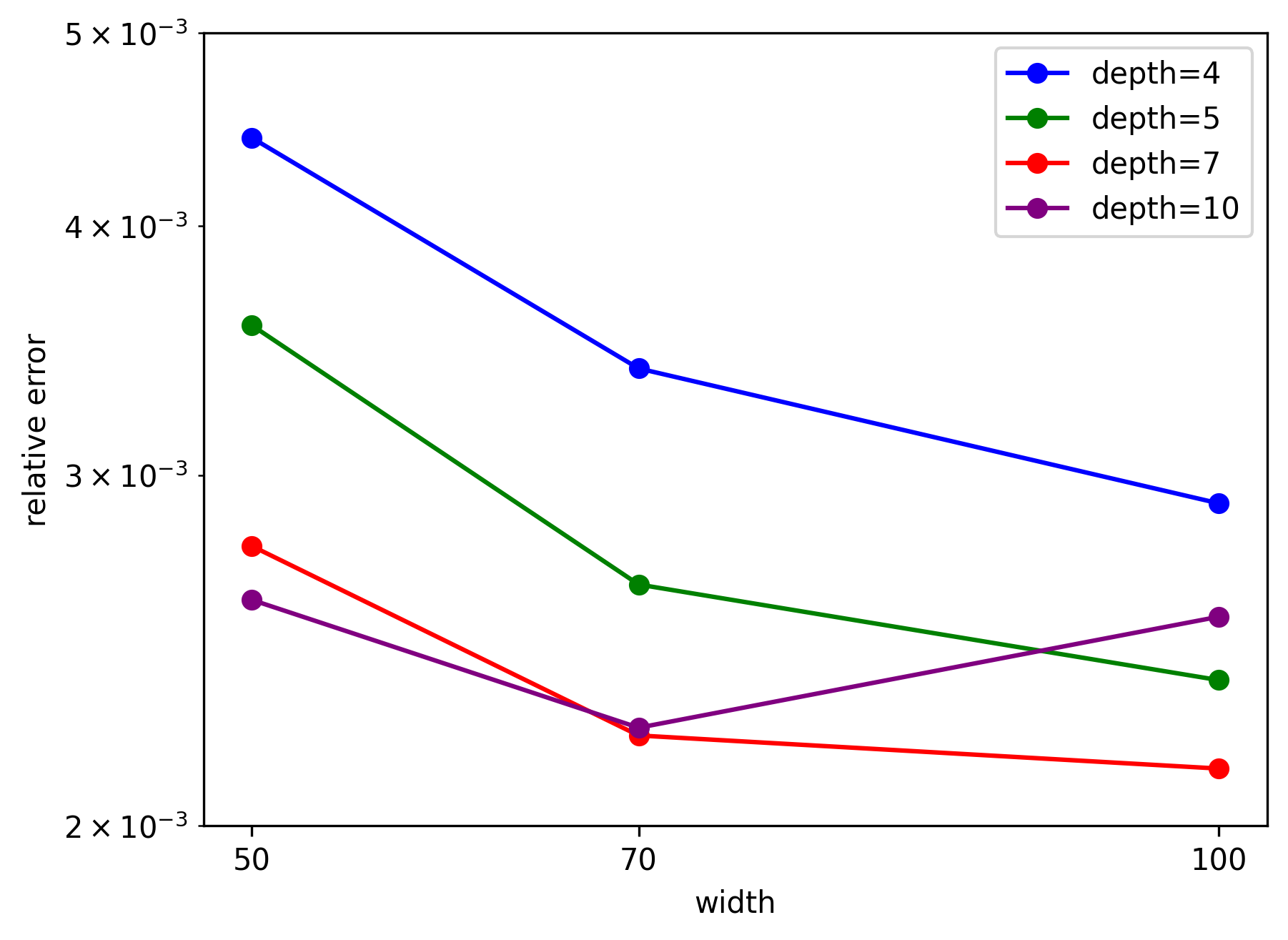}}
  \subfigure[\label{sac10} Sobol]{\includegraphics[width=0.19\linewidth]{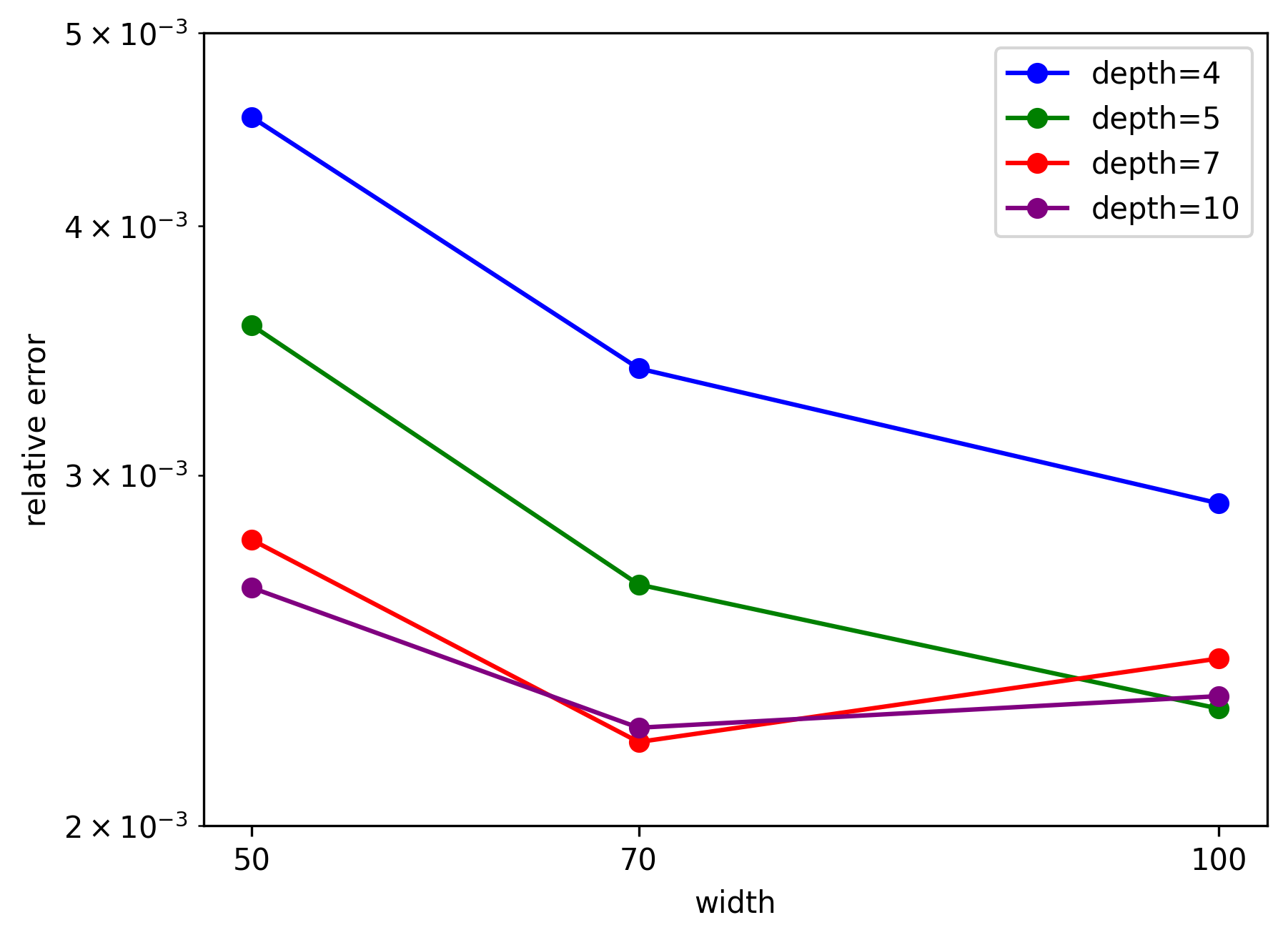}}
  \caption{Allen-Cahn with $d=10$}
  \label{fig:allencahn10}
\end{figure}

\paragraph{Sine-Gordon equations}
    
We consider the steady Sine-Gordon equations:
\begin{equation}\label{eq: sg}
    \Delta u + \sin(u)=f, \boldsymbol{x}\in [-1,1]^d,
\end{equation}
we use the three-body interaction exact solution with $c_i\sim \mathcal{N}(0, 1)$:
\begin{equation}
    u(\boldsymbol{x})=\left(1-\frac{1}{d}\|\boldsymbol{x}\|_2^2\right)\left(\frac{1}{d-2}\sum_{i=1}^{d-2}c_ie^{x_ix_{i+1}x_{i+2}}\right)
\end{equation}
Let $A:=\frac{1}{d-2}\sum_{i=1}^{d-2}c_ie^{x_ix_{i+1}x_{i+2}}$, $B:=1-\frac{1}{d}\|\boldsymbol{x}\|_2^2$,
then 
\begin{equation}
    f(\boldsymbol{x})=B\Delta A+A\Delta B+\nabla A^T\nabla B +\sin(AB).
\end{equation}
\begin{figure}[ht]
  \centering
\subfigure[\label{vanilla_sg3} Vanilla]{\includegraphics[width=0.19\linewidth]{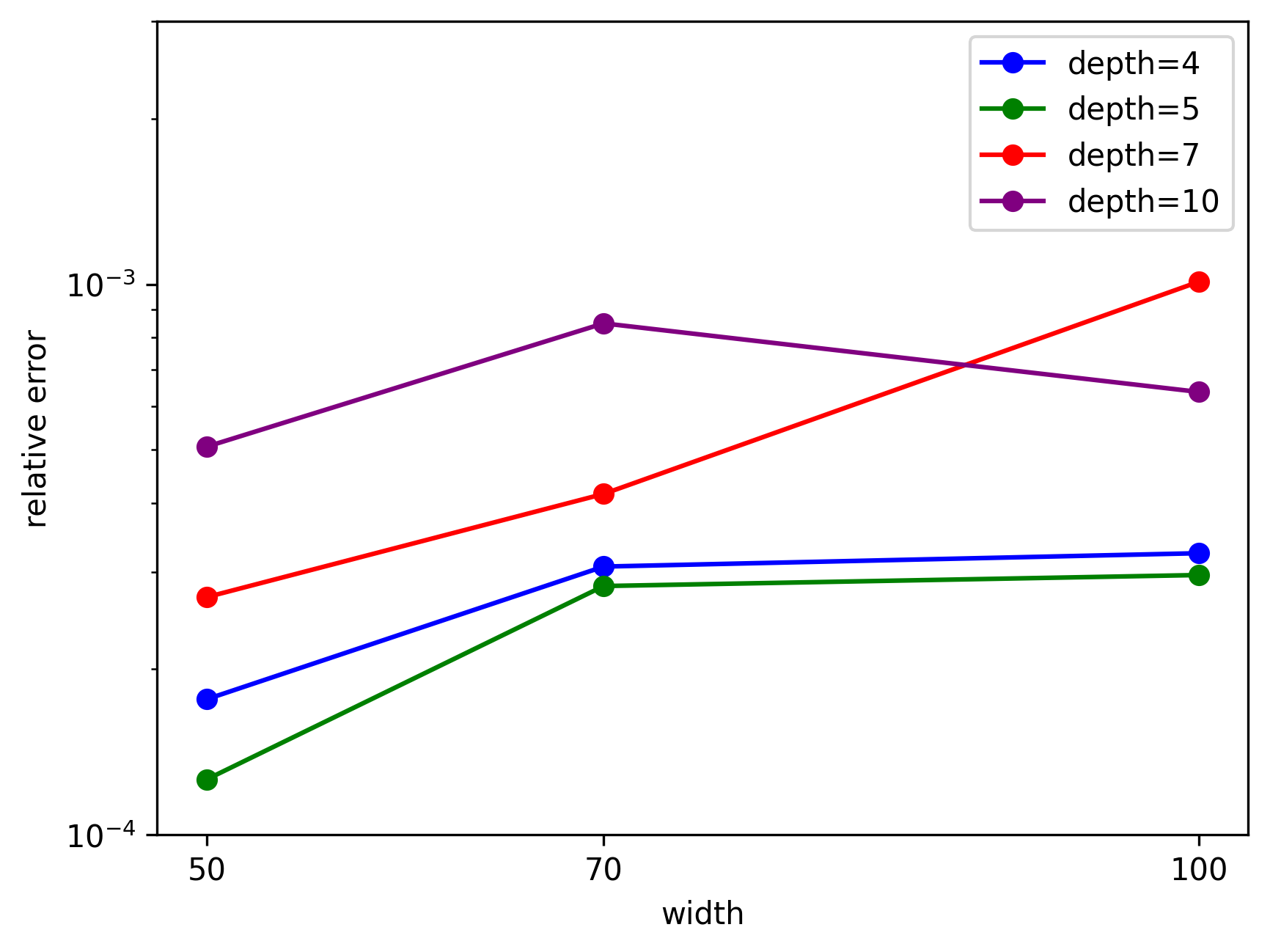}}
  \subfigure[\label{rad_sg3} RAD]{\includegraphics[width=0.19\linewidth]{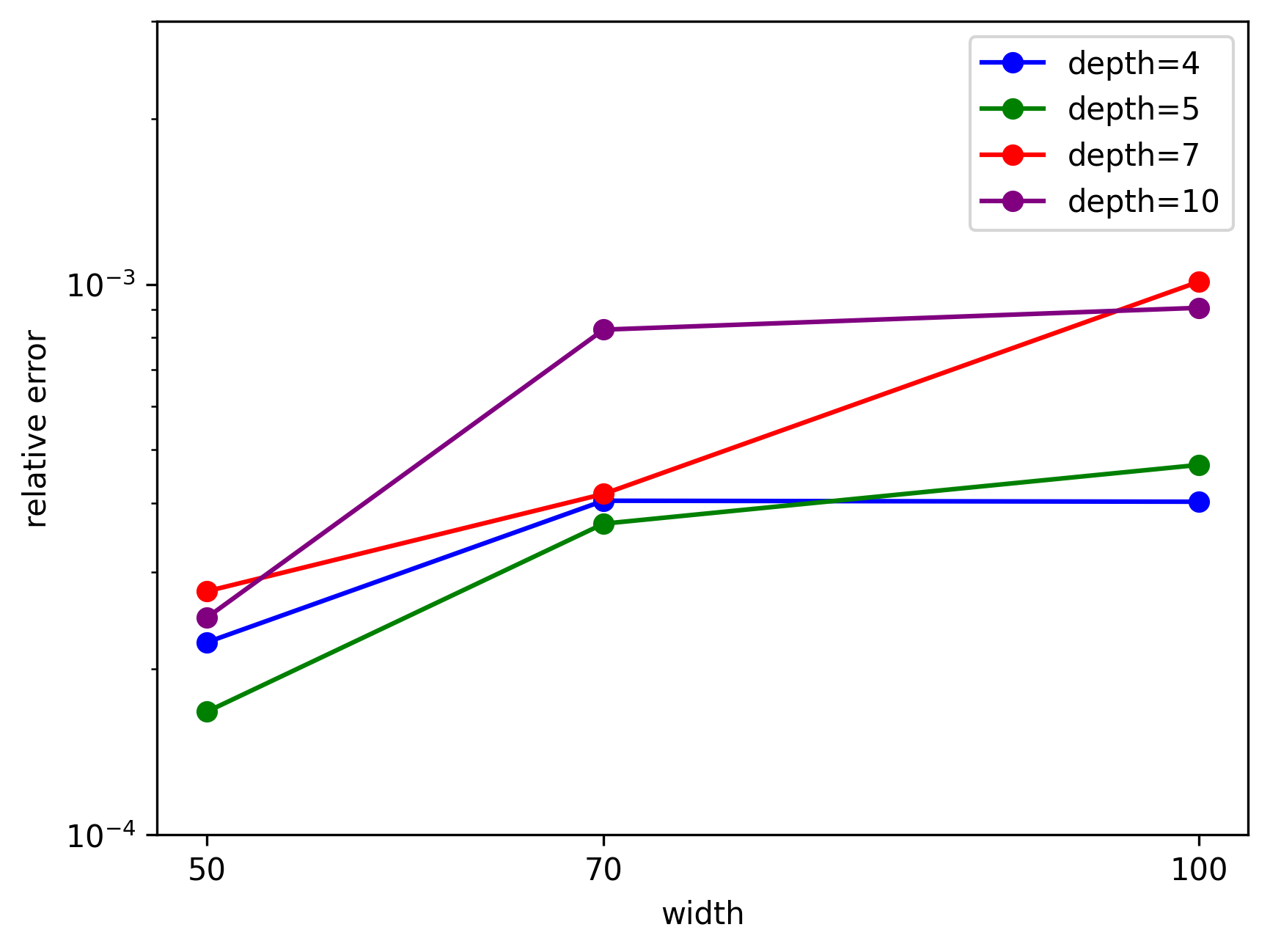}} 
  \subfigure[\label{acle_sg3} ACLE]{\includegraphics[width=0.19\linewidth]{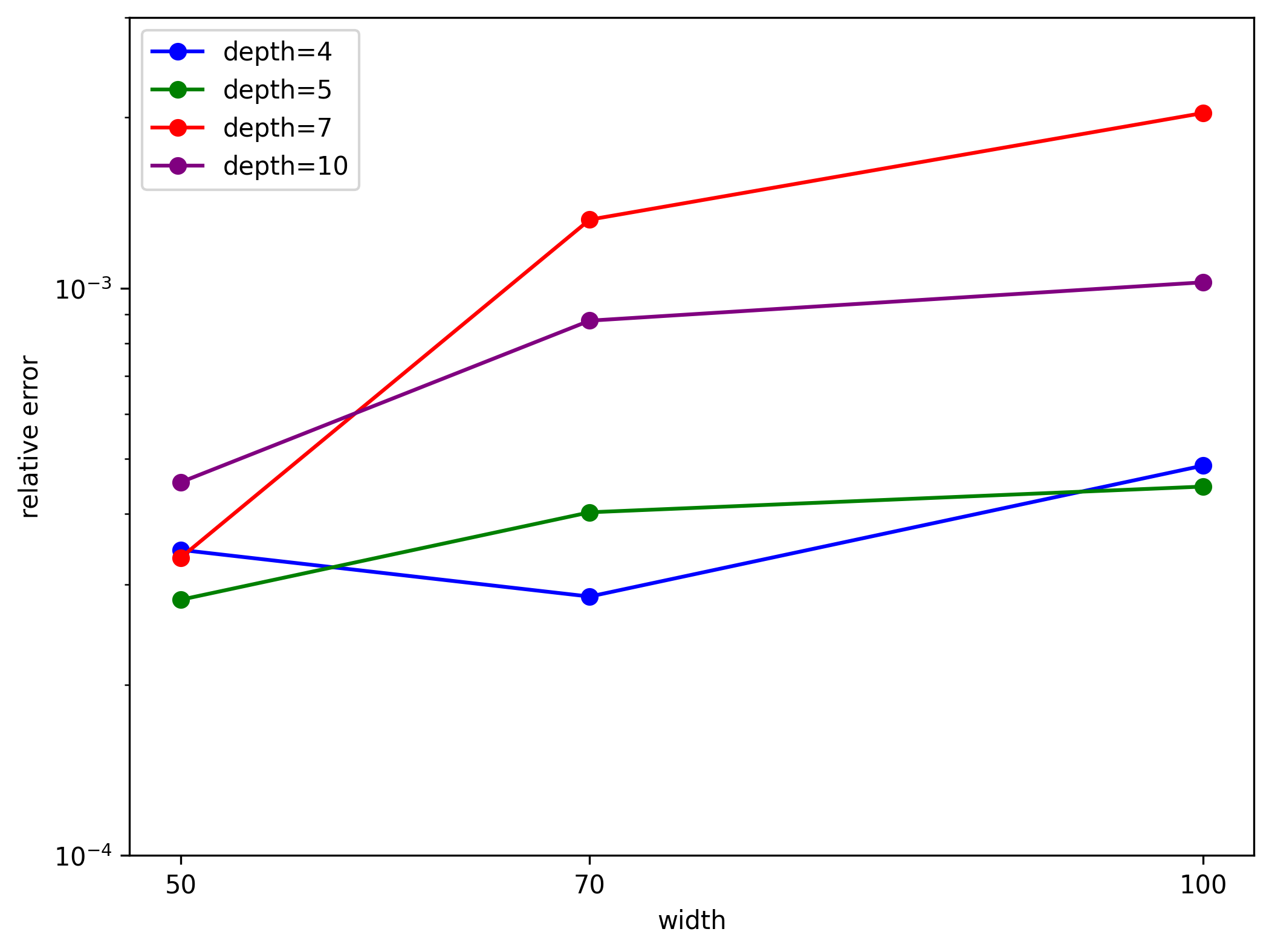}}
  \subfigure[\label{hsg3} Helton]{\includegraphics[width=0.19\linewidth]{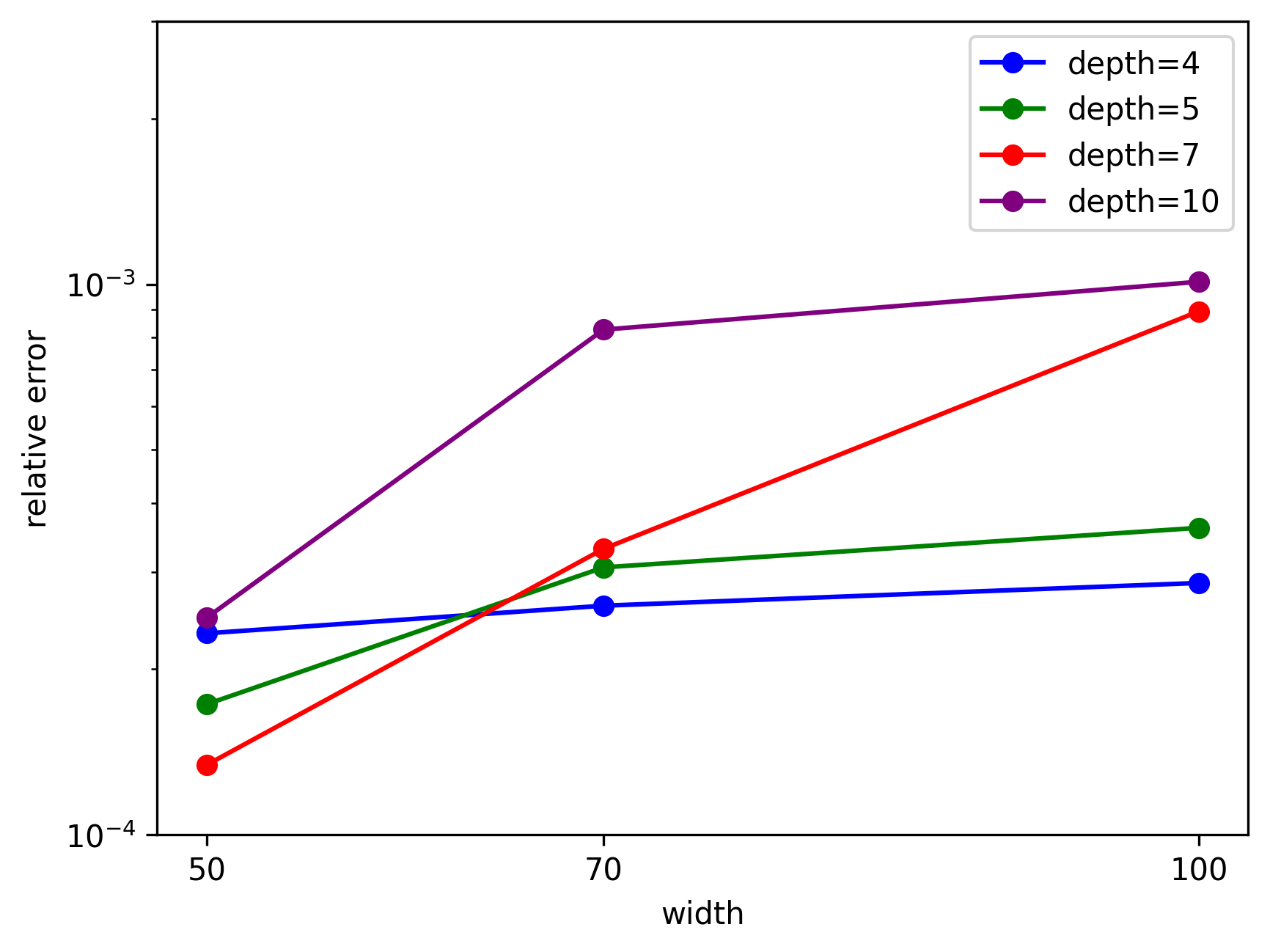}}
  \subfigure[\label{ssg3} Sobol]{\includegraphics[width=0.19\linewidth]{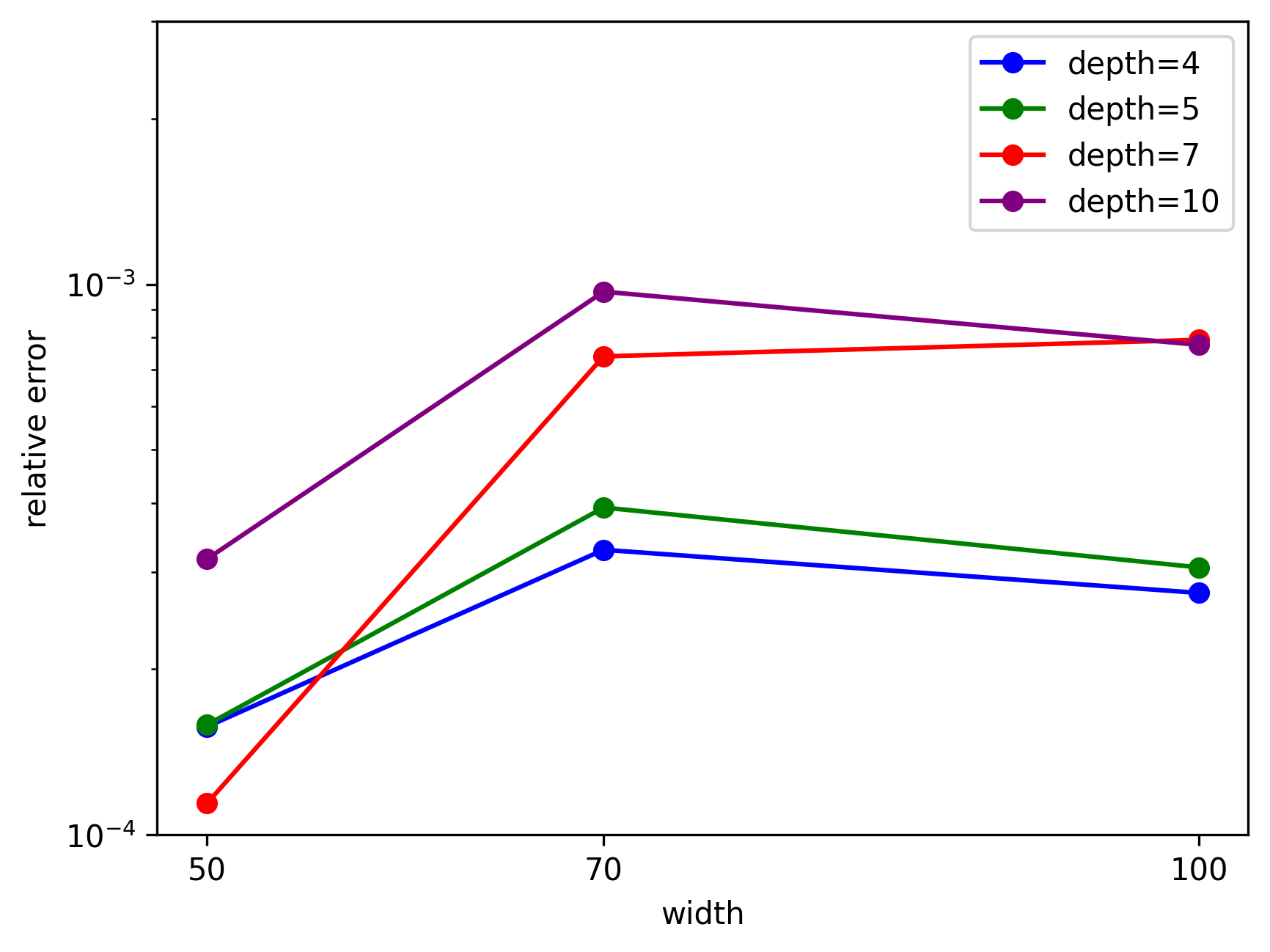}}
  \caption{Sine-Gordon with $d=3$}
  \label{fig:sinegordon3}
\end{figure}

\begin{figure}[ht]
  \centering
\subfigure[\label{vanilla_sg10} Vanilla]{\includegraphics[width=0.19\linewidth]{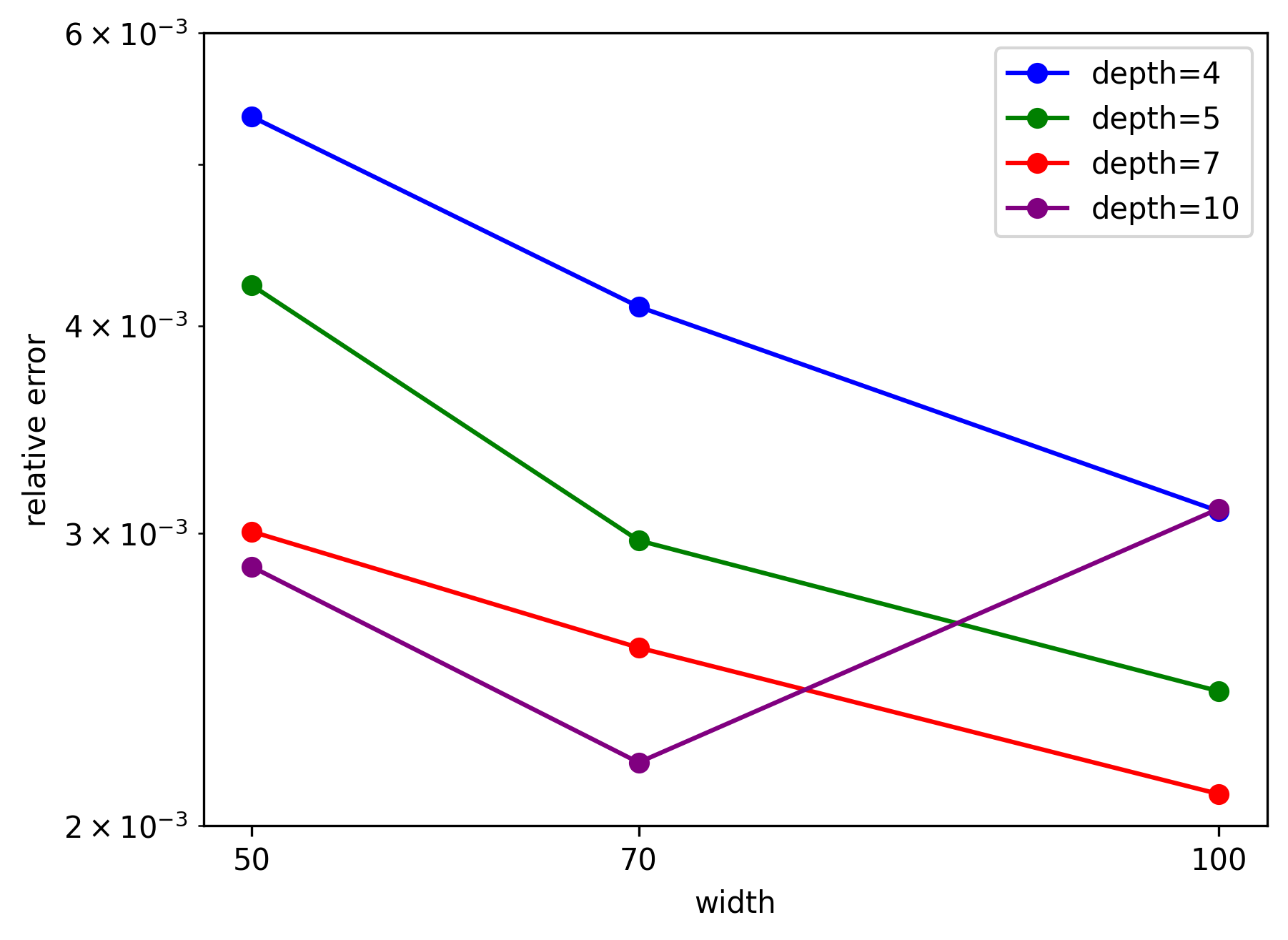}}
  \subfigure[\label{rad_sg10} RAD]{\includegraphics[width=0.19\linewidth]{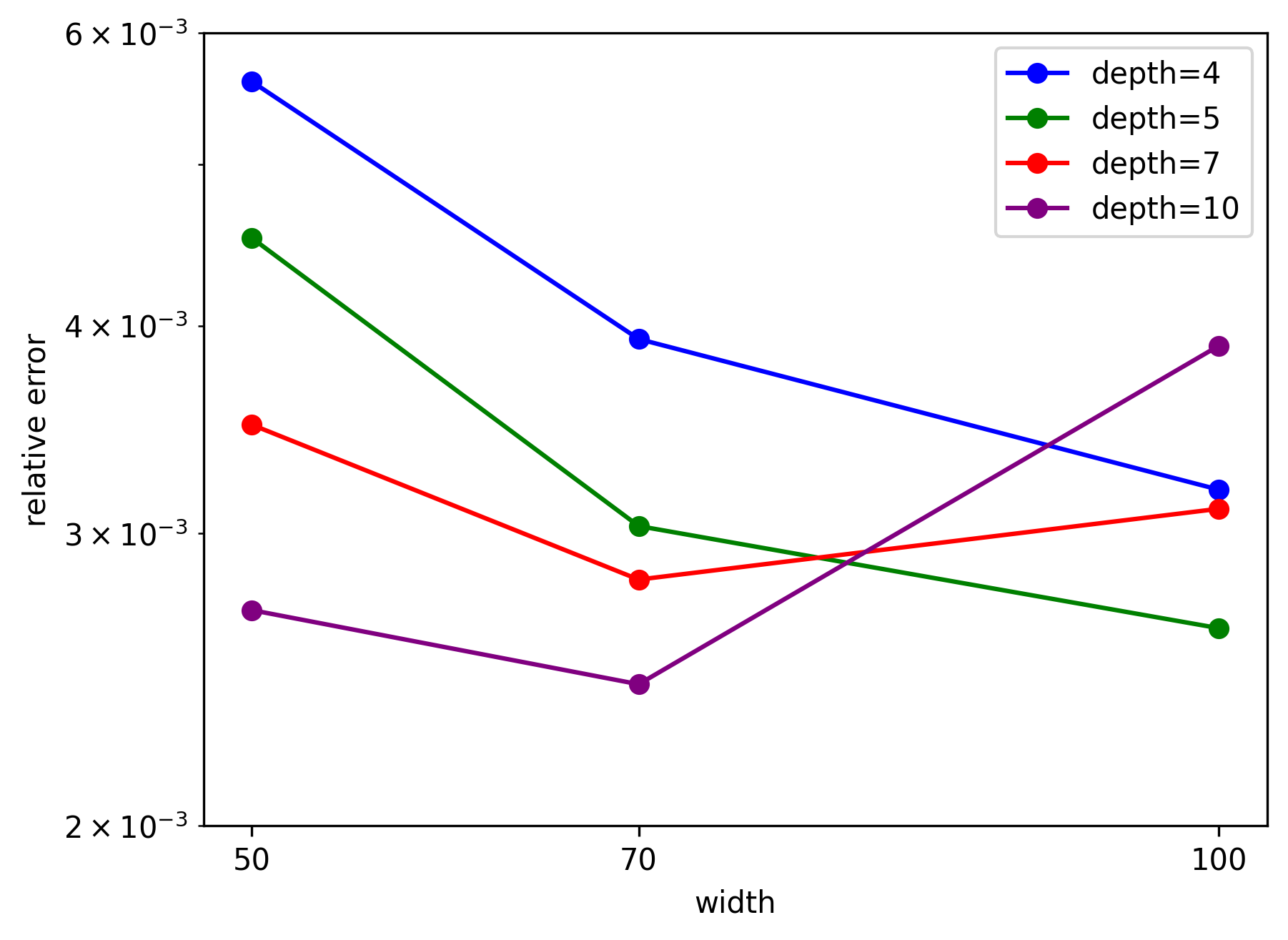}} 
  \subfigure[\label{acle_sg10} ACLE]{\includegraphics[width=0.19\linewidth]{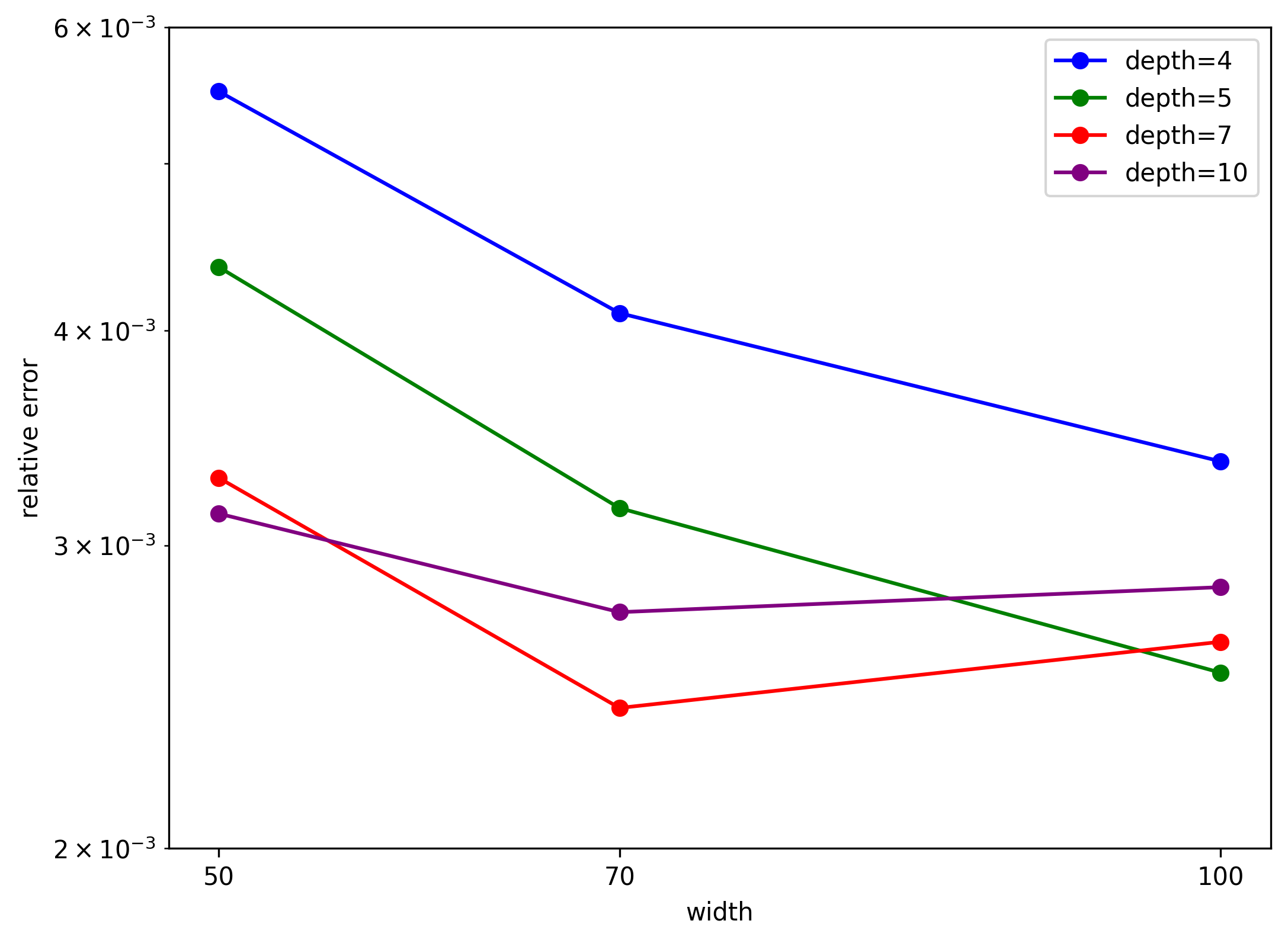}}
  \subfigure[\label{hsg10} Helton]{\includegraphics[width=0.19\linewidth]{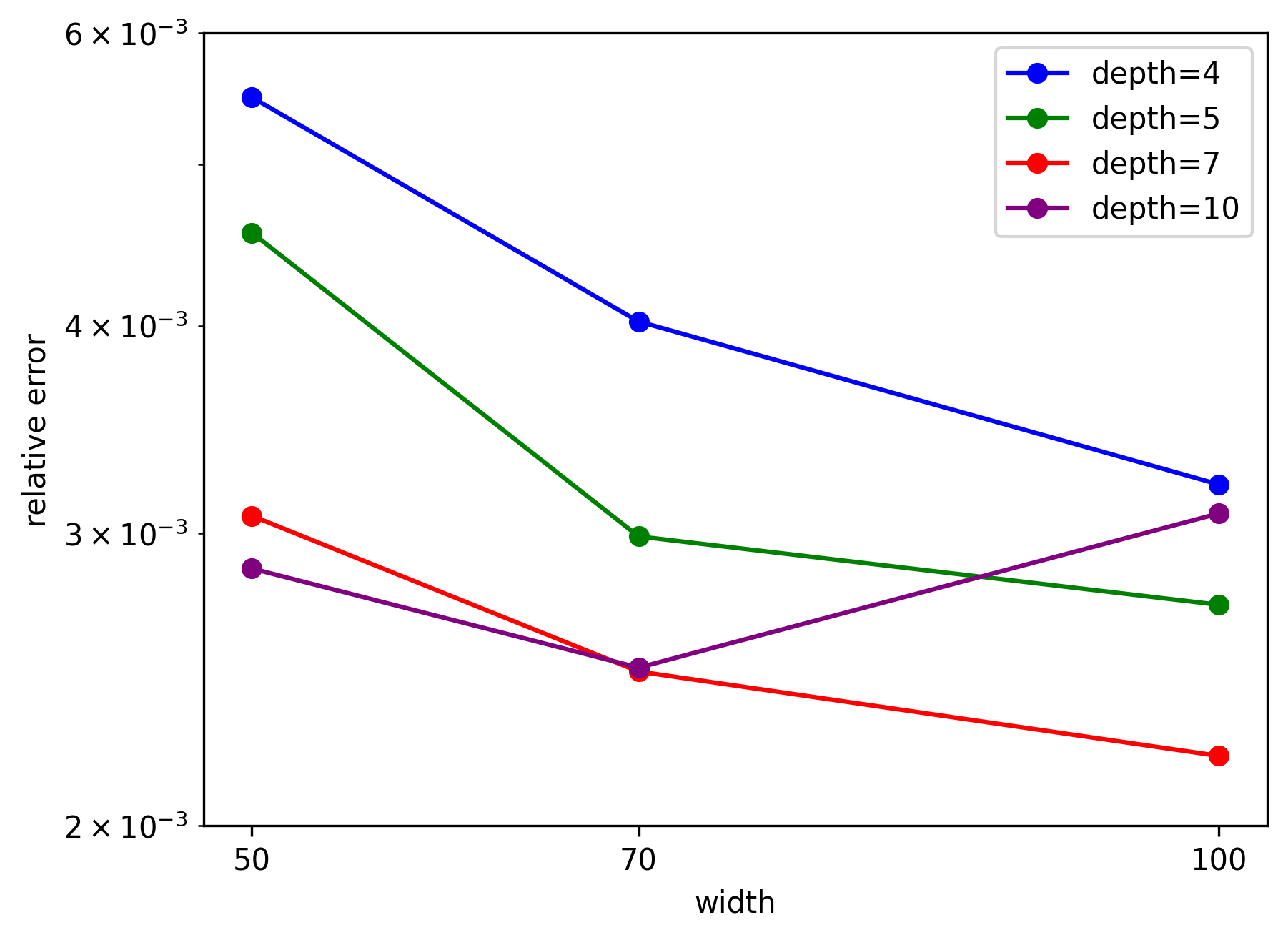}}
  \subfigure[\label{ssg10} Sobol]{\includegraphics[width=0.19\linewidth]{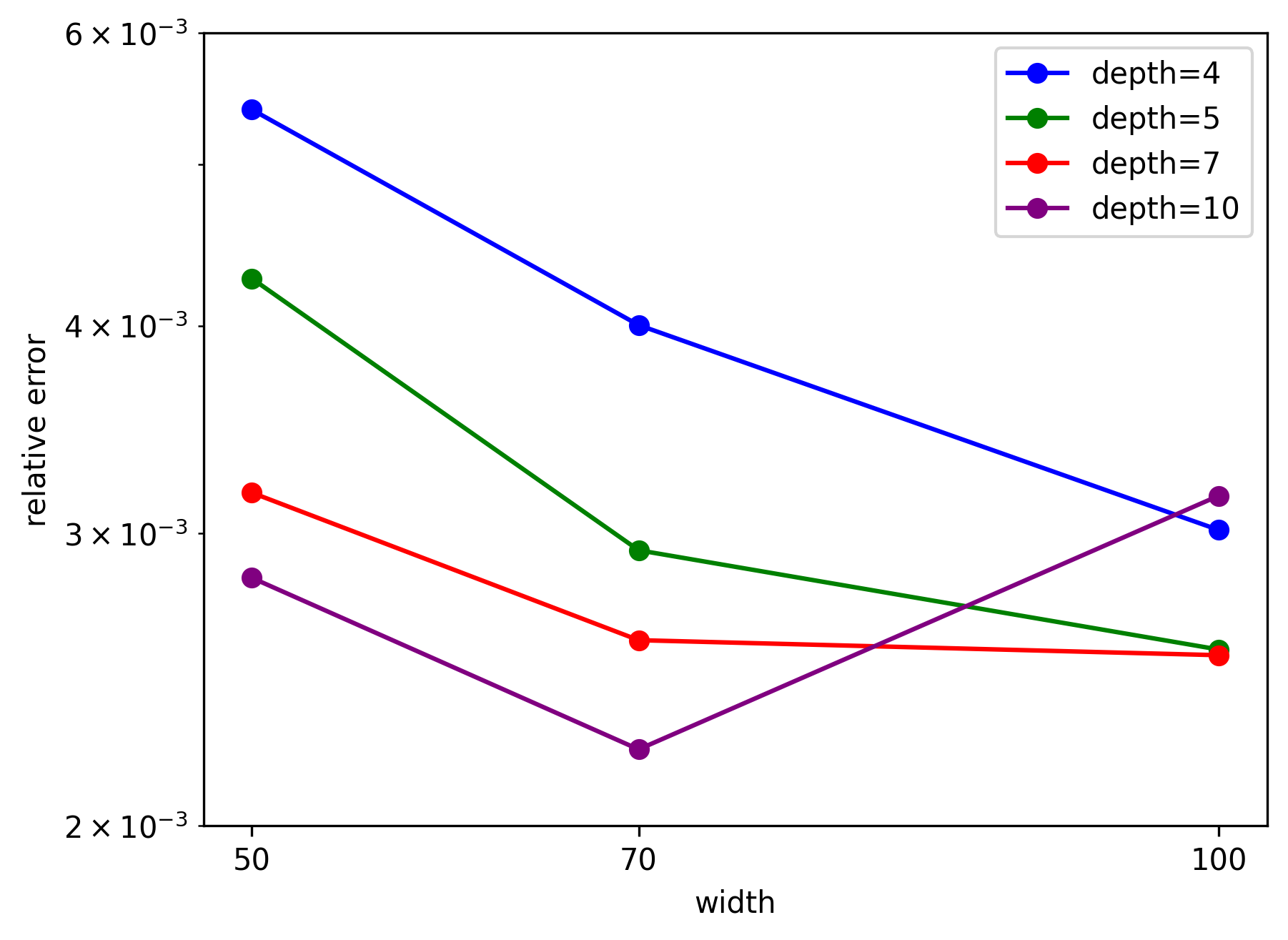}}
  \caption{Sine-Gordon with $d=10$}
  \label{fig:sinegordon10}
\end{figure}
\begin{figure}[ht]
  \centering
  \subfigure[\label{error_points_ac}]{\includegraphics[width=0.23\linewidth]{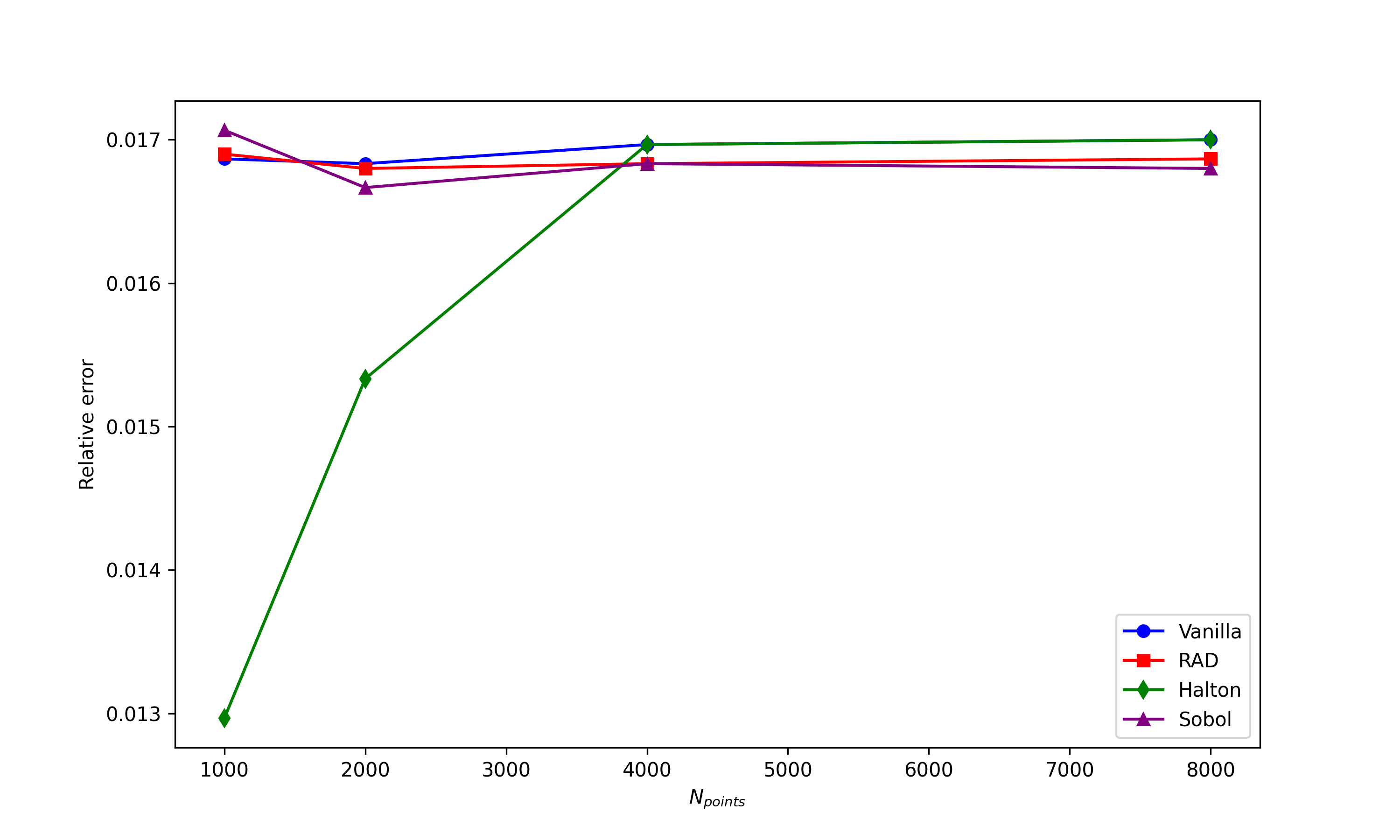}}
  \subfigure[\label{error_points_sg}]{\includegraphics[width=0.23\linewidth]{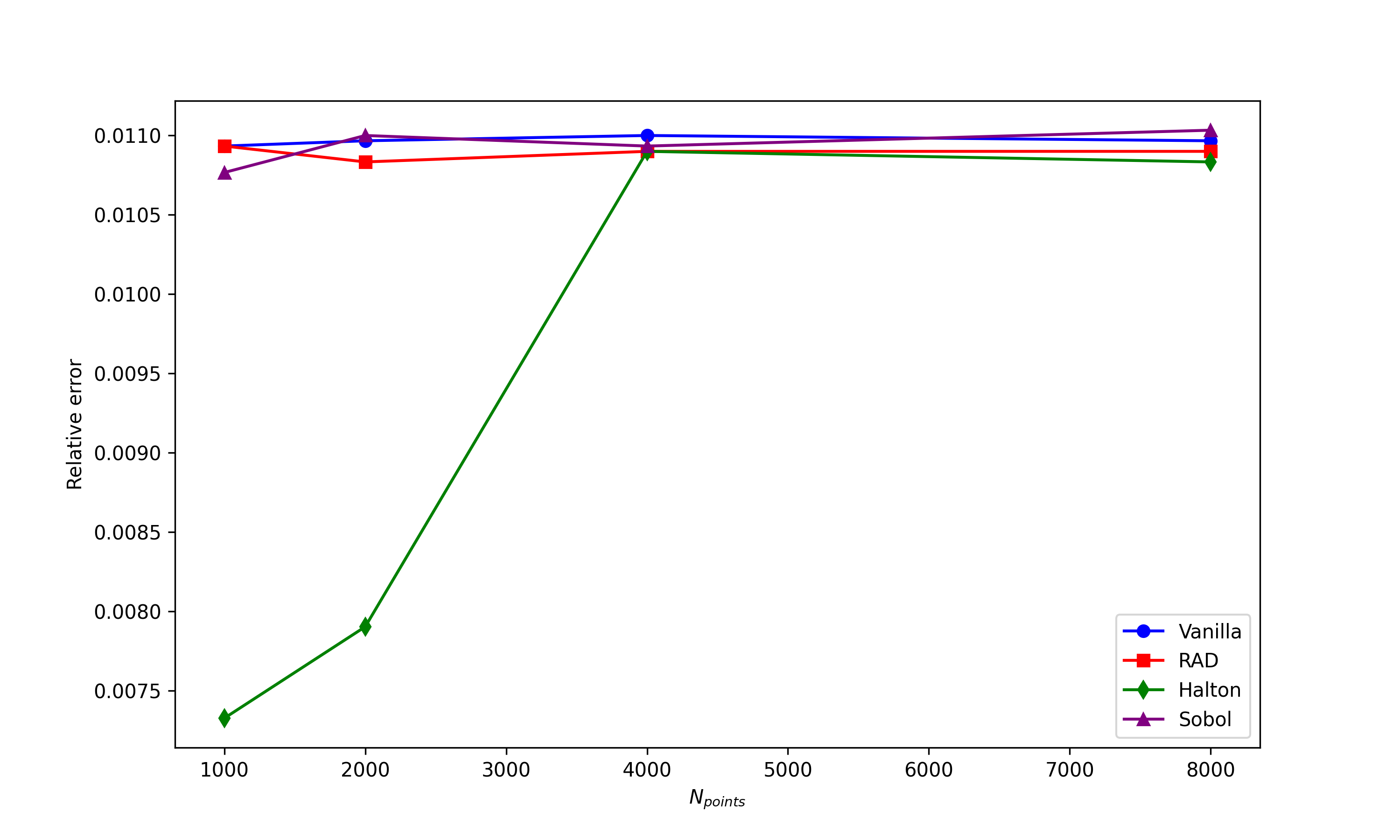}}
  \subfigure[\label{errors_ite_ac}]{\includegraphics[width=0.23\linewidth]{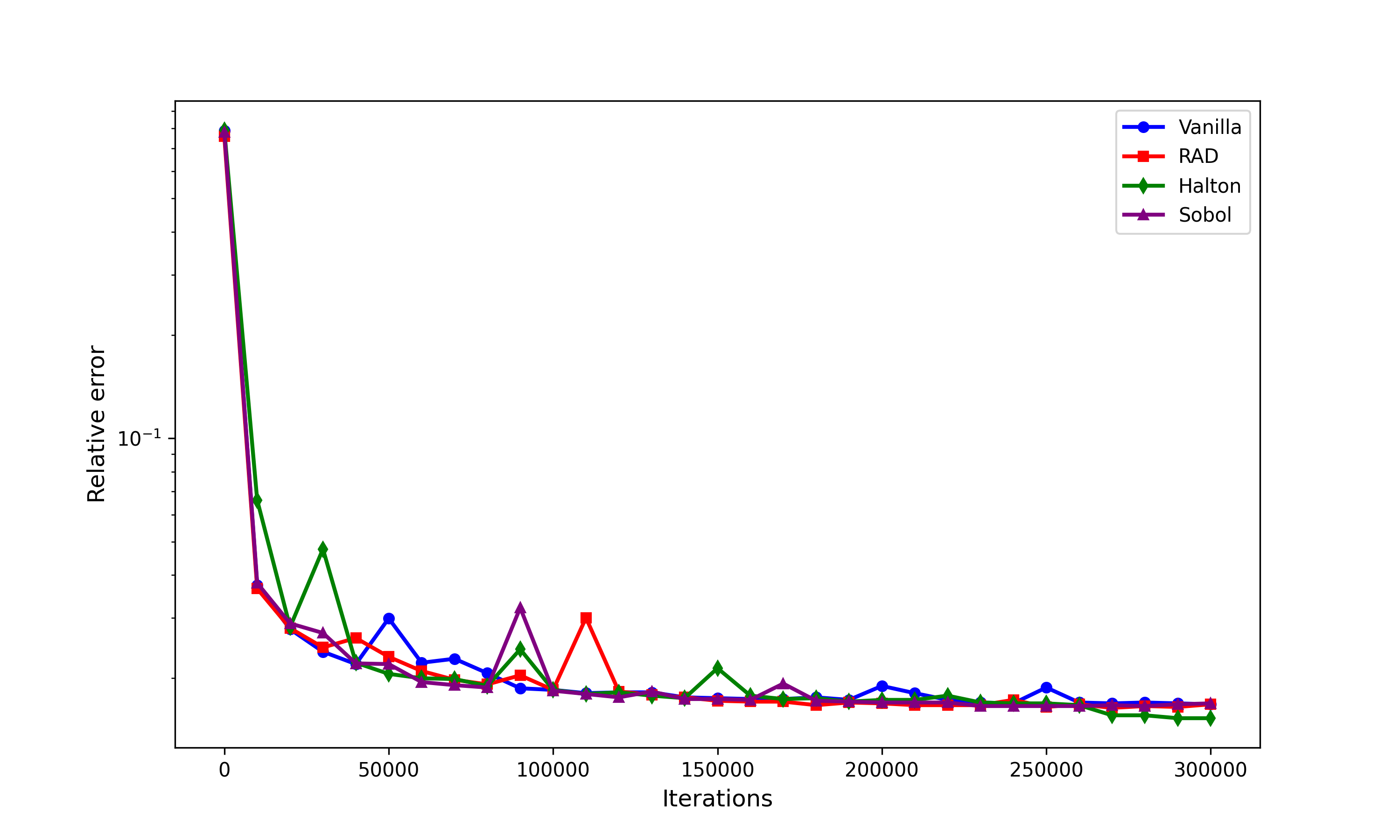}}
  \subfigure[\label{errors_ite_sg}]{\includegraphics[width=0.23\linewidth]{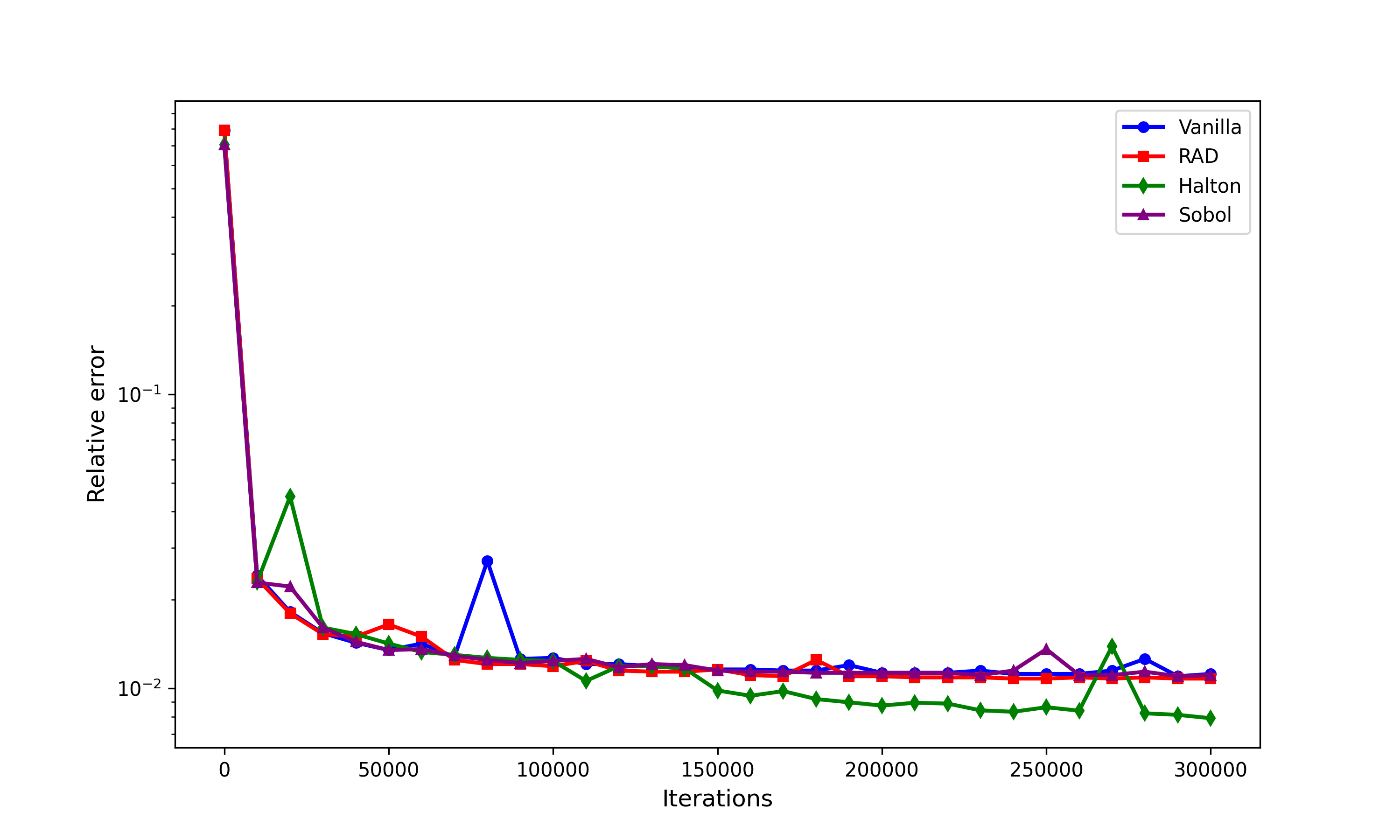}}
  \caption{Sine-Gordon and Allen-Cahn with $d=100$. \subref{error_points_ac} and \subref{error_points_sg} present the relationship between the relative error and $N_{points}$ for Allen-Cahn and Sine-Gordon respectively. \subref{errors_ite_ac} and \subref{error_points_sg} illustrate the dynamic relative error during training for Allen-Cahn and Sine-Gordon respectively.}
  \label{fig:sg_ac_d100}
\end{figure}
\paragraph{Results}
We conducted the experiments with different size of networks to verify the performance of different, the results are shown in \cref{fig:allencahn3,fig:allencahn10} for Allen-Cahn equations with $d=3$ and $d=10$ respectively, \cref{fig:sinegordon3,fig:sinegordon10} for Sine-Gordon equations with $d=3$ and $d=10$ respectively, and \cref{fig:sg_ac_d100} demonstrate the results for both Allen-Cahn equations and Sine-Gordon equations with $d=100$. The details including standard variance are provided in \cref{appendix: details of experiments}. 

For the Allen-Cahn equation with $d=3$, the smallest relative error is $2.15\times 10^{-4}\pm 3.78\times 10^{-5}$, achieved by Vanilla with a depth of 10 and the width of 50. In the 10D case of the Allen-Cahn equation, the minimum relative error is $2.07\times 10^{-3}\pm 3.40\times 10^{-5}$, which is obtained by Vanilla with the depth of 10 and the width of 70. Turning to the Sine-Gordon equation in 3D, the smallest relative error is $1.14\times 10^{-4}\pm 6.48\times 10^{-6}$, achieved by Sobol with the depth of 7 and the width of 50. For the 10D Sine-Gordon equation, the minimum relative error is $2.09\times 10^{-3}\pm 4.08\times 10^{-5}$, accomplished by Vanilla when the depth is 7 and the width is 100. Although most of the best results are from the vanilla PINNs, QRPINNs outperforms these adaptive methods and reveals that adaptive methods are more suitable for low-dimensional problems and problems that have locality . On the other hand, for $d=100$, in \cref{fig:sg_ac_d100}, QRPINNs achieves the best results: compared with the best results among Vanilla and RAD, QRPINNs obtain a promotion of 22.6\% for Allen-Cahn equations and 32.8\% for Sine-Gordon equations. Furthermore, we also conducted a series of experiments with different number of input points $N_{points}$, the results in \cref{error_points_sg,error_points_ac} support that a small batch size of machine learning is necessary and can improve the performance. Hence, it is impractical to expect that increasing the number of samples in MC alone can achieve the same performance as QMC.

\subsection{The cost of generating low-discrepancy sequences}\label{section: cost}
In PINNs, when solving high-dimensional PDEs, it is impossible to generate a high-resolution point set by uniform grids due to the exponentially consuming of the memory. Herein, in most framework \cite{shi2024stochastic,cen2024deep}, the input points are directly generated by an uniformly random sampling on the domain $\Omega$, \textit{i.e.} the random sampling mentioned above. And those adaptive sampling methods will sample more points on $\Omega$ to construct the point pool and then select the input points from the point pool randomly or deterministically. However, for RQMC, the point pool is exactly the low-discrepancy sequences. Herein, if the cost of generating those sequences is also intolerable like generating the uniform grids in high-dimensionality, it is impractical to use RQMC in PINNs. 

Fortunately, the cost of RQMC is acceptable, we conducted some experiments to demonstrate the relationship between cost and the number of points $N$ as well as the dimensionality $d$. \cref{fig:cost} shows that for any dimensions, the execution time is linear with the number of points and is negligible compared with the training time of PINNs. The other details including the consuming of the CPU memory can be found in \cref{appendix: cost}.

\begin{figure}[ht]
  \centering 
  \subfigure[\label{cost_3} $d=3$]{\includegraphics[width=0.23\linewidth]{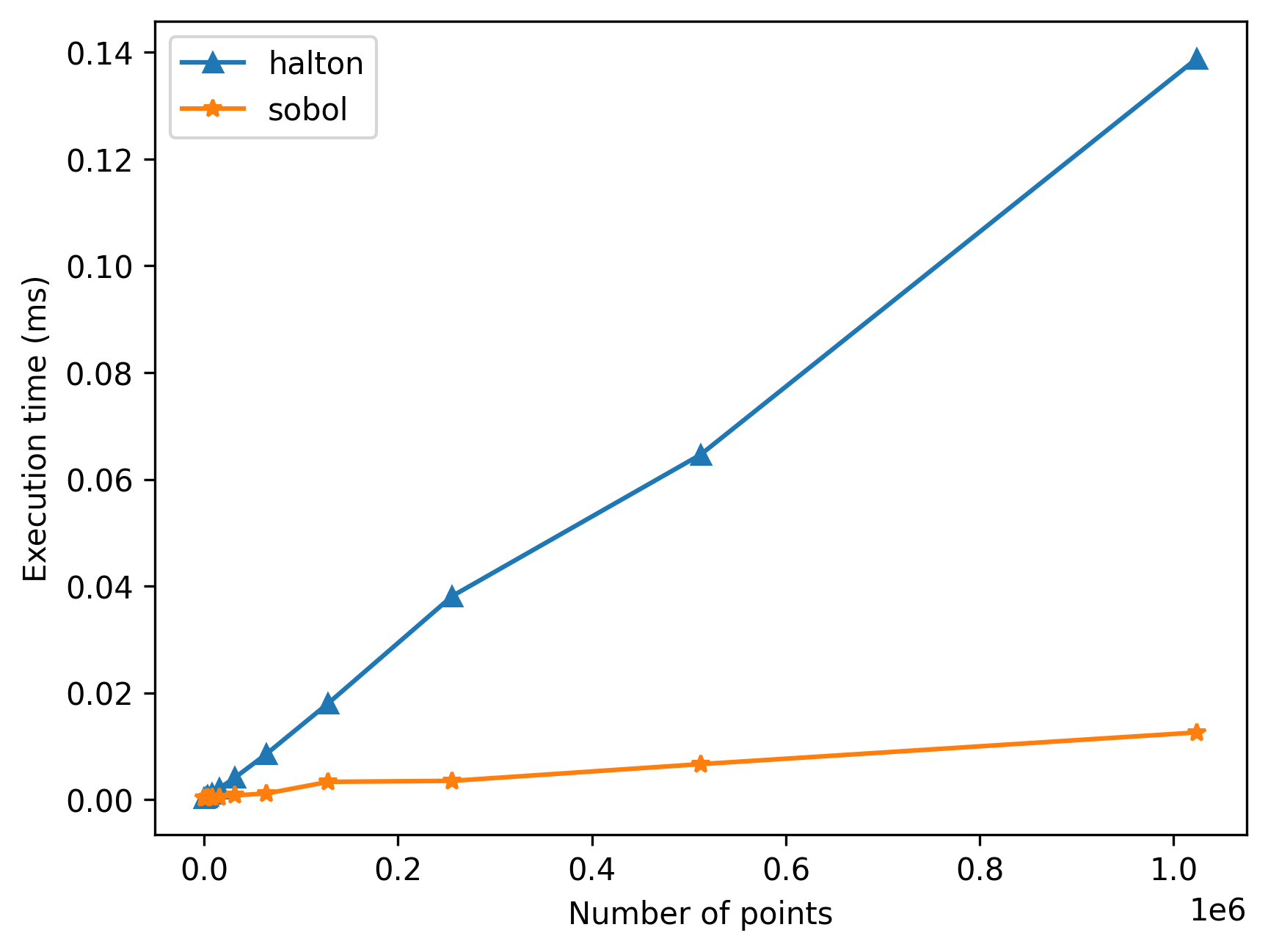}} 
  \subfigure[\label{cost_10} $d=10$]{\includegraphics[width=0.23\linewidth]{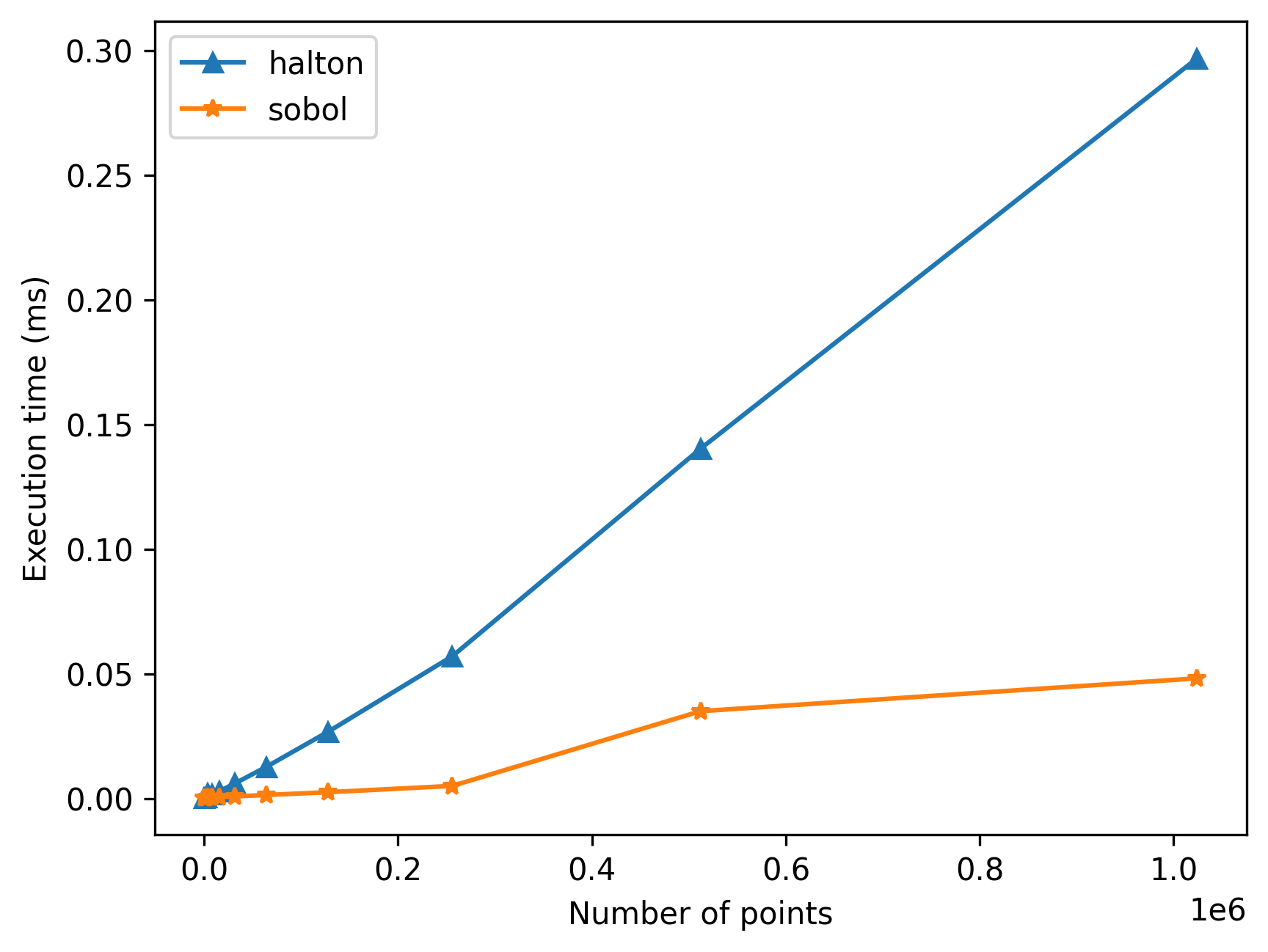}}
  \subfigure[\label{cost_50} $d=50$]{\includegraphics[width=0.23\linewidth]{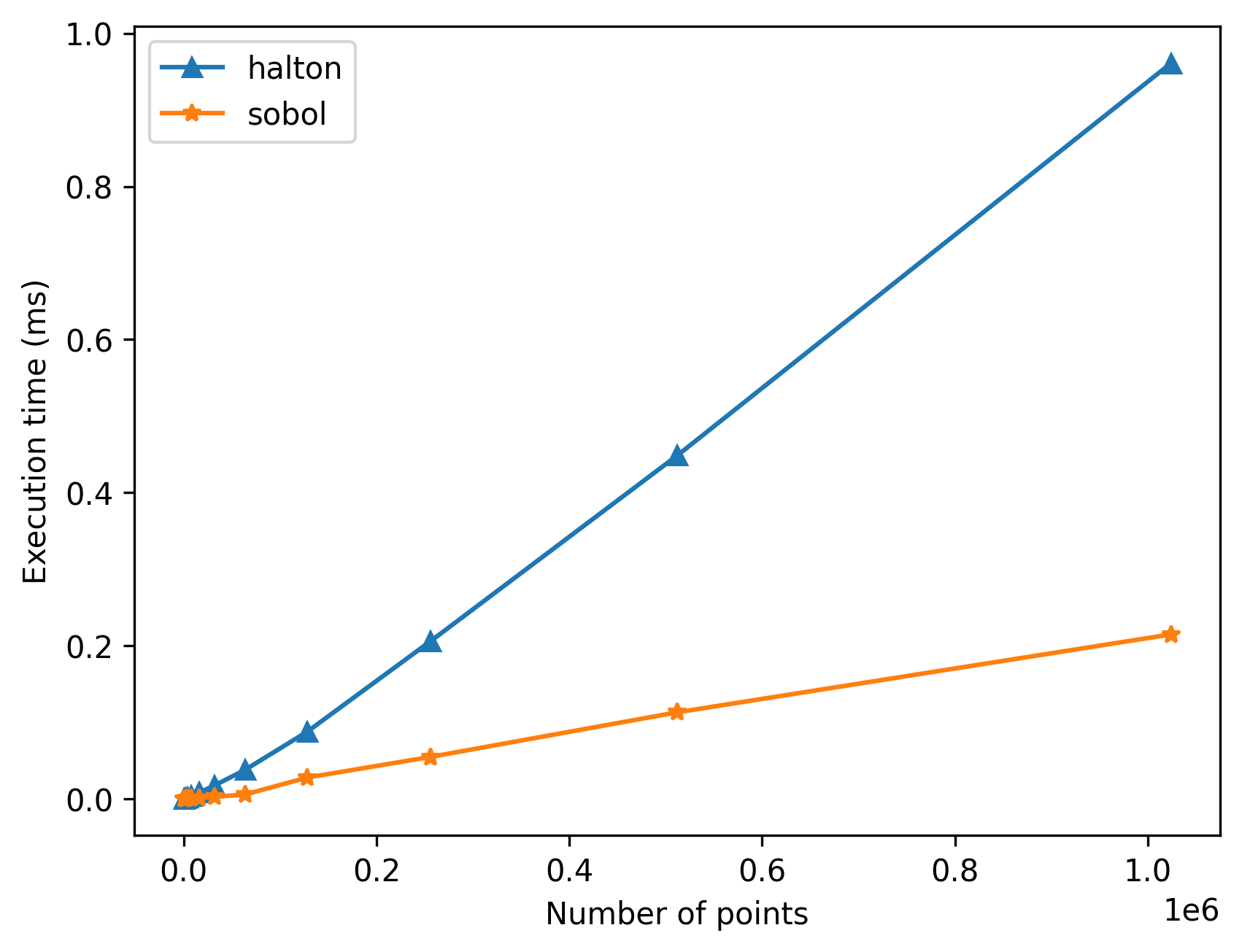}}
  \subfigure[\label{cost_100} $d=100$]{\includegraphics[width=0.23\linewidth]{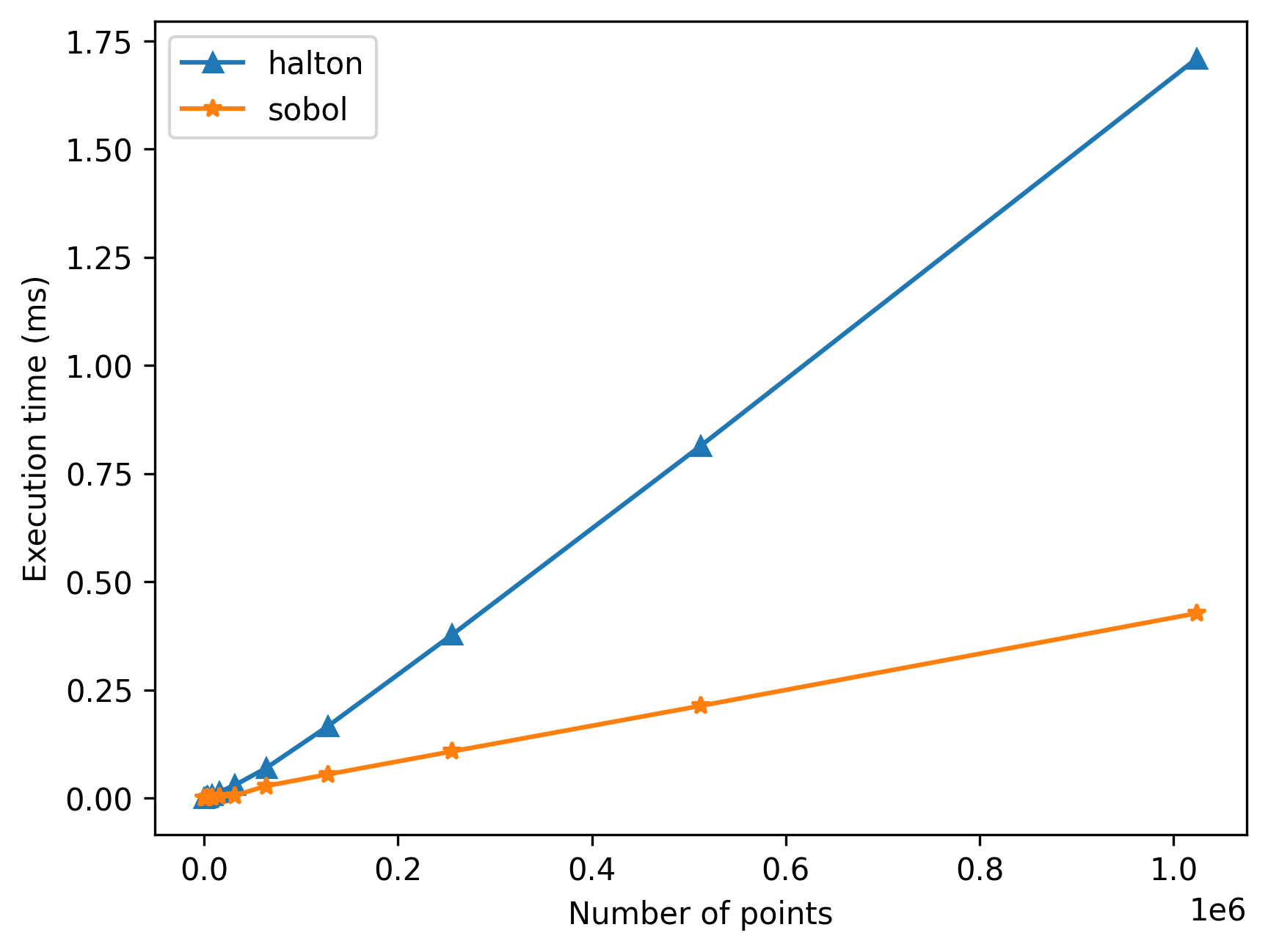}}\\
  \caption{Execution time with different dimensionality and different number of points for Halton sequences and Sobol' sequences.}
  \label{fig:cost}
\end{figure}

\subsection{Scale of the sampling pool}\label{section: sacle}
Although QRPINNs are training with a batch from the low-discrepancy sequence,  \cref{theorem: RQMC} shows that with randomly sampling in a pool of low-discrepancy sequences, the scale of the pool will affect the convergence rate of the estimation. Here we conducted experiments on different $N_{scale}$ to verify the performance and the necessity of batch size. The results, demonstrated in \cref{fig:scale}, show that with suitable $N_{scale}$, the performance of QRPINNs is better than $N_{scale}=1$ \textit{i.e.} training with deterministic low-discrepancy sequences. Furthermore, Halton sequences are more sensitive in $N_{scale}$ than Sobol' sequences. In the experiments of \cref{fig:sg_ac_d100}, we set $N_{scale}=10$ to balance the convergence rate of PINNs and the upper bound of \cref{theorem: RQMC} and training 30 epochs, \textit{i.e.} $s=30$ in \cref{eq: qmc_possibility}.

\begin{figure}[ht]
  \centering 
  \subfigure[\label{scale_d100_ac} Allen-Cahn]{\includegraphics[width=0.23\linewidth]{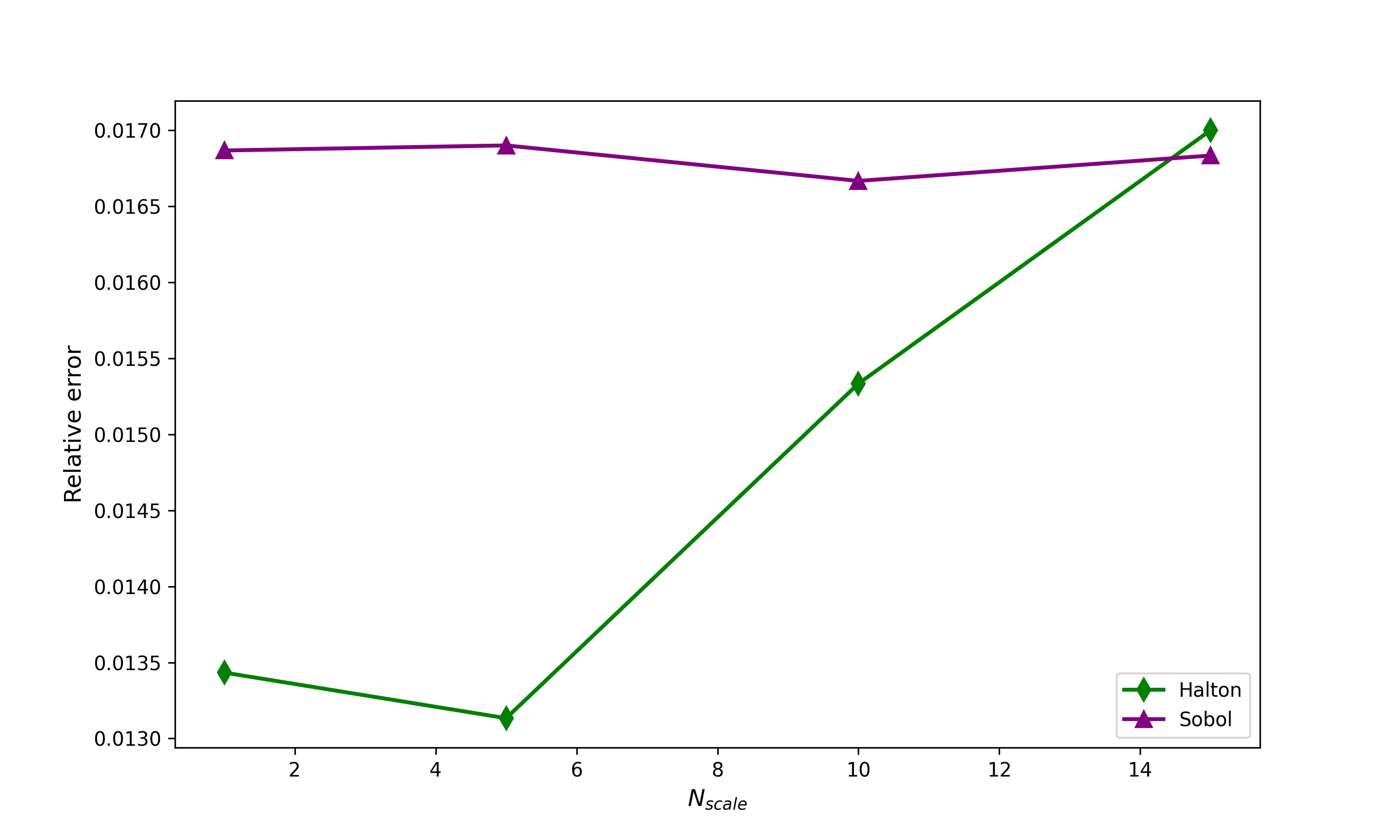}}
  \subfigure[\label{scale_d100_sg} Sine-Gordon]{\includegraphics[width=0.23\linewidth]{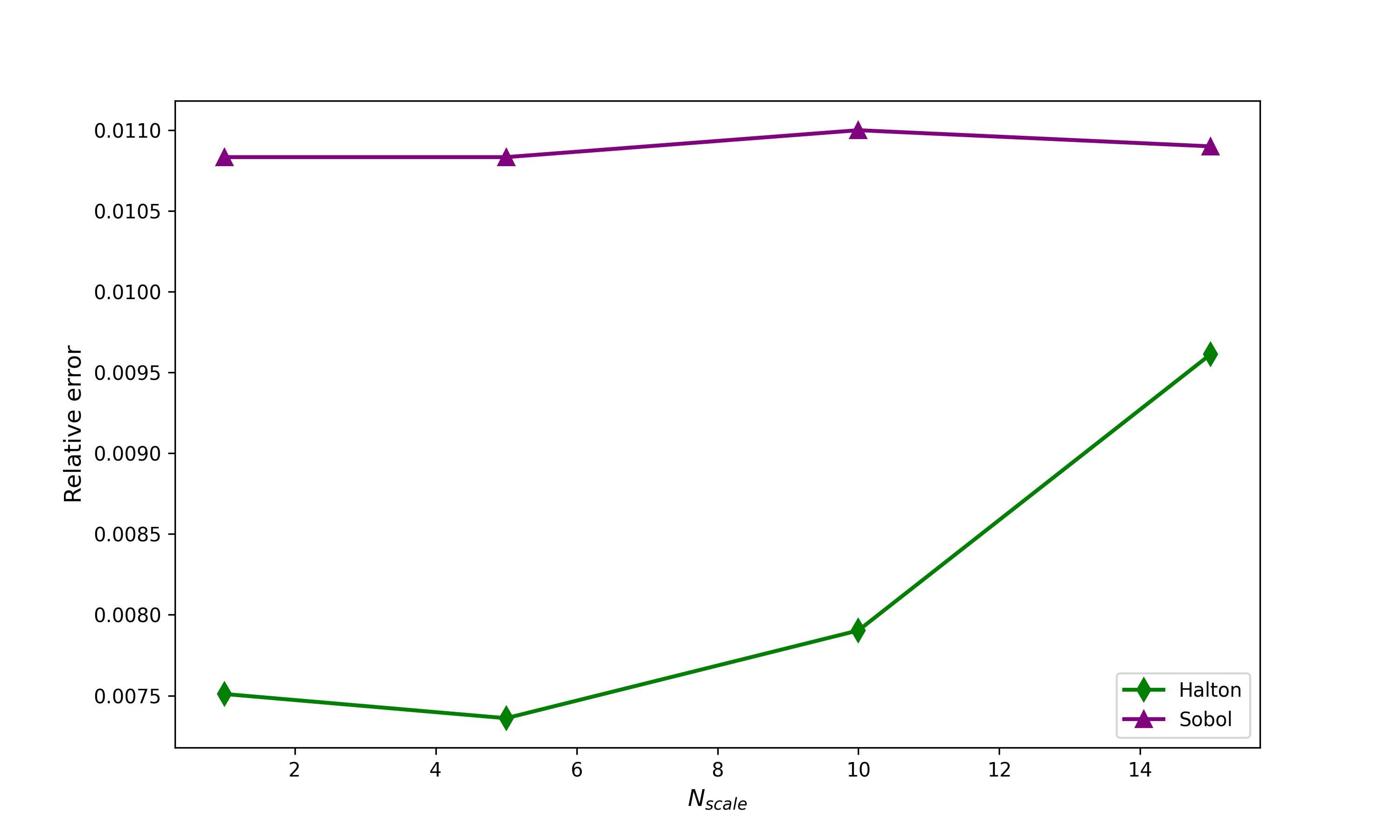}}\\
  \caption{Relationship between relative error and $N_{scale}$ in $d=100$}
  \label{fig:scale}
\end{figure}

\subsection{Ablation study\label{section: ablation}}
As we discussed in \cref{section: qmc_in_pinns}, low-discrepancy points can be the sampling pool for those adaptive sampling methods. Herein, in this section, we demonstrate the performance for the ablation study in Poisson's equations. We choose one sampling method from \{RAD,ACLE\} and one low discrepancy sequence from \{Halton, Sobol\}. We provide the results in \cref{table:ablation} for $d=100$. Overall, replacing the sampling pool by low-discrepancy sequences achieves better results in both $\alpha=0.1$ and $\alpha=1$. Furthermore, compared with \cref{table:poisson}, \cref{table:ablation} reveals that training with the best sampling pool (Halton for $\alpha=0.1$ and Vanilla for $\alpha=1$) may fail to obtain the best results when combining with adaptive sampling methods. For example, when using Halton+RAD in $\alpha=0.1$, the performance is slightly worse than that achieved by using only Halton. The discussion of $d=3,10$ and the ablation study of ACLE is in \cref{appendix: ablation}. 

\begin{table}[ht]
\caption{Ablation study $d=100$ on Poisson's equations}
\label{table:ablation}
\begin{center}
\begin{tabular}{lll}
\multicolumn{1}{c}{\bf $\alpha$ } & \multicolumn{1}{c}{\bf Sampling Pool} & \multicolumn{1}{c}{\bf RAD} \\ \toprule 
 
     \multirow{3}{*}{0.1}    & Vanilla & $3.19e-03\pm 2.16e-05$ \\
                           & Halton  & $3.14e-03\pm 9.43e-06$\\
                           & Sobol   & \cellcolor{blue!25}$3.11e-03\pm 4.71e-06$ \\ \midrule
     \multirow{3}{*}{1}   & Vanilla & $3.13e-02\pm 2.49e-04$ \\
                           & Halton  & \cellcolor{blue!25}$2.36e-02\pm 2.36e-04$ \\
                           & Sobol   & $3.15e-02\pm 8.16e-05$ \\ 
\bottomrule
\end{tabular}
\end{center}
\end{table}

\subsection{Solving higher dimensional PDEs by STDE}\label{section: stde}
Stochastic Taylor Derivative Estimator (STDE) \cite{shi2024stochastic} is proposed to use stochastic Taylor derivative to estimate the derivatives in the high-dimensional space \footnote{The details are provided in \cref{appendix: stde}.}. Using this estimator can solve higher dimensional PDEs more efficiently. To explore the capability of QRPINNs, we implement the Halton sequences and Sobol' sequences in STDE with $d=10^2,10^3,10^4$. Here we choose the benchmark from STDE, including time-dependent semi-linear heat equation \cref{eq: stde heat}, time-dependent Allen-Cahn equation \cref{eq: stde ac}, and time-dependent Sine-Gordon equation \cref{eq: stde sg}. The results in \cref{table: stde} shows that except Sine-Gordon with $d=100$ and some experiments that their relative error is in the scale of $10^{-1}$, QRPINNs obtain the best error. The maximum promotion is 77.5\%\ in Allen-Cahn with $d=10000$ and the minimum promotion is 10.6\% in Allen-Cahn with $d=1000$. The results reveal that QRPINNs are more competitive in higher dimensionality. However, since the cost of generating low-discrepancy sequences increases with dimensionality as shown in \cref{section: cost}, it is not practical to claim that, based on either Halton sequences or Sobol' sequences, QRPINNs are still efficient in high enough dimensionality. 
\begin{table}[htbp]
    \centering
    \caption{Higher dimensional PDEs}\label{table: stde}
    \begin{adjustbox}{width=\columnwidth, center}
    \begin{tabular}{l l cccc}
        Equation & $d$ & \textbf{Vanilla} & \textbf{Halton} & \textbf{Sobol} & \textbf{Promotion}\\
       \toprule
        \multirow{3}{*}{Allen-Cahn} 
        & $10^2$ & $8.12\text{e-}03\pm 4.06\text{e-}03$ & \cellcolor{blue!25}$2.54\text{e-}03\pm 2.43\text{e-}03$ & $3.01\text{e-}03\pm 2.48\text{e-}03$ & 68.7\% \\
        & $10^3$ & $1.10\text{e-}02\pm 1.01\text{e-}02$ & $2.25\text{e-}02\pm 8.92\text{e-}03$ & \cellcolor{blue!25}$9.83\text{e-}03\pm 1.42\text{e-}02$ & 10.6\%\\
        & $10^4$ & $1.55\text{e-}02\pm 1.03\text{e-}02$ & $1.49\text{e-}02\pm 1.37\text{e-}02$ & \cellcolor{blue!25}$3.48\text{e-}03\pm 9.93\text{e-}04$ & 77.5\%\\
        \midrule
        \multirow{3}{*}{Heat} 
        & $10^2$ & $1.45\text{e-}02\pm 1.20\text{e-}02$ & $4.98\text{e-}02\pm 4.62\text{e-}02$ & \cellcolor{blue!25}$1.05\text{e-}02\pm 1.02\text{e-}02$ & 27.6\%\\
        & $10^3$ & $2.16\text{e-}01\pm 7.35\text{e-}02$ & $3.30\text{e-}01\pm 1.64\text{e-}01$ & $2.41\text{e-}01\pm 9.97\text{e-}02$ &\multicolumn{1}{c}{-} \\
        & $10^4$ & $2.70\text{e-}01\pm 1.27\text{e-}01$ & $3.83\text{e-}01\pm 2.35\text{e-}01$ & $3.02\text{e-}01\pm 1.51\text{e-}01$ & \multicolumn{1}{c}{-}\\
        \midrule
        \multirow{3}{*}{Sine-Gordon} 
        & $10^2$ & \cellcolor{blue!25}$2.43\text{e-}02\pm 1.05\text{e-}02$ & $7.87\text{e-}02\pm 1.73\text{e-}02$ & $2.99\text{e-}02\pm 1.90\text{e-}02$ & -23.0\%\\
        & $10^3$ & $1.03\text{e-}01\pm 2.01\text{e-}02$ & $2.04\text{e-}01\pm 9,58\text{e-}03$ & $1.62\text{e-}01\pm 1.26\text{e-}02$ & \multicolumn{1}{c}{-} \\
        & $10^4$ & $1.19\text{e-}01\pm 4.10\text{e-}03$ & $1.17\text{e-}01\pm 8.08\text{e-}02$ & \cellcolor{blue!25}$3.28\text{e-}02\pm 3.50\text{e-}02$ & 72.4\%\\
        \bottomrule
    \end{tabular}
    \end{adjustbox}
\end{table}
\section{Conclusion}\label{Section: conclusion}
In this paper, we propose Quasi-Random Physics-Informed Neural Networks (QRPINNs), a novel framework that leverages low-discrepancy sequences to enhance the performance of Physics-Informed Neural Networks (PINNs), especially in solving high-dimensional PDEs. By replacing random sampling from the high-dimensional domain with random sampling from deterministic low-discrepancy sequences, QRPINNs address the inefficiency of traditional sampling in PINNs, particularly in high-dimensional PDEs. 

Comprehensive experiments on high-dimensional PDEs confirm that QRPINNs outperform both traditional PINNs and competitive adaptive sampling methods. QRPINNs exhibit smaller relative errors, due to the more uniform distribution of low-discrepancy sequences, which can better capture the structure of high-dimensional spaces. The experiments indicate that although MC performs similarly to QMC under a large enough number of sampling points, the requirement of a suitable batch size ensures the necessity of RQMC in machine learning. The experiments also reveal that adaptive sampling methods become invalid when dealing with high-dimensional PDEs. 

\paragraph{Limitations and Future work}
In this paper, we merely implemented two classical low-discrepancy sequences (Halton and Sobol) in QRPINNs and their computation of generating large deterministic pools will become significant in extremely high dimensions. There are more advanced sequences that can improve the convergence rate which we haven't implemented, for example modified Halton \cite{kocis1997computational}, Niederreiter \cite{tezuka2013discrepancy} and FNN \cite{bennett2004filtered}. On the other hand, some algorithms like component-by-component \cite{baldeaux2012efficient} and random lattice rule \cite{goda2025randomized} can construct quasi-random sequences in high dimensions efficiently. Furthermore, the weighted QMC \cite{sloan1998quasi, goda2015fast} may be more efficient since most high-dimensional functions merely have several principle dimensions. Finally, extending QRPINNs to couple with more adaptive strategies remains an important direction for exploration.

\bibliography{iclr2025_conference}
\bibliographystyle{iclr2025_conference}

\appendix

\crefname{section}{appendix}{appendices}
\Crefname{section}{Appendix}{Appendices}

\section{Sequences for quasi-Monte Carlo method}\label{appendix: Sequences}
\subsection{Halton sequence}
Firstly,  let's introduce the fractional number representation:

\begin{definition}
    Given a base $b$, a fractional number $0.d_1 d_2 d_3 \ldots b$ is converted to decimal by summing each digit $d_i$ multiplied by $b^{-i}$. For example, if $b=2, d_1=1,d_2=1,d_i=0 \text{ for } i>2 $, then $0.11_2=1 \cdot 2^{-1}+1 \cdot 2^{-2}=0.75$.
\end{definition}
Let $p_1, p_2, \ldots, p_d$ be the first $d$ prime numbers, where $d$ is the dimensionality. The Halton sequence $\boldsymbol{x}_0, \boldsymbol{x}_1, \ldots$ is given by
\begin{equation}
\boldsymbol{x}_i=\left(\phi_{p_1}(i), \phi_{p_2}(i), \ldots, \phi_{p_d}(i)\right), \quad i=0,1, \ldots,    
\end{equation}
where $\phi_b(i):=\sum_{a=1}^{\infty}\frac{i_a}{b^a}$. Explicit expression is as follows,
\begin{equation}
\begin{aligned}
\boldsymbol{x}_0 & =(0,0,0, \ldots, 0) \\
\boldsymbol{x}_1 & =\left(0.1_2, 0.1_3, 0.1_5, \ldots, 0.1_{p_d}\right) \\
\boldsymbol{x}_2 & =\left(0.01_2, 0.2_3, 0.2_5, \ldots, 0.2_{p_d}\right) \\
\boldsymbol{x}_3 & =\left(0.11_2, 0.01_3, 0.3_5, \ldots, 0.3_{p_d}\right) \\
& \vdots
\end{aligned}    
\end{equation}

In our experiments, the prime number starts from $2$.
\subsection{Sobol' sequence}
Suppose the sequence is $\boldsymbol{x}_0, \boldsymbol{x}_1, \ldots$, and $\boldsymbol{t}_i=\left(t_{i,1}, t_{i,2}, \ldots, t_{i,d},\right)$, Sobol' sequence is expressed by:
\begin{equation}
   t_{i, j}=i_0 v_{1, j} \oplus i_1 v_{2, j} \oplus \cdots \oplus i_{r-1} v_{r, j},
\end{equation}
where $\oplus$ is the bitwise XOR operator, $i_r$ is generated by the dyadic expansion of index $i$: $i=i_0+2 i_1+\cdots+2^{r-1} i_{r-1}$, and the direction numbers $v_{k,j}$ is defined by:
\begin{equation}\label{eq:v in sob}
    v_{k,j}:=\frac{m_{k,j}}{2^k},
\end{equation}
where $m_{k,j}$ is generated as follows:

Let $p_1,\cdots,p_d \in \mathbb{Z}_2[x]$ be  distinct primitive polynomials ordered according to their degree, and let
\begin{equation}
    p_j(x)=x^{e_j}+a_{1, j} x^{e_j-1}+a_{2, j} x^{e_j-2}+\cdots+a_{e_j-1, j} x+1, \quad \text { for } 1 \leq j \leq d,
\end{equation}
where $a_{j,k} \in \mathbb{Z}_b$ is the coefficient of the polynomial $p_j$ and $e_j$ is the degree of $p_j$. 

Then we can define $m_{k,j}$ for every $j \in [1,d]$, if $1 \leq k \leq e_j$, choose odd natural numbers $1 \leq m_{1, j}, \ldots, m_{e_j, j}$ such that $m_{k, j}<2^k$, and if $k>e_j$, define $m_{k, j}$ recursively by
\begin{equation}
m_{k, j}=2 a_{1, j} m_{k-1, j} \oplus \cdots \oplus 2^{e_j-1} a_{e_j-1, j} m_{k-e_j+1, j} \oplus 2^{e_j} m_{k-e_j, j} \oplus m_{k-e_j, j}.    
\end{equation}
In conclusion,
\begin{small}
    \begin{equation}\label{eq:m in sob}
    m_{k,j}=\begin{cases}
        \text{odd nature numbers}, & \text { for } k<e_j,\\
         a_{1, j} m_{k-1, j} \oplus \cdots \oplus 2^{e_j-1} a_{e_j-1, j} m_{k-e_j+1, j} \oplus 2^{e_j} m_{k-e_j, j} \oplus m_{k-e_j, j}, & \text { for } k>e_j.
    \end{cases}
\end{equation}
\end{small}

In totally, the whole algorithm for generating Sobol' sequence is provided in Algorithm \ref{algorithm:sob}.
\vspace{0.2cm}
\begin{algorithm}[H]
    \SetAlgoLined
    \KwIn{Dimensionality $d$, number of points $N$, primitive polynomials $p_1, \dots, p_d$.}
    \KwOut{Sobol' sequence $t_0, t_1, \dots, t_{N-1}$ in $[0,1]^d$.}
    
    Compute $m_{k,j}$ by \cref{eq:m in sob} and $v_{k,j}$ by \cref{eq:v in sob} for $j=1$ to $d$ and $k=1$ to $\lceil \log_2 N \rceil$\;
    
    \For{$i = 0,\cdots,N-1$}{
        Convert $i$ to binary: $i = \sum_{r=0}^{\infty} i_r 2^r$\;
        \For{$j = 1,\cdots,d$}{
            $t_{i,j} = \bigoplus_{r=0}^{\infty} i_r v_{r+1,j}$\;
        }
    }
    \Return{$t_0, t_1, \dots, t_{N-1}$.}
    \vspace{0.1cm}
    \caption{Sobol' sequence generation}\label{algorithm:sob}
\end{algorithm}
\section{Proof of the convergence rate of Monte Carlo method}\label{Appendix: MC convergence}
Here we give a brief proof of 1D Monte Carlo method under the uniform sampling distribution.
\begin{proof}

The Monte Carlo method approximates the quadrature of a function $f: [0,1]\rightarrow \mathbb{R}$ by a summation of a randomly sampling set $\{x_i\}_{i=1}^N$, \textit{i.e.}:
\begin{equation}
    I\left(f\right):=\int_0^1f(x)\mathrm{d}x\approx I_{\text{MC}}(f) :=\frac{1}{N}\sum_{i=1}^N f(x_i).
\end{equation}
The error of the Monte Carlo integration is:
\begin{equation}
    E_{\text{MC}}=\left|I\left(f\right)-I_{\text{MC}}(f)\right|.
\end{equation}
However, as the set is sampled randomly, so it is impossible to obtain a deterministic error of Monte Carlo method. Hence, we consider the expectation of the square of $E_{MC}$:
\begin{equation}
\mathbb{E}\left[\left|I\left(f\right)-I_{\text{MC}}(f)\right|^2\right]=\mathbb{E}\left[
I_{\text{MC}}\left(f\right)^2\right]-2 \mathbb{E}\left[I_{\text{MC}}(f)\right] I(f)+I\left(f\right)^2,
\end{equation}
As 
\begin{equation}
\begin{aligned}
\mathbb{E}\left[I_{\text{MC}}(f)\right] & =\int_{[0,1]} \cdots \int_{[0,1]}\left(\frac{1}{N} \sum_{i=1}^{N} f\left(x_i\right)\right) \mathrm{d} x_1 \ldots \mathrm{~d} x_{N} \\
& =\frac{1}{N} \sum_{i=1}^{N} I(f)\\
& =I(f),
\end{aligned}
\end{equation}
and
\begin{equation}
\begin{aligned}
& \mathbb{E}\left[I\left(f\right)^2\right]=\int_{[0,1]} \ldots \int_{[0,1]}\left(\frac{1}{N} \sum_{i=1}^{N} f\left(x_i\right)\right)^2 \mathrm{~d} x_1 \ldots \mathrm{~d} x_N \\
& =\int_{[0,1]} \ldots \int_{[0,1]}\left(\frac{1}{N^2} \sum_{i=1}^{N} \sum_{k=1}^{N} f\left(x_i\right) f\left(x_k\right)\right) \mathrm{d} x_1 \ldots \mathrm{~d} x_N \\
& =\int_{[0,1]} \ldots \int_{[0,1]}\left(\frac{1}{N^2} \sum_{i=1}^{N} f^2\left(x_i\right)+\frac{1}{N^2} \sum_{i=1}^{N} \sum_{\substack{k=1 \\
k \neq i}}^{N} f\left(x_i\right) f\left(x_k\right)\right) \mathrm{d} x_1 \ldots \mathrm{~d} x_N \\
& =\frac{1}{N^2} \sum_{i=1}^{N} \int_{[0,1]} f^2\left(x_i\right) \mathrm{d} x_i+\frac{1}{N^2} \sum_{i=1}^{N} \sum_{\substack{k=1 \\
k \neq i}}^{N} \int_{[0,1]} f\left(x_i\right) \mathrm{d} x_i \int_{[0,1]} f\left(x_k\right) \mathrm{d} x_k \\
& =\frac{1}{N} I\left(f^2\right)+\frac{N-1}{N}I\left(f\right)^2.
\end{aligned}
\end{equation}
Consequently, 
\begin{equation}
\begin{aligned}
    \mathbb{E}\left[\left|I\left(f\right)-I_{\text{MC}}(f)\right|^2\right] & =\frac{1}{N} I\left(f^2\right)+\frac{N-1}{N}I\left(f\right)^2-2I\left(f\right)^2 +I\left(f\right)^2 \\
    & = \frac{I(f^2)-I\left(f\right)^2}{N}.
\end{aligned}
\end{equation}

Hence, we obtain that 
\begin{equation}
    \begin{aligned}
        & \mathbb{E}\left[I_{\text{MC}}(f)\right]=I(f), \\
        & \operatorname{Var}\left[I_{\text{MC}}(f)\right]=\mathbb{E}\left[\left|I\left(f\right)-I_{\text{MC}}(f)\right|^2\right]=\frac{I(f^2)-I\left(f\right)^2}{N}.
    \end{aligned}
\end{equation}
Suppose $\sigma^2(f)=I(f^2)-I\left(f\right)^2$, then by Central Limit Theorem: 
\begin{equation}\label{eq:mc error}
\lim _{N \rightarrow \infty} \mathbb{P}\left(\left|I(f)-I_{\text{MC}}(f)\right| \leq c \frac{\sigma(f)}{\sqrt{N}}\right)=\frac{1}{\sqrt{2 \pi}} \int_{-c}^c \mathrm{e}^{-x^2 / 2} \mathrm{~d} x.
\end{equation}
Hence, instead of a deterministic error bound, we obtain a  ‘probabilistic’ error bound with a convergence rate $\mathcal{O}(N^{-1/2})$
\end{proof}
\section{Proof of the convergence rate of quasi Monte Carlo method}\label{Appendix: QMC convergence}
Here we give a brief proof of 1D quasi Monte Carlo method. 
\begin{proof}
Consider a function $f: [0,1]\rightarrow \mathbb{R}$ that holds:
\begin{equation}
    f(x)=f(1)-\int_0^1f^\prime(y)\mathbbm{1}_{[0,y]}(x)\mathrm{d}y,
\end{equation}
where
\begin{equation}
\mathbbm{1}_{[0, y]}(x)= \begin{cases}1 & \text { if } x \in[0, y], \\ 0 & \text { if } x \notin[0, y].\end{cases}
\end{equation}

The quasi Monte Carlo method also approximates the quadrature by a summation of a sequence set $\{x_i\}_{i=1}^N$, \textit{i.e.}:
\begin{equation}
    \int_0^1f(x)\mathrm{d}x\approx I_{\text{QMC}}(f) :=\frac{1}{N}\sum_{i=1}^N f(x_i).
\end{equation}

Consequently, the error of this approximation is: 
\begin{equation}\label{eq: proof qmc}
\begin{aligned}
      E_{\text{QMC}} & =\left|\int_0^1f(x)\mathrm{d}x-\frac{1}{N}\sum_{i=1}^N f(x_i)\right|,\\
      &= \left|\int_0^1\left(-\int_0^1 \mathbbm{1}_{[0, y]}(x) \mathrm{d} x+\frac{1}{N} \sum_{i=1}^{N} \mathbbm{1}_{[0, y]}\left(x_i\right)\right) f^{\prime}(y) \mathrm{d} y \right|\\
      &=\left|\int_0^1 \Delta_P(y) f^\prime (y)\mathrm{d}y\right|\\
      &\leq \left(\int_0^1\left|\Delta_P(y)\right|^p \mathrm{~d} y\right)^{1 / p}\left(\int_0^1\left|f^{\prime}(y)\right|^q \mathrm{~d} y\right)^{1 / q}\\
      & = \|\Delta p\|_{L_q} \|f^{\prime}\|_{L_q}.
\end{aligned}
\end{equation}
where $\Delta_P(y)$ is the local discrepancy of the set $P=\{x_i\}_{i=1}^N$, and 
\begin{equation}
    \Delta_P(y):=\frac{1}{N} \sum_{i=1}^{N} \mathbbm{1}_{[0, y]}\left(x_i\right)-\int_0^1 \mathbbm{1}_{[0, y]}(x) \mathrm{d} x.
\end{equation}
The last inequality of \cref{eq: proof qmc} held by Hölder's inequality, so $p,q$ satisfy $\frac{1}{p}+\frac{1}{q}=1$.
\end{proof}
Both the term $\|f^{\prime}\|_{L_q}$ and the $\sigma(f)$ in \cref{eq:mc error} are dependent on the function $f$, or more specifically, on the variation of $f$. While the $\|\Delta p\|_{L_q}$ depends on the specific sequence used in QMC.

Additionally, if $p=\infty, q=1$, we have
\begin{equation}
    E_{\text{QMC}}\leq V(f)D_N^*\left(x_1,\cdots x_N\right),
\end{equation}
where $V(f)$ is the bounded variation of $f$ in $[0,1]$.
\section{Proof of Theorem 1}\label{appendix: convergence rate}
Suppose that $f(x) \in \mathcal{H}$ is a smooth function defined on the domain $\Omega$, $L$ is a smooth operator, $J(f)=\int_\Omega L[f]\mathrm{d}x$, $J_N(f)=\frac{1}{N}\sum_{i=1}^N L[f(x_i)]$ where $x_i$ is sampled on $\Omega$. Also, suppose
\begin{equation}
\begin{aligned}
    f_N^* & =\arg\min_f J_N(f), \\
    f^* & =\arg\min_f J(f).
\end{aligned}
\end{equation}
Assume that
\begin{enumerate}
    \item $f^*$ is the unique global minimizer.
    \item The functional $J(f)$ is  strongly convex with a constant $\mu>0$ if $f \in B_r(f^*)$ and $f^*_N \in B_r(f^*)$.
\end{enumerate}
we want to show $\|f_N^*-f^*\| = \mathcal{O}\left(\|J[f]-J_N(f)\|\right)$.
\begin{proof}
As $J(f)$ is strongly convex, 
\begin{equation}
J(f) \geq J\left(f^*\right)+\left\langle\delta J\left(f^*\right), f-f^*\right\rangle+\frac{\mu}{2}\left\|f-f^*\right\|^2, \forall f \in B_r(f^*).
\end{equation}
Since $f^*$ is the unique global minimizer, then $\delta J\left(f^*\right)=0$, 
\begin{equation}
J(f) \geq J\left(f^*\right)+\frac{\mu}{2}\left\|f-f^*\right\|^2.
\end{equation}
Let $f=f_N^*$, we can derive that
\begin{equation}
\frac{\mu}{2}\left\|f_N^*-f^*\right\|^2 \leq J\left(f_N^*\right)-J\left(f^*\right).
\end{equation}
Since $J_N\left(f_N^*\right) \leq J_N\left(f^*\right)$,
\begin{equation}
\begin{aligned}
    J\left(f_N^*\right)-J\left(f^*\right) & = J\left(f_N^*\right)-J_N\left(f_N^*\right)+J_N\left(f_N^*\right)-J_N\left(f^*\right)+J_N\left(f^*\right)-J\left(f^*\right)\\
    &\leq J\left(f_N^*\right)-J_N\left(f_N^*\right)+J_N\left(f^*\right)-J\left(f^*\right)\\
    &\leq \|J\left(f_N^*\right)-J_N\left(f_N^*\right)\|+\|J_N\left(f^*\right)-J\left(f^*\right)\|.
\end{aligned}
\end{equation}
Herein,
\begin{equation}
\frac{\mu}{2}\left\|f_N^*-f^*\right\|^2 \leq \|J\left(f_N^*\right)-J_N\left(f_N^*\right)\|+\|J_N\left(f^*\right)-J\left(f^*\right)\|.
\end{equation}
Since the scale of convergence rate $\mathcal{O}\left(\|J[f]-J_N(f)\|\right)$ is invariant for $f$ including $f^*$ and $f_N^*$. Thus 
\begin{equation}
    \left\|f_N^*-f^*\right\|^2 = \mathcal{O}\left(\|J[f]-J_N(f)\|\right)+\mathcal{O}\left(\|J[f]-J_N(f)\|\right)=\mathcal{O}\left(\|J[f]-J_N(f)\|\right).
\end{equation}

Then, if those assumptions are satisfied for PINNs, the proof is finished. Obviously, $J$ is the $\mathcal{L}_{int}(f)$ in \cref{eq: loss_int} and $J_N$ is the $\mathcal{L}(f)$ in \cref{eq: split pinn loss}. 

For assumption 1. Since the PDE system is well defined, then the solution $u$ is unique and $J(u)=0$ which is the minimal value. Consequently, if $f(x)$ is smooth, then we can claim that $f^*(x)=u(x) \text{ for } \forall x \in \Omega$. Fortunately, in PINNs, the function $f$ is the output of the networks, herein, the assumption of smoothness is reasonable especially the vanilla PINNs \cite{raissi2019physics} embedded by $\tanh$ activation function. Herein, as $u$ is unique,  $f^*$ is unique.

For assumption 2. Since $\mathcal{N}$ and $\mathcal{B}$ are differential operators, $J(f)$ is a smooth functional with respected to $f$. Since the solution $u$ is unique, then $J(f)> 0$ if $f\neq u$. Combining with that $J(f)$ is smooth functional, there exist a constant $r$ such that for any $f$ in the compact ball $B_r(u)=\{f| \|f-u\|\leq r\}$, $\delta^2J(f)>0$. Furthermore, since  $\delta^2J[f]$ is bounded for $f \in B_r(u)$, then there exist a constant number $\mu>0$ such that $\delta^2J(f)\geq\mu, \forall f \in B_r(u)$. Consequently, $J(f)$ is strongly convex for $f \in B_r(u)=B_r(f^*)$.   
\end{proof}
\textbf{Remark}. Although in the above proof, the radius $r$ looks like a small number. In practice, if $\mathcal{N}$ is linear, then for any $f\in \mathcal{H}$, $J(f)$ is strongly convex (\textit{i.e.} $r=\infty$); if $\mathcal{N}$ is nonlinear, then $J(f)$ is still strongly convex for any $f\in\mathcal{H}$ if $\int \delta N^2+\mathcal{N}\cdot \delta^2\mathcal{N}\mathrm{d}x\geq 0$.
\section{Adaptive sampling methods}\label{Appendix: adaptive sampling methods}
\subsection{RAD}
Suppose $p(\boldsymbol{x})$ is the probability density function (PDF) for sampling points. Then residual-based adaptive distribution method (RAD) defines the PDF as 
\begin{equation}\label{eq: rad}
p(\boldsymbol{x}) \propto \frac{\varepsilon_r^k(\theta,\boldsymbol{x})}{\mathbb{E}[\varepsilon_r^k(\theta,\boldsymbol{x})]} + c,
\end{equation}
where $k>0$, $c>0$ are two hyperparameters. The whole algorithm is demonstrated in Algorithm \ref{algorithm:rad}. In our experiments, we inherit the hyperparameters from the original RAD methods: $k=1$, $c$=1. The size of the sampling pool is $50\times N_r$.
\vspace{0.2cm}
\begin{algorithm}[H]
    \SetAlgoLined
    \KwIn{Number of epochs $s$.}
    \KwOut{the output of PINN $u$}
     
     \vspace{0.1cm}
    Generate $\mathcal{P}$ by uniformly sampling from the domain. 
        
    Train the PINN for a certain number of iterations;

    \vspace{0.1cm}
    \For{$i = 2,\cdots,s$}{
        Generate $\mathcal{P}$ by randomly sampling based on the probability distribution function \cref{eq: rad};        
        Train the PINN for a certain number of iterations;
    }
     \vspace{0.1cm}
    \Return the output of PINN.
    \vspace{0.1cm}
    \caption{RAD}\label{algorithm:rad}
\end{algorithm}

\subsection{ACLE}

ACLE is the first algorithm to jointly optimize the selection of all training point types ($\boldsymbol{x}_{i c},\boldsymbol{x}_{b c},\boldsymbol{x}_{r}$ and experimental points) by automatically adjusting the proportion of collocation point types during training by leveraging the neural tangent kernel (NTK) \cite{wang2022and}. Notably, in our experiments, we focus on forward problems and don't have experimental points. 

Suppose $P=\{\boldsymbol{x}_i\}_{i=1}^N$ is the set of input points, then the $\Theta_t(P)=\{\theta_{ij}\}$ is the NTK matrix , its element $\theta_{ij}=\theta_t(\boldsymbol{x}_i, \boldsymbol{x}_j)$ computed as follows:
\begin{equation}\label{eq: acle_theta}
    \theta_t(\boldsymbol{x}_i, \boldsymbol{x}_j) = \nabla_\theta \varepsilon_r(\theta,\boldsymbol{x}_i) \left( \nabla_\theta \varepsilon_r(\theta,\boldsymbol{x}_j) \right)^\top.
\end{equation}
In addition, we define $\Theta_t(\boldsymbol{x},P)=\{\theta_t(\boldsymbol{x}, \boldsymbol{x}_i)\}_{i=1}^N$. And the convergence degree is defined:
\begin{equation}\label{eq: acle_alpha}
    \alpha(\boldsymbol{x};t)=\sum_{i=1}^{\infty} \lambda_{t,i}^{-1} \bigl\langle \varepsilon_r(\theta_{t+1},\boldsymbol{x})-\varepsilon_r(\theta_{t},\boldsymbol{x}), \psi_{t,i} \bigr\rangle_{\mathcal{H}_{\theta_t}}^2,
\end{equation}
where $\lambda_t$ and $\psi_t$ are the corresponding eigenvalue and eigenvector of $\Theta_t$, $\mathcal{H}_{\Theta_t}$ is the reproducing kernel Hilbert space with the kernel $\Theta_t$. The whole algorithm is shown in Algorithm \ref{algorithm:pinnacle}. However, since the computation of \cref{eq: acle_alpha,eq: acle_theta} is pretty complex, ACLE introduces Nystrom approximation \cite{williams2000using} to estimate the NTK matrix $\Theta_t$ and the convergence degree $\alpha$. Although this approximation is efficient, ACLE still requires unaffordable requirements of computation when solving PDEs with $d=100$ in our experiments. 
\vspace{0.2cm}
\begin{algorithm}[H]
\SetAlgoLined
    \KwIn{Number of epochs $s$.}
    \KwOut{the output of PINN $u$}

\BlankLine
\For{$i = 1,\cdots,s$}{
    Randomly sample candidate points $\mathcal{P}_{\text{pool}}$ from the domain.\;
    Compute $\Theta_t$ by \cref{eq: acle_theta}\;
    Select subset $\mathcal{P} \subset \mathcal{P}_{\text{pool}}$ based on the PDF ${\alpha}(\boldsymbol{x};t)$ computed by \cref{eq: acle_alpha}\;
    Train the PINN for a certain number of iterations;
}
\Return the output of PINN.
\caption{ACLE}\label{algorithm:pinnacle}
\end{algorithm}
\section{The errors of experiments}\label{appendix: details of experiments}
Here we provide the details of the relative error for Allen-Cahn equations and Sine-Gordon equations  in \cref{section: exp_high_pde}.
\subsection{Allen-Cahn}\label{Appendix: ac experiments}

For $d=3$, with different width and depth, the results are shown in \cref{table: ac3}. For the vanilla model, the best result is $2.15 \times 10^{-4} \pm 3.78 \times 10^{-5}$ achieved with depth=10 and width=50. The RAD model yields the smallest mean at $3.85 \times 10^{-4} \pm 1.96 \times 10^{-4}$ with depth=10 and width=50. For the ACLE model, the optimal outcome is $5.40 \times 10^{-4} \pm 8.41 \times 10^{-5}$ with depth=10 and width=100. The Halton model's best performance is $2.54 \times 10^{-4} \pm 3.72 \times 10^{-5}$ with depth=10 and width=50. Lastly, the Sobol model achieves the best result of $2.64 \times 10^{-4} \pm 2.11 \times 10^{-4}$ with depth=7 and width=50.

For $d=10$, with different width and depth, the results are shown in \cref{table: ac10}. For the vanilla model, the best result is $2.07 \times 10^{-3} \pm 3.40 \times 10^{-5}$ achieved with depth=10 and width=70. For the RAD model, the optimal outcome is $2.12 \times 10^{-3} \pm 2.89 \times 10^{-4}$ with depth=7 and width=100. The ACLE model's best performance is $2.09 \times 10^{-3} \pm 2.19 \times 10^{-4}$ obtained with depth=7 and width=100. The Halton yields the smallest error at $2.14 \times 10^{-3} \pm 8.73 \times 10^{-5}$ with depth=7 and width=100. Lastly, the Sobol model's best result is $2.20 \times 10^{-3} \pm 6.18 \times 10^{-5}$ with depth=7 and width=70.

For $d=100$, with different input points, the results are shown in \cref{table: ac100}. The results show that the best relative error achieved by Halton for Points $=1000,2000$, by Sobol for Points $=4000,8000$.
\begin{table}[htbp]
    \centering
    \caption{Relative error in Allen-Cahn $d=3$}\label{table: ac3}
    \begin{tabular}{l l ccc}
        \toprule
        \multirow{2}{*}{} & \multirow{2}{*}{\textbf{depth}} & \multicolumn{3}{c}{\textbf{width}} \\
        \cmidrule(lr){3-5}
         &  & \textbf{50} & \textbf{70} & \textbf{100} \\
        \midrule
        \multirow{4}{*}{\textbf{vanilla}} 
        & 4 & $3.85\text{e-}04\pm 1.54\text{e-}04$ & $7.49\text{e-}04\pm 3.69\text{e-}04$ & $4.93\text{e-}04\pm 9.81\text{e-}05$ \\
        & 5 & $3.64\text{e-}04\pm 2.60\text{e-}04$ & $5.21\text{e-}04\pm 1.27\text{e-}04$ & $5.28\text{e-}04\pm 2.31\text{e-}04$ \\
        & 7 & $2.26\text{e-}04\pm 1.74\text{e-}04$ & $3.37\text{e-}04\pm 5.99\text{e-}05$ & $4.94\text{e-}04\pm 4.10\text{e-}05$ \\
        & 10 & \cellcolor{blue!25}$2.15\text{e-}04\pm 3.78\text{e-}05$ & $6.14\text{e-}04\pm 7.74\text{e-}05$ & $6.85\text{e-}04\pm 1.45\text{e-}04$ \\
        \midrule
        \multirow{4}{*}{\textbf{RAD}} 
        & 4 & $5.54\text{e-}04\pm 2.80\text{e-}04$ & $5.54\text{e-}04\pm 1.18\text{e-}04$ & $6.49\text{e-}04\pm 2.37\text{e-}04$ \\
        & 5 & $5.42\text{e-}04\pm 4.46\text{e-}04$ & $5.01\text{e-}04\pm 9.64\text{e-}05$ & $5.92\text{e-}04\pm 1.98\text{e-}04$ \\
        & 7 & $4.51\text{e-}04\pm 2.50\text{e-}04$ & $7.03\text{e-}04\pm 1.99\text{e-}04$ & $5.69\text{e-}04\pm 1.04\text{e-}04$ \\
        & 10 & \cellcolor{blue!25}$3.85\text{e-}04\pm 1.96\text{e-}04$ & $7.52\text{e-}04\pm 2.79\text{e-}04$ & $4.94\text{e-}04\pm 4.10\text{e-}05$ \\
        \midrule
        \multirow{4}{*}{\textbf{ACLE}} 
        & 4 & $4.91\text{e-}04\pm 1.56\text{e-}04$ & $1.05\text{e-}03\pm 1.79\text{e-}04$ & $6.99\text{e-}04\pm 1.74\text{e-}04$ \\
        & 5 & $7.90\text{e-}04\pm 5.16\text{e-}04$ & $8.42\text{e-}04\pm 2.13\text{e-}04$ & $6.24\text{e-}04\pm 1.53\text{e-}04$ \\
        & 7 & $5.05\text{e-}04\pm 1.58\text{e-}04$ & $1.90\text{e-}03\pm 1.39\text{e-}03$ & $1.39\text{e-}03\pm 5.43\text{e-}04$ \\
        & 10 & $5.23\text{e-}03\pm 6.98\text{e-}03$ & $2.05\text{e-}02\pm 1.41\text{e-}02$ & \cellcolor{blue!25}$5.40\text{e-}04\pm 8.41\text{e-}05$ \\
        \midrule
        \multirow{4}{*}{\textbf{Helton}} 
        & 4 & $3.99\text{e-}04\pm 2.03\text{e-}04$ & $8.34\text{e-}04\pm 3.67\text{e-}04$ & $4.77\text{e-}04\pm 2.07\text{e-}04$ \\
        & 5 & $2.69\text{e-}04\pm 1.25\text{e-}04$ & $5.09\text{e-}04\pm 1.94\text{e-}04$ & $3.67\text{e-}04\pm 1.47\text{e-}04$ \\
        & 7 & $3.75\text{e-}04\pm 3.69\text{e-}04$ & $3.91\text{e-}04\pm 1.43\text{e-}04$ & $4.95\text{e-}04\pm 1.30\text{e-}04$ \\
        & 10 & \cellcolor{blue!25}$2.54\text{e-}04\pm 3.72\text{e-}05$ & $7.52\text{e-}04\pm 2.79\text{e-}04$ & $5.40\text{e-}04\pm 8.41\text{e-}05$ \\
        \midrule
        \multirow{4}{*}{\textbf{Sobol}} 
        & 4 & $3.89\text{e-}04\pm 1.89\text{e-}04$ & $6.17\text{e-}04\pm 2.41\text{e-}04$ & $5.18\text{e-}04\pm 5.60\text{e-}05$ \\
        & 5 & $2.93\text{e-}04\pm 1.47\text{e-}04$ & $4.83\text{e-}04\pm 1.22\text{e-}04$ & $3.99\text{e-}04\pm 1.77\text{e-}04$ \\
        & 7 & \cellcolor{blue!25}$2.64\text{e-}04\pm 2.11\text{e-}04$ & $4.08\text{e-}04\pm 1.07\text{e-}04$ & $5.11\text{e-}04\pm 2.02\text{e-}04$ \\
        & 10 & $3.88\text{e-}04\pm 5.65\text{e-}05$ & $5.95\text{e-}04\pm 2.21\text{e-}04$ & $6.98\text{e-}04\pm 1.01\text{e-}04$ \\
        \bottomrule
    \end{tabular}
\end{table}

\begin{table}[htbp]
    \centering
    \caption{Relative error in Allen-Cahn $d=10$}\label{table: ac10}
    \begin{tabular}{l l ccc}
        \toprule
        \multirow{2}{*}{} & \multirow{2}{*}{\textbf{depth}} & \multicolumn{3}{c}{\textbf{width}} \\
        \cmidrule(lr){3-5}
         &  & \textbf{50} & \textbf{70} & \textbf{100} \\
        \midrule
        \multirow{4}{*}{\textbf{vanilla}} 
        & 4 & $4.49\text{e-}03\pm 2.87\text{e-}05$ & $3.42\text{e-}03\pm 1.35\text{e-}04$ & $2.91\text{e-}03\pm 1.38\text{e-}04$ \\
        & 5 & $3.49\text{e-}03\pm 1.17\text{e-}04$ & $2.69\text{e-}03\pm 6.13\text{e-}05$ & $2.39\text{e-}03\pm 7.41\text{e-}05$ \\
        & 7 & $2.89\text{e-}03\pm 9.43\text{e-}06$ & $2.24\text{e-}03\pm 9.09\text{e-}05$ & $2.10\text{e-}03\pm 1.17\text{e-}04$ \\
        & 10 & $2.62\text{e-}03\pm 7.07\text{e-}05$ & \cellcolor{blue!25}$2.07\text{e-}03\pm 3.40\text{e-}05$ & $2.08\text{e-}03\pm 1.40\text{e-}04$ \\
        \midrule
        \multirow{4}{*}{\textbf{RAD}} 
        & 4 & $4.62\text{e-}03\pm 1.45\text{e-}04$ & $3.41\text{e-}03\pm 1.87\text{e-}04$ & $2.98\text{e-}03\pm 2.36\text{e-}05$ \\
        & 5 & $3.49\text{e-}03\pm 5.73\text{e-}05$ & $2.79\text{e-}03\pm 1.37\text{e-}04$ & $2.44\text{e-}03\pm 6.16\text{e-}05$ \\
        & 7 & $2.89\text{e-}03\pm 4.50\text{e-}05$ & $2.33\text{e-}03\pm 6.85\text{e-}05$ & \cellcolor{blue!25}$2.12\text{e-}03\pm 2.89\text{e-}04$ \\
        & 10 & $2.57\text{e-}03\pm 1.43\text{e-}04$ & $2.26\text{e-}03\pm 2.08\text{e-}04$ & $2.60\text{e-}03\pm 2.58\text{e-}04$ \\
        \midrule
        \multirow{4}{*}{\textbf{ACLE}} 
        & 4 &  $4.79\text{e-}03\pm 8.06\text{e-}05$ & $3.55\text{e-}03\pm 7.93\text{e-}05$ & $3.12\text{e-}03\pm 2.52\text{e-}04$ \\
        & 5 & $3.53\text{e-}03\pm 1.44\text{e-}04$ & $2.84\text{e-}03\pm 1.89\text{e-}05$ & $2.35\text{e-}03\pm 1.48\text{e-}04$ \\
        & 7 & $2.90\text{e-}03\pm 1.19\text{e-}04$ & $2.33\text{e-}03\pm 1.74\text{e-}04$ & \cellcolor{blue!25}$2.09\text{e-}03\pm 2.19\text{e-}04$ \\
        & 10 & $2.63\text{e-}03\pm 2.22\text{e-}04$ & $2.22\text{e-}03\pm 4.90\text{e-}05$ & $2.76\text{e-}03\pm 3.52\text{e-}04$ \\
        \midrule
        \multirow{4}{*}{\textbf{Halton}} 
        & 4 & $4.43\text{e-}03\pm 4.32\text{e-}05$ & $3.39\text{e-}03\pm 1.23\text{e-}04$ & $2.90\text{e-}03\pm 1.80\text{e-}04$ \\
        & 5 & $3.57\text{e-}03\pm 1.02\text{e-}04$ & $2.64\text{e-}03\pm 9.43\text{e-}06$ & $2.37\text{e-}03\pm 8.96\text{e-}05$ \\
        & 7 & $2.76\text{e-}03\pm 4.92\text{e-}05$ & $2.22\text{e-}03\pm 5.72\text{e-}05$ & \cellcolor{blue!25}$2.14\text{e-}03\pm 8.73\text{e-}05$ \\
        & 10 & $2.60\text{e-}03\pm 1.73\text{e-}04$ & $2.24\text{e-}03\pm 2.83\text{e-}05$ & $2.55\text{e-}03\pm 9.98\text{e-}05$ \\
        \midrule
        \multirow{4}{*}{\textbf{Sobol}} 
        & 4 & $4.54\text{e-}03\pm 1.93\text{e-}04$ & $3.39\text{e-}03\pm 1.23\text{e-}04$ & $2.90\text{e-}03\pm 1.80\text{e-}04$ \\
        & 5 & $3.57\text{e-}03\pm 1.02\text{e-}04$ & $2.64\text{e-}03\pm 9.43\text{e-}06$ & $2.29\text{e-}03\pm 2.94\text{e-}05$ \\
        & 7 & $2.78\text{e-}03\pm 1.10\text{e-}04$ & \cellcolor{blue!25}$2.20\text{e-}03\pm 6.18\text{e-}05$ & $2.43\text{e-}03\pm 1.89\text{e-}04$ \\
        & 10 & $2.63\text{e-}03\pm 1.60\text{e-}04$ & $2.24\text{e-}03\pm 2.83\text{e-}05$ & $2.32\text{e-}03\pm 1.41\text{e-}04$ \\
        \bottomrule
    \end{tabular}
\end{table}

\begin{table}[htbp]
    \centering
    \caption{Relative error in Allen-Cahn $d=100$}\label{table: ac100}
    \begin{adjustbox}{width=\columnwidth, center}
    \begin{tabular}{l cccc}
        \textbf{Points} & \textbf{Vanilla} & \textbf{RAD} & \textbf{Halton} & \textbf{Sobol} \\
        \toprule
        1000 & $1.69\text{e-}02\pm 5.77\text{e-}05$ & $1.69\text{e-}02\pm 4.04\text{e-}04$ & \cellcolor{blue!25}$1.30\text{e-}02\pm 5.51\text{e-}04$ & $1.71\text{e-}02\pm 1.53\text{e-}04$ \\
        2000 & $1.68\text{e-}02\pm 5.77\text{e-}05$ & $1.68\text{e-}02\pm 3.79\text{e-}04$ & \cellcolor{blue!25}$1.53\text{e-}02\pm 1.57\text{e-}03$ & $1.67\text{e-}02\pm 5.77\text{e-}05$ \\
        4000 & $1.70\text{e-}02\pm 4.04\text{e-}04$ & $1.68\text{e-}02\pm 4.04\text{e-}04$ & $1.70\text{e-}02\pm 2.31\text{e-}04$ & \cellcolor{blue!25}$1.68\text{e-}02\pm 1.15\text{e-}04$ \\
        8000 & $1.70\text{e-}02\pm 1.53\text{e-}04$ & $1.69\text{e-}02\pm 2.89\text{e-}04$ & $1.70\text{e-}02\pm 8.66\text{e-}05$ & \cellcolor{blue!25}$1.68\text{e-}02\pm 1.15\text{e-}04$ \\
        \bottomrule
    \end{tabular}
    \end{adjustbox}
\end{table}
\subsection{Sine-Gordon}\label{Appendix: sg experiments}
For $d=3$, with different width and depth, the results are shown in \cref{table: sg3}. For the vanilla model, the best result is $1.26 \times 10^{-4} \pm 2.39 \times 10^{-5}$ achieved with depth=5 and width=50. The RAD model yields the smallest mean at $1.67 \times 10^{-4} \pm 3.94 \times 10^{-5}$ with depth=5 and width=50. For the ACLE model, the optimal outcome is $2.82 \times 10^{-4} \pm 2.57 \times 10^{-5}$ with depth=5 and width=50. The Halton model's best performance is $1.34 \times 10^{-4} \pm 5.43 \times 10^{-5}$ obtained with depth=7 and width=50. Lastly, the Sobol model achieves the best result of $1.14 \times 10^{-4} \pm 6.48 \times 10^{-6}$ with depth=7 and width=50.

For $d=10$, with different width and depth, the results are shown in \cref{table: sg10}. For the vanilla model, the best result is $2.09 \times 10^{-3} \pm 4.08 \times 10^{-5}$ achieved with depth=7 and width=100. The RAD model yields the smallest mean at $2.43 \times 10^{-3} \pm 2.71 \times 10^{-4}$ with depth=10 and width=70. For the ACLE model, the optimal outcome is $2.41 \times 10^{-3} \pm 1.77 \times 10^{-4}$ with depth=7 and width=70. The Halton model's best performance is $2.20 \times 10^{-3} \pm 1.11 \times 10^{-4}$ obtained with depth=7 and width=100. Lastly, the Sobol model achieves the best result of $2.22 \times 10^{-3} \pm 2.05 \times 10^{-4}$ with depth=5 and width=100.

For $d=100$, with different input points, the results are shown in \cref{table: sg100}. The results show that the best relative error achieved by Halton for Points $=1000,2000, 8000$. by RAD for Points $=4000$. 
\begin{table}[htbp]
    \centering
    \caption{Relative error in Sine-Gordon $d=3$}\label{table: sg3}
    \begin{tabular}{l l ccc}
        \toprule
        \multirow{2}{*}{\textbf{model}} & \multirow{2}{*}{\textbf{depth}} & \multicolumn{3}{c}{\textbf{width}} \\
        \cmidrule(lr){3-5}
         &  & \textbf{50} & \textbf{70} & \textbf{100} \\
        \midrule
        \multirow{4}{*}{\textbf{vanilla}} 
        & 4 & $1.76\text{e-}04\pm 7.52\text{e-}05$ & $3.07\text{e-}04\pm 1.18\text{e-}04$ & $3.25\text{e-}04\pm 1.82\text{e-}04$ \\
        & 5 & \cellcolor{blue!25}$1.26\text{e-}04\pm 2.39\text{e-}05$ & $2.83\text{e-}04\pm 7.39\text{e-}05$ & $2.96\text{e-}04\pm 1.20\text{e-}04$ \\
        & 7 & $2.70\text{e-}04\pm 1.53\text{e-}04$ & $4.16\text{e-}04\pm 1.41\text{e-}04$ & $1.01\text{e-}03\pm 5.99\text{e-}04$ \\
        & 10 & $5.07\text{e-}04\pm 3.98\text{e-}04$ & $8.49\text{e-}04\pm 1.40\text{e-}04$ & $6.37\text{e-}04\pm 2.55\text{e-}04$ \\
        \midrule
        \multirow{4}{*}{\textbf{RAD}} 
        & 4 & $2.24\text{e-}04\pm 9.63\text{e-}05$ & $4.04\text{e-}04\pm 1.55\text{e-}04$ & $4.03\text{e-}04\pm 1.96\text{e-}04$ \\
        & 5 & \cellcolor{blue!25}$1.67\text{e-}04\pm 3.94\text{e-}05$ & $3.67\text{e-}04\pm 1.41\text{e-}04$ & $4.70\text{e-}04\pm 2.34\text{e-}04$ \\
        & 7 & $2.77\text{e-}04\pm 1.57\text{e-}04$ & $4.16\text{e-}04\pm 1.41\text{e-}04$ & $1.01\text{e-}03\pm 5.99\text{e-}04$ \\
        & 10 & $2.48\text{e-}04\pm 1.09\text{e-}04$ & $8.27\text{e-}04\pm 5.62\text{e-}05$ & $9.06\text{e-}04\pm 5.59\text{e-}04$ \\
        \midrule
        \multirow{4}{*}{\textbf{ACLE}} 
        & 4 & $3.46\text{e-}04\pm 7.83\text{e-}05$ & $2.86\text{e-}04\pm 8.38\text{e-}05$ & $4.87\text{e-}04\pm 1.70\text{e-}04$ \\
        & 5 & \cellcolor{blue!25}$2.82\text{e-}04\pm 2.57\text{e-}05$ & $4.03\text{e-}04\pm 1.69\text{e-}04$ & $4.47\text{e-}04\pm 8.44\text{e-}05$ \\
        & 7 & $3.35\text{e-}04\pm 2.22\text{e-}05$ & $1.32\text{e-}03\pm 1.29\text{e-}03$ & $2.04\text{e-}03\pm 6.17\text{e-}04$ \\
        & 10 & $4.55\text{e-}04\pm 2.12\text{e-}04$ & $8.77\text{e-}04\pm 1.86\text{e-}04$ & $1.02\text{e-}03\pm 3.43\text{e-}04$ \\
        \midrule
        \multirow{4}{*}{\textbf{Halton}} 
        & 4 & $2.32\text{e-}04\pm 1.07\text{e-}04$ & $2.61\text{e-}04\pm 7.45\text{e-}05$ & $2.87\text{e-}04\pm 1.47\text{e-}04$ \\
        & 5 & $1.73\text{e-}04\pm 4.52\text{e-}05$ & $3.06\text{e-}04\pm 1.34\text{e-}04$ & $3.61\text{e-}04\pm 1.26\text{e-}04$ \\
        & 7 & \cellcolor{blue!25}$1.34\text{e-}04\pm 5.43\text{e-}05$ & $3.31\text{e-}04\pm 8.20\text{e-}05$ & $8.93\text{e-}04\pm 3.70\text{e-}04$ \\
        & 10 & $2.48\text{e-}04\pm 1.09\text{e-}04$ & $8.27\text{e-}04\pm 5.62\text{e-}05$ & $1.01\text{e-}03\pm 5.99\text{e-}04$ \\
        \midrule
        \multirow{4}{*}{\textbf{Sobol}} 
        & 4 & $1.57\text{e-}04\pm 4.09\text{e-}05$ & $3.29\text{e-}04\pm 1.51\text{e-}04$ & $2.75\text{e-}04\pm 6.22\text{e-}05$ \\
        & 5 & $1.58\text{e-}04\pm 4.12\text{e-}05$ & $3.93\text{e-}04\pm 9.10\text{e-}05$ & $3.06\text{e-}04\pm 1.33\text{e-}04$ \\
        & 7 & \cellcolor{blue!25}$1.14\text{e-}04\pm 6.48\text{e-}06$ & $7.40\text{e-}04\pm 3.53\text{e-}04$ & $7.92\text{e-}04\pm 3.59\text{e-}04$ \\
        & 10 & $3.17\text{e-}04\pm 2.82\text{e-}04$ & $9.69\text{e-}04\pm 6.32\text{e-}04$ & $7.76\text{e-}04\pm 3.31\text{e-}04$ \\
        \bottomrule
    \end{tabular}
\end{table}

\begin{table}[htbp]
    \centering
    \caption{Relative error in Sine-Gordon $d=10$}\label{table: sg10}
    \begin{tabular}{l l ccc}
        \toprule
        \multirow{2}{*}{} & \multirow{2}{*}{\textbf{depth}} & \multicolumn{3}{c}{\textbf{width}} \\
        \cmidrule(lr){3-5}
         &  & \textbf{50} & \textbf{70} & \textbf{100} \\
        \midrule
        \multirow{4}{*}{\textbf{vanilla}} 
        & 4 & $5.34\text{e-}03\pm 3.17\text{e-}04$ & $4.11\text{e-}03\pm 1.17\text{e-}04$ & $3.09\text{e-}03\pm 1.72\text{e-}04$ \\
        & 5 & $4.23\text{e-}03\pm 1.79\text{e-}04$ & $2.97\text{e-}03\pm 2.17\text{e-}04$ & $2.41\text{e-}03\pm 8.83\text{e-}05$ \\
        & 7 & $3.01\text{e-}03\pm 1.20\text{e-}04$ & $2.56\text{e-}03\pm 1.40\text{e-}04$ & \cellcolor{blue!25}$2.09\text{e-}03\pm 4.08\text{e-}05$ \\
        & 10 & $2.86\text{e-}03\pm 2.38\text{e-}04$ & $2.18\text{e-}03\pm 1.16\text{e-}04$ & $3.10\text{e-}03\pm 2.33\text{e-}04$ \\
        \midrule
        \multirow{4}{*}{\textbf{RAD}} 
        & 4 & $5.61\text{e-}03\pm 5.59\text{e-}04$ & $3.93\text{e-}03\pm 3.00\text{e-}04$ & $3.19\text{e-}03\pm 3.67\text{e-}04$ \\
        & 5 & $4.52\text{e-}03\pm 3.23\text{e-}04$ & $3.03\text{e-}03\pm 1.06\text{e-}04$ & $2.63\text{e-}03\pm 2.12\text{e-}04$ \\
        & 7 & $3.49\text{e-}03\pm 5.70\text{e-}04$ & $2.81\text{e-}03\pm 2.01\text{e-}04$ & $3.10\text{e-}03\pm 9.47\text{e-}04$ \\
        & 10 & $2.70\text{e-}03\pm 1.57\text{e-}04$ & \cellcolor{blue!25}$2.43\text{e-}03\pm 2.71\text{e-}04$ & $3.89\text{e-}03\pm 5.81\text{e-}04$ \\
        \midrule
        \multirow{4}{*}{\textbf{ACLE}} 
        & 4 & $5.51\text{e-}03\pm 4.54\text{e-}04$ & $4.09\text{e-}03\pm 1.45\text{e-}04$ & $3.36\text{e-}03\pm 2.35\text{e-}04$ \\
        & 5 & $4.35\text{e-}03\pm 4.99\text{e-}04$ & $3.15\text{e-}03\pm 3.77\text{e-}05$ & $2.53\text{e-}03\pm 2.12\text{e-}04$ \\
        & 7 & $3.28\text{e-}03\pm 1.58\text{e-}04$ & \cellcolor{blue!25}$2.41\text{e-}03\pm 1.77\text{e-}04$ & $2.64\text{e-}03\pm 4.20\text{e-}04$ \\
        & 10 & $3.13\text{e-}03\pm 3.06\text{e-}04$ & $2.74\text{e-}03\pm 5.31\text{e-}04$ & $2.84\text{e-}03\pm 3.50\text{e-}04$ \\
        \midrule
        \multirow{4}{*}{\textbf{Halton}} 
        & 4 & $5.49\text{e-}03\pm 1.79\text{e-}04$ & $4.02\text{e-}03\pm 7.32\text{e-}05$ & $3.21\text{e-}03\pm 1.45\text{e-}04$ \\
        & 5 & $4.55\text{e-}03\pm 2.67\text{e-}04$ & $2.99\text{e-}03\pm 2.26\text{e-}04$ & $2.72\text{e-}03\pm 4.64\text{e-}05$ \\
        & 7 & $3.07\text{e-}03\pm 6.94\text{e-}05$ & $2.48\text{e-}03\pm 2.13\text{e-}04$ & \cellcolor{blue!25}$2.20\text{e-}03\pm 1.11\text{e-}04$ \\
        & 10 & $2.86\text{e-}03\pm 1.74\text{e-}04$ & $2.49\text{e-}03\pm 3.35\text{e-}04$ & $3.08\text{e-}03\pm 4.42\text{e-}04$ \\
        \midrule
        \multirow{4}{*}{\textbf{Sobol}} 
        & 4 & $5.40\text{e-}03\pm 2.24\text{e-}04$ & $4.00\text{e-}03\pm 6.60\text{e-}05$ & $3.01\text{e-}03\pm 1.05\text{e-}04$ \\
        & 5 & $4.27\text{e-}03\pm 9.80\text{e-}05$ & $2.93\text{e-}03\pm 8.16\text{e-}05$ & $2.55\text{e-}03\pm 4.92\text{e-}05$ \\
        & 7 & $3.17\text{e-}03\pm 1.21\text{e-}04$ & $2.59\text{e-}03\pm 1.03\text{e-}04$ & $2.53\text{e-}03\pm 2.28\text{e-}04$ \\
        & 10 & $2.82\text{e-}03\pm 1.07\text{e-}04$ & \cellcolor{blue!25}$2.22\text{e-}03\pm 2.05\text{e-}04$ & $3.16\text{e-}03\pm 3.69\text{e-}04$ \\
        \bottomrule
    \end{tabular}
\end{table}

\begin{table}[htbp]
    \centering
    \caption{Relative error in Sine-Gordon $d=100$}\label{table: sg100}
    \begin{adjustbox}{width=\columnwidth, center}
    \begin{tabular}{l cccc}
        \textbf{Points} & \textbf{Vanilla} & \textbf{RAD} & \textbf{Halton} & \textbf{Sobol} \\
        \toprule
        1000 & $1.09\text{e-}02\pm 2.00\text{e-}04$ & $1.09\text{e-}02\pm 3.51\text{e-}04$ & \cellcolor{blue!25}$7.33\text{e-}03\pm 1.49\text{e-}04$ & $1.08\text{e-}02\pm 3.06\text{e-}04$ \\
        2000 & $1.09\text{e-}02\pm 3.61\text{e-}04$ & $1.08\text{e-}02\pm 2.65\text{e-}04$ & \cellcolor{blue!25}$7.90\text{e-}03\pm 1.15\text{e-}04$ & $1.10\text{e-}02\pm 2.65\text{e-}04$ \\
        4000 & $1.10\text{e-}02\pm 2.52\text{e-}04$ & \cellcolor{blue!25}$1.09\text{e-}02\pm 1.53\text{e-}04$ & $1.09\text{e-}02\pm 3.06\text{e-}04$ & $1.09\text{e-}02\pm 2.52\text{e-}04$ \\
        8000 & $1.10\text{e-}02\pm 1.53\text{e-}04$ & $1.09\text{e-}02\pm 2.52\text{e-}04$ & \cellcolor{blue!25}$1.08\text{e-}02\pm 3.06\text{e-}04$ & $1.10\text{e-}02\pm 5.77\text{e-}05$ \\
        \bottomrule
    \end{tabular}
    \end{adjustbox}
\end{table}
\section{Details of cost}\label{appendix: cost}
For the experiments of generating the low-discrepancy sequences, we warm-up 2 times and run 20 times to obtain the expectation and the standard deviation. The used package is SciPy \cite{2020SciPy-NMeth}. \cref{fig:cost} has already indicated that both the memory consumption and the execution time exhibit a linear dependence on $N$. However, based on the detail shown in \cref{table: Halton_time,table: Sobol_time,table: Halton_memory,table: Sobol_memory}, the memory consumption is also linear dependent on $d$ while the execution time is not. Furthermore, because of the difference of the generating algorithm, Halton sequences and Sobol' sequences exhibit different sensitivities to dimensionality, with Sobol showing a more pronounced increase than Halton in computation time under high-dimensional problems.
\begin{table}[ht]
    \centering
    \caption{Halton Execution Time (ms)}\label{table: Halton_time}
    \begin{tabular}{cccc}
        \toprule
        \multirow{2}{*}{N} & \multicolumn{3}{c}{dim} \\
        \cmidrule(lr){2-4}
        & 3 & 10 & 50 \\
        \midrule
        100 & $3.20 \times 10^{-1} \pm 2.02 \times 10^{-2}$ & $2.81 \times 10^{-1} \pm 1.27 \times 10^{-2}$ & $4.49 \times 10^{-1} \pm 9.88 \times 10^{-3}$ \\
        2000 & $4.71 \times 10^{-1} \pm 1.24 \times 10^{-2}$ & $5.56 \times 10^{-1} \pm 1.37 \times 10^{-2}$ & $1.28 \times 10^{0} \pm 3.13 \times 10^{-2}$ \\
        4000 & $6.70 \times 10^{-1} \pm 1.04 \times 10^{-2}$ & $2.92 \times 10^{0} \pm 8.86 \times 10^{0}$ & $2.16 \times 10^{0} \pm 3.33 \times 10^{-2}$ \\
        8000 & $1.20 \times 10^{0} \pm 2.92 \times 10^{-1}$ & $1.57 \times 10^{0} \pm 7.20 \times 10^{-2}$ & $4.00 \times 10^{0} \pm 2.85 \times 10^{-2}$ \\
        16000 & $2.08 \times 10^{0} \pm 4.41 \times 10^{-2}$ & $2.98 \times 10^{0} \pm 1.01 \times 10^{-1}$ & $7.96 \times 10^{0} \pm 9.03 \times 10^{-2}$ \\
        32000 & $4.16 \times 10^{0} \pm 9.41 \times 10^{-2}$ & $6.10 \times 10^{0} \pm 1.92 \times 10^{-1}$ & $1.74 \times 10^{1} \pm 1.09 \times 10^{0}$ \\
        64000 & $6.76 \times 10^{0} \pm 2.25 \times 10^{-1}$ & $1.27 \times 10^{1} \pm 2.98 \times 10^{-1}$ & $3.73 \times 10^{1} \pm 8.06 \times 10^{-1}$ \\
        128000 & $1.42 \times 10^{1} \pm 4.88 \times 10^{-1}$ & $2.67 \times 10^{1} \pm 4.59 \times 10^{-1}$ & $8.59 \times 10^{1} \pm 2.56 \times 10^{-1}$ \\
        256000 & $3.03 \times 10^{1} \pm 9.14 \times 10^{-1}$ & $5.67 \times 10^{1} \pm 1.16 \times 10^{0}$ & $2.05 \times 10^{2} \pm 2.78 \times 10^{0}$ \\
        512000 & $6.76 \times 10^{1} \pm 6.25 \times 10^{0}$ & $1.29 \times 10^{2} \pm 2.97 \times 10^{-1}$ & $4.46 \times 10^{2} \pm 3.07 \times 10^{0}$ \\
        1024000 & $1.47 \times 10^{2} \pm 1.84 \times 10^{1}$ & $2.97 \times 10^{2} \pm 2.16 \times 10^{0}$ & $9.53 \times 10^{2} \pm 1.93 \times 10^{0}$ \\
        \bottomrule
    \end{tabular}
\end{table}
 
\begin{table}[ht]
    \centering
    \caption{Sobol Execution Time (ms)}\label{table: Sobol_time}
    \begin{tabular}{cccc}
        \toprule
        \multirow{2}{*}{N} & \multicolumn{3}{c}{dim} \\
        \cmidrule(lr){2-4}
        & 3 & 10 & 50 \\
        \midrule
        100 & $3.46 \times 10^{-1} \pm 5.40 \times 10^{-2}$ & $2.68 \times 10^{-1} \pm 8.64 \times 10^{-3}$ & $2.82 \times 10^{-1} \pm 8.24 \times 10^{-3}$ \\
        2000 & $3.53 \times 10^{-1} \pm 1.48 \times 10^{-2}$ & $3.09 \times 10^{-1} \pm 1.50 \times 10^{-2}$ & $4.25 \times 10^{-1} \pm 1.14 \times 10^{-2}$ \\
        4000 & $3.79 \times 10^{-1} \pm 1.04 \times 10^{-2}$ & $3.57 \times 10^{-1} \pm 1.55 \times 10^{-2}$ & $5.51 \times 10^{-1} \pm 1.79 \times 10^{-2}$ \\
        8000 & $4.43 \times 10^{-1} \pm 1.62 \times 10^{-2}$ & $4.31 \times 10^{-1} \pm 1.90 \times 10^{-2}$ & $7.85 \times 10^{-1} \pm 2.14 \times 10^{-2}$ \\
        16000 & $5.55 \times 10^{-1} \pm 2.63 \times 10^{-2}$ & $5.96 \times 10^{-1} \pm 6.10 \times 10^{-2}$ & $1.24 \times 10^{0} \pm 1.99 \times 10^{-2}$ \\
        32000 & $8.61 \times 10^{-1} \pm 4.56 \times 10^{-1}$ & $8.89 \times 10^{-1} \pm 1.00 \times 10^{-1}$ & $2.38 \times 10^{0} \pm 1.70 \times 10^{-1}$ \\
        64000 & $7.61 \times 10^{-1} \pm 1.08 \times 10^{-1}$ & $1.46 \times 10^{0} \pm 1.44 \times 10^{-1}$ & $5.52 \times 10^{0} \pm 2.18 \times 10^{0}$ \\
        128000 & $2.20 \times 10^{0} \pm 7.10 \times 10^{-1}$ & $2.92 \times 10^{0} \pm 6.64 \times 10^{-1}$ & $2.75 \times 10^{1} \pm 1.99 \times 10^{-1}$ \\
        256000 & $3.42 \times 10^{0} \pm 1.03 \times 10^{0}$ & $5.37 \times 10^{0} \pm 9.74 \times 10^{-1}$ & $5.35 \times 10^{1} \pm 2.19 \times 10^{-1}$ \\
        512000 & $1.08 \times 10^{1} \pm 3.99 \times 10^{0}$ & $2.51 \times 10^{1} \pm 1.19 \times 10^{0}$ & $1.06 \times 10^{2} \pm 1.05 \times 10^{0}$ \\
        1024000 & $1.27 \times 10^{1} \pm 4.13 \times 10^{0}$ & $4.87 \times 10^{1} \pm 2.09 \times 10^{0}$ & $2.08 \times 10^{2} \pm 5.55 \times 10^{-1}$ \\
        \bottomrule
    \end{tabular}
\end{table}
 
\begin{table}[ht]
    \centering
    \caption{Halton Memory Usage (KB)}\label{table: Halton_memory}
    \begin{tabular}{cccc}
        \toprule
        \multirow{2}{*}{N} & \multicolumn{3}{c}{dim} \\
        \cmidrule(lr){2-4}
        & 3 & 10 & 50 \\
        \midrule
        100 & $3.21 \times 10^{0} \pm 5.23 \times 10^{-3}$ & $1.02 \times 10^{1} \pm 0.00 \times 10^{0}$ & $5.11 \times 10^{1} \pm 0.00 \times 10^{0}$ \\
        2000 & $4.88 \times 10^{1} \pm 0.00 \times 10^{0}$ & $1.62 \times 10^{2} \pm 0.00 \times 10^{0}$ & $8.11 \times 10^{2} \pm 0.00 \times 10^{0}$ \\
        4000 & $9.68 \times 10^{1} \pm 0.00 \times 10^{0}$ & $3.22 \times 10^{2} \pm 5.23 \times 10^{-3}$ & $1.61 \times 10^{3} \pm 0.00 \times 10^{0}$ \\
        8000 & $1.93 \times 10^{2} \pm 0.00 \times 10^{0}$ & $6.42 \times 10^{2} \pm 0.00 \times 10^{0}$ & $3.21 \times 10^{3} \pm 0.00 \times 10^{0}$ \\
        16000 & $3.85 \times 10^{2} \pm 0.00 \times 10^{0}$ & $1.28 \times 10^{3} \pm 0.00 \times 10^{0}$ & $6.41 \times 10^{3} \pm 0.00 \times 10^{0}$ \\
        32000 & $7.69 \times 10^{2} \pm 0.00 \times 10^{0}$ & $2.56 \times 10^{3} \pm 0.00 \times 10^{0}$ & $1.28 \times 10^{4} \pm 0.00 \times 10^{0}$ \\
        64000 & $1.54 \times 10^{3} \pm 0.00 \times 10^{0}$ & $5.12 \times 10^{3} \pm 0.00 \times 10^{0}$ & $2.56 \times 10^{4} \pm 0.00 \times 10^{0}$ \\
        128000 & $3.07 \times 10^{3} \pm 0.00 \times 10^{0}$ & $1.02 \times 10^{4} \pm 0.00 \times 10^{0}$ & $5.12 \times 10^{4} \pm 0.00 \times 10^{0}$ \\
        256000 & $6.14 \times 10^{3} \pm 0.00 \times 10^{0}$ & $2.05 \times 10^{4} \pm 0.00 \times 10^{0}$ & $1.02 \times 10^{5} \pm 0.00 \times 10^{0}$ \\
        512000 & $1.23 \times 10^{4} \pm 0.00 \times 10^{0}$ & $4.10 \times 10^{4} \pm 0.00 \times 10^{0}$ & $2.05 \times 10^{5} \pm 0.00 \times 10^{0}$ \\
        1024000 & $2.46 \times 10^{4} \pm 0.00 \times 10^{0}$ & $8.19 \times 10^{4} \pm 0.00 \times 10^{0}$ & $4.10 \times 10^{5} \pm 0.00 \times 10^{0}$ \\
        \bottomrule
    \end{tabular}
\end{table}
 
\begin{table}[ht]
    \centering
    \caption{Sobol Memory Usage (KB)}\label{table: Sobol_memory}
    \begin{tabular}{cccc}
        \toprule
        \multirow{2}{*}{N} & \multicolumn{3}{c}{dim} \\
        \cmidrule(lr){2-4}
        & 3 & 10 & 50 \\
        \midrule
        100 & $2.68 \times 10^{0} \pm 0.00 \times 10^{0}$ & $8.28 \times 10^{0} \pm 0.00 \times 10^{0}$ & $4.03 \times 10^{1} \pm 0.00 \times 10^{0}$ \\
        2000 & $4.83 \times 10^{1} \pm 0.00 \times 10^{0}$ & $1.60 \times 10^{2} \pm 0.00 \times 10^{0}$ & $8.00 \times 10^{2} \pm 0.00 \times 10^{0}$ \\
        4000 & $9.63 \times 10^{1} \pm 0.00 \times 10^{0}$ & $3.20 \times 10^{2} \pm 0.00 \times 10^{0}$ & $1.60 \times 10^{3} \pm 0.00 \times 10^{0}$ \\
        8000 & $1.92 \times 10^{2} \pm 0.00 \times 10^{0}$ & $6.40 \times 10^{2} \pm 0.00 \times 10^{0}$ & $3.20 \times 10^{3} \pm 0.00 \times 10^{0}$ \\
        16000 & $3.84 \times 10^{2} \pm 0.00 \times 10^{0}$ & $1.28 \times 10^{3} \pm 0.00 \times 10^{0}$ & $6.40 \times 10^{3} \pm 0.00 \times 10^{0}$ \\
        32000 & $7.68 \times 10^{2} \pm 0.00 \times 10^{0}$ & $2.56 \times 10^{3} \pm 0.00 \times 10^{0}$ & $1.28 \times 10^{4} \pm 0.00 \times 10^{0}$ \\
        64000 & $1.54 \times 10^{3} \pm 0.00 \times 10^{0}$ & $5.12 \times 10^{3} \pm 0.00 \times 10^{0}$ & $2.56 \times 10^{4} \pm 0.00 \times 10^{0}$ \\
        128000 & $3.07 \times 10^{3} \pm 0.00 \times 10^{0}$ & $1.02 \times 10^{4} \pm 0.00 \times 10^{0}$ & $5.12 \times 10^{4} \pm 0.00 \times 10^{0}$ \\
        256000 & $6.14 \times 10^{3} \pm 0.00 \times 10^{0}$ & $2.05 \times 10^{4} \pm 0.00 \times 10^{0}$ & $1.02 \times 10^{5} \pm 0.00 \times 10^{0}$ \\
        512000 & $1.23 \times 10^{4} \pm 0.00 \times 10^{0}$ & $4.10 \times 10^{4} \pm 0.00 \times 10^{0}$ & $2.05 \times 10^{5} \pm 0.00 \times 10^{0}$ \\
        1024000 & $2.46 \times 10^{4} \pm 0.00 \times 10^{0}$ & $8.19 \times 10^{4} \pm 0.00 \times 10^{0}$ & $4.10 \times 10^{5} \pm 0.00 \times 10^{0}$ \\
        \bottomrule
    \end{tabular}
\end{table}
\section{Ablation Study}\label{appendix: ablation}
Except the $d=100$ experiments in \cref{table:ablation}, we also conducted the experiments of Poisson's equations \cref{eq: poisson} for $d=3$ and $d=100$ and the results are shown in \cref{table:ablation_3_10}. The results reveals that, replacing the sampling pool by low-discrepancy sequences is more accurate for RAD, the corresponding maximum promotion is 69.4\% when $\alpha=10$ and $d=3$. However, when implementing ACLE methods, low-discrepancy sequences are invalid and decrease the accuracy. We argue that randomly sampling the $P_{pool}$ is better to detect the convergence degree $\alpha$. Herein, utilizing low-discrepancy sequences will reduce the infinite $P_{pool}$ to a finite point set, although it can be alleviated by increasing $N_{scale}$.

\begin{table}[ht]
\caption{Ablation study $d=3$ and $d=10$}
\label{table:ablation_3_10}
\begin{center}
\begin{tabular}{lllll}
\multicolumn{1}{c}{\bf $d$ } & \multicolumn{1}{c}{\bf $\alpha$ } & \multicolumn{1}{c}{\bf Sampling Pool} & \multicolumn{1}{c}{\bf RAD} &\multicolumn{1}{c}{\bf ACLE}  \\ \toprule 
\multirow{9}{*}{3}     & \multirow{3}{*}{1}    & Vanilla & $2.00e-04\pm 9.83e-05$& \cellcolor{blue!25}$9.90e-05\pm 3.41e-05$ \\                       
                       &                       & Halton  & $1.77e-04\pm 7.85e-05$& $3.06e-04\pm 1.33e-04$ \\                       
                       &                       & Sobol   & \cellcolor{blue!25}$6.72e-05\pm 1.21e-05$& $2.40e-04\pm 1.41e-04$ \\ \cmidrule(lr){2-5}              
                       & \multirow{3}{*}{10}   & Vanilla & $3.11e-03\pm 1.84e-03$& \cellcolor{blue!25}$1.04e-03\pm 3.30e-04$ \\                       
                       &                       & Halton  & \cellcolor{blue!25}$9.49e-04\pm 5.74e-04$& $2.30e-03\pm 1.06e-02$ \\                       
                       &                       & Sobol   & $1.44e-03\pm 1.50e-04$& $4.37e-03\pm 2.87e-03$ \\ \cmidrule(lr){2-5}            
                       & \multirow{3}{*}{100}  & Vanilla & \cellcolor{blue!25}$5.51e-02\pm 1.22e-02$& \cellcolor{blue!25}$7.68e-02\pm 2.70e-02$ \\                       
                       &                       & Halton  & $6.12e-02\pm 1.69e-02$& $2.29e-01\pm 1.02e-01$ \\                       
                       &                       & Sobol   & $6.42e-02\pm 6.24e-03$& $2.19e-01\pm 4.60e-02$ \\     \midrule   
\multirow{9}{*}{10}    & \multirow{3}{*}{1}    & Vanilla & $1.04e-03\pm 3.27e-04$& \cellcolor{blue!25}$7.98e-04\pm 4.85e-05$ \\                       
                       &                       & Halton  & \cellcolor{blue!25}$8.18e-04\pm 4.55e-05$& $8.22e-04\pm 1.12e-04$ \\                       
                       &                       & Sobol   & $8.46e-04\pm 1.06e-04$& $8.23e-04\pm 5.63e-05$ \\ \cmidrule(lr){2-5}              
                       & \multirow{3}{*}{10}   & Vanilla & $3.31e-02\pm 1.80e-02$& \cellcolor{blue!25}$1.89e-02\pm 2.93e-03$ \\                       
                       &                       & Halton  & \cellcolor{blue!25}$1.88e-02\pm 3.02e-03$& $2.43e-02\pm 6.94e-03$ \\                       
                       &                       & Sobol   & $2.62e-02\pm 2.87e-03$& $2.85e-02\pm 1.19e-03$ \\ \cmidrule(lr){2-5}            
                       & \multirow{3}{*}{100}  & Vanilla & $1.86e+00\pm 7.17e-01$& $1.04e+00\pm 3.86e-02$ \\                       
                       &                       & Halton  & $9.63e-01\pm 2.52e-02$& $4.01e+00\pm 4.66e+00$ \\                       
                       &                       & Sobol   & $3.32e+01\pm 5.23e+01$& $4.55e+00\pm 2.81e+00$ \\
\bottomrule
\end{tabular}
\end{center}
\end{table}

\section{Details of stochastic Taylor derivative estimator}\label{appendix: stde}
\subsection{Introduction}

The Stochastic Taylor Derivative Estimator (STDE) is a novel framework designed to efficiently estimate arbitrary high-order or high-dimensional differential operators, addressing the critical computational bottlenecks of traditional automatic differentiation (AD) methods. STDE bridges the gap between univariate high-order AD (Taylor mode) and multivariate derivative tensor contractions, enabling scalable and accurate computation of high-dimensional and high-order differential operators that are intractable for conventional methods. 

Given a function \(u: \mathbb{R}^d \to \mathbb{R}\), a differentiable operator $\mathcal{L}$ can be expressed by: 
\begin{equation}\label{eq: stde differentiable operator}
\mathcal{L}u(x) = \sum_{\alpha \in \mathcal{I}(\mathcal{L})} C_{\alpha} \mathcal{D}^{\alpha} u(x),
\end{equation}

where $\alpha = (\alpha_1, \alpha_2, \ldots, \alpha_d)$ is a multi-index, $\mathcal{D}^{\alpha} = \frac{\partial^{|\alpha|}}{\partial x_1^{\alpha_1} \cdots \partial x_d^{\alpha_d}}$ denotes the partial derivative of order $|\alpha| = \sum \alpha_i$, and $C_{\alpha}$ are coefficients defining the operator $\mathcal{L}$. \cref{eq: stde differentiable operator} can be rewritten using tensor contractions:

\begin{equation}
\mathcal{L}u(x) = D_u^k(x) \cdot C(\mathcal{L}),
\end{equation}

where $D_u^k(x) \in \mathbb{R}^{d^k}$ is the $k$-th order derivative tensor of $u$ at $x$, $C(\mathcal{L}) \in \mathbb{R}^{d^k}$ is the coefficient tensor (with entries $C_{\alpha}$ at positions corresponding to $\alpha$), and $\cdot$ denotes tensor dot product.

To estimate $\mathcal{L}u(x)$, STDE constructs random jets such that the expectation of $\partial^l u(J_g^l)$ over jet samples equals the target contraction $D_u^k(x) \cdot C(\mathcal{L})$. Formally, for a distribution $p$ over $l$-jets, we require:

\begin{equation}
\mathbb{E}_{J_g^l \sim p} \left[ \partial^l u(J_g^l) \right] = D_u^k(x) \cdot C(\mathcal{L}).
\end{equation}

This is achieved by designing $p$ such that the expected value of the tangent product tensor matches $C(\mathcal{L})$:

\begin{equation}
\mathbb{E} \left[ \bigotimes_{i=1}^k v^{(v_i)} \right] = C(\mathcal{L}),
\end{equation}

where $v^{(v_i)}$ denotes the $v_i$-th tangent in the jet.

By averaging over multiple jet samples, STDE produces a stochastic estimator of $\mathcal{L}u(x)$ with computational complexity $O(k^2 d L)$ (where $L$ is the number of operations in the forward graph) and memory complexity $O(kd)$, avoiding the exponential scaling in $k$ and polynomial scaling in $d$ of traditional methods. This efficiency enables STDE to handle differential operators of high-order and high-dimension.
\subsection{Equations}
In STDE, we consider the following three PDEs:
\begin{itemize}
    \item Semi-linear Heat equation.
    \begin{equation}\label{eq: stde heat}
            \mathcal{L}u(\boldsymbol{x}, t) = \nabla^2 u(\boldsymbol{x}, t) + \frac{1 - u(\boldsymbol{x}, t)^2}{1 + u(\boldsymbol{x}, t)^2}, \quad (\boldsymbol{x},t) \in [-0.5,0.5]^d\times [0,T].
    \end{equation}
    with initial condition \( u(\boldsymbol{x},0) = 5/(10 + 2\|\boldsymbol{x}\|^2) \),
    \item Allen-Cahn equation
    \begin{equation}\label{eq: stde ac}
    \mathcal{L}u(\boldsymbol{x}, t) = \nabla^2 u(\boldsymbol{x}, t) + u(\boldsymbol{x}, t) - u(\boldsymbol{x}, t)^3, \quad (\boldsymbol{x},t) \in [-0.5,0.5]^d\times [0,T].
    \end{equation}
    with initial condition \( u(\boldsymbol{x},0) = \arctan(\max_i x_i) \),
    \item Sine-Gordon equation
    \begin{equation}\label{eq: stde sg}
    \mathcal{L}u(\boldsymbol{x}, t) = \nabla^2 u(\boldsymbol{x}, t) + \sin(u(\boldsymbol{x}, t)),\quad (\boldsymbol{x},t) \in [-0.5,0.5]^d\times [0,T].
    \end{equation}
    with initial condition \( u(\boldsymbol{x},0) = 5/(10 + 2\|\boldsymbol{x}\|^2) \),
\end{itemize}
All three equation uses the test point $\boldsymbol{x}_{\text{test}} = \boldsymbol{0}$ and terminal time $T = 0.3$. The used numerical solution $u(\boldsymbol{0},0.3)$ is listed in \cref{table: stde_numerical}.
\begin{table}[htbp]
    \centering
    \caption{Numerical solution $u(\boldsymbol{0},0.3)$}\label{table: stde_numerical}
    \begin{tabular}{l c c}  
        Equation & $d$ & $u(\boldsymbol{0},0.3)$ \\
       \toprule
        \multirow{3}{*}{\textbf{Allen-Cahn}} 
        & $10^2$ & $1.04510$ \\
        & $10^3$ & $1.09100$ \\
        & $10^4$ & $1.11402$ \\
        \midrule
        \multirow{3}{*}{\textbf{Heat}} 
        & $10^2$ & $0.31674$ \\
        & $10^3$ & $0.28753$ \\
        & $10^4$ & $0.28433$ \\
        \midrule
        \multirow{3}{*}{\textbf{Sine-Gordon}} 
        & $10^2$ & $0.0528368$ \\
        & $10^3$ & $0.0055896$ \\
        & $10^4$ & $0.0005621$ \\
        \bottomrule
    \end{tabular}

\end{table}

\end{document}